\pdfoutput=1
\documentclass[Afour,sageh,times]{sagej}

\usepackage{bm}
\usepackage{amssymb}
\usepackage{pifont}
\usepackage{multirow}
\newcommand{\cmark}{\ding{51}}\newcommand{\xmark}{\ding{55}}
\usepackage{microtype}
\usepackage{booktabs}
\usepackage{tabularx}
\newcolumntype{C}{>{\centering\arraybackslash}X}
\AtBeginDocument{%
  \setlength{\abovedisplayskip}{4pt plus 2pt minus 1pt}%
  \setlength{\belowdisplayskip}{4pt plus 2pt minus 1pt}%
  \setlength{\abovedisplayshortskip}{1.5pt plus 2pt}%
  \setlength{\belowdisplayshortskip}{3pt plus 2pt minus 1pt}%
}
\usepackage{algorithm}
\usepackage{algpseudocode}
\usepackage{tikz}
\usetikzlibrary{arrows.meta,positioning,fit,backgrounds,calc}
\definecolor{arrowsteel}{HTML}{5D6884}
\definecolor{pblue}{HTML}{3C5488}
\definecolor{pred}{HTML}{E64B35}
\definecolor{pgreen}{HTML}{00A087}
\definecolor{pgrey}{HTML}{8491B4}
\definecolor{pslate}{HTML}{4A4A4A}
\definecolor{porange}{HTML}{F39B7F}
\tikzset{
  fl/.style={-{Latex[length=1.6mm,width=1.2mm]}, line width=0.6pt, draw=arrowsteel},
  nd/.style={draw, rounded corners=1.4pt, line width=0.45pt, align=center,
             inner sep=2.2pt, font=\sffamily\scriptsize},
  inp/.style={nd, fill=pgrey!14,   draw=pgrey!90},
  stg/.style={nd, fill=pblue!9,    draw=pblue!80},
  brn/.style={nd, fill=pgreen!11,  draw=pgreen!85},
  dec/.style={nd, fill=pred!9,     draw=pred!85},
  res/.style={nd, fill=porange!24, draw=porange!80!black},
  th/.style={nd, inner sep=1.6pt, align=center, text width=22mm,
             font=\sffamily\scriptsize},
}
\usepackage{url}
\usepackage[colorlinks,bookmarksopen,bookmarksnumbered,citecolor=blue,urlcolor=blue,linkcolor=blue]{hyperref}
\usepackage[nameinlink,capitalize,noabbrev]{cleveref}
\makeatletter
\long\def\@makefntext#1{\parindent 1em\noindent\textsuperscript{\@thefnmark}\,#1}
\makeatother
\usepackage{colortbl}
\definecolor{tabhead}{gray}{0.93}
\colorlet{tabbest}{pblue!34}\colorlet{tabsec}{pblue!18}\colorlet{tabthird}{pblue!7}
\definecolor{bestcell}{RGB}{158,190,224}
\definecolor{seccell}{RGB}{214,228,242}

\graphicspath{{figures/}}
\def\journalname{The International Journal of Robotics Research}
\def\volumeyear{2026}

\theoremstyle{plain}
\newtheorem{theorem}{Theorem}
\newtheorem{proposition}{Proposition}

\newtheorem{corollary}{Corollary}
\theoremstyle{definition}
\newtheorem{definition}{Definition}
\newtheorem{assumption}{Assumption}
\newtheorem{problem}{Problem}
\newtheorem{remark}{Remark}
\crefname{remark}{Remark}{Remarks}
\crefname{assumption}{Assumption}{Assumptions}

\newcommand{\SO}{\mathrm{SO}}
\newcommand{\SE}{\mathrm{SE}}

\newcommand{\R}{\mathbb{R}}
\newcommand{\tr}{\mathrm{tr}}
\newcommand{\diag}{\mathrm{diag}}
\newcommand{\rank}{\mathrm{rank}}
\newcommand{\rangeop}{\mathrm{range}}
\newcommand{\nullop}{\mathrm{null}}
\newcommand{\Exp}{\mathrm{Exp}}
\newcommand{\Log}{\mathrm{Log}}
\newcommand{\calV}{\mathcal{V}}
\newcommand{\calS}{\mathcal{S}}
\newcommand{\calR}{\mathcal{R}}
\newcommand{\calL}{\mathtt{L}}
\newcommand{\calC}{\mathtt{C}}
\newcommand{\calB}{\mathtt{B}}
\newcommand{\calD}{\mathcal{D}}
\newcommand{\calF}{\mathcal{F}}
\newcommand{\hatT}{\hat{\bm{T}}}
\newcommand{\skewop}[1]{[#1]_{\times}}
\newcommand{\LamR}{\bm{\Lambda}_{R}}
\newcommand{\LamT}{\bm{\Lambda}_{t}}
\newcommand{\LamRv}{\bm{\Lambda}^{\mathrm{v}}_{R}}
\newcommand{\LamTv}{\bm{\Lambda}^{\mathrm{v}}_{t}}
\newcommand{\MDB}{\mathrm{MDB}}
\newcommand{\NIS}{\mathrm{NIS}}

\usepackage{flafter}                  %

\begin{document}
\widowpenalty=10000\clubpenalty=10000\displaywidowpenalty=10000%
\setcounter{dbltopnumber}{1}%

\runninghead{Hu et al.}

\title{VeriCalib: Verifiable Target-Based LiDAR-Camera Calibration}

\author{Xiangcheng Hu\affilnum{1,2}, Xieyuanli Chen\affilnum{3}\textsuperscript{*}, Feiyi Chen\affilnum{4}, Shengli Ren\affilnum{2}, Shiyang Chen\affilnum{2}, Guoqiang Ye\affilnum{2}, Hao Wang\affilnum{2},  Ping Tan\affilnum{1}, Steven L. Waslander\affilnum{5}, Fuchun Sun\affilnum{6} and Jin Wu\affilnum{7}}

\affiliation{\affilnum{1}Department of Electronic and Computer Engineering, Hong Kong University of Science and Technology, Hong Kong, China\\
\affilnum{2}X Square Robot Technology (Shenzhen) Co., Ltd., Shenzhen, China\\
\affilnum{3}College of Intelligence Science and Technology at National University of Defense Technology, Changsha, Hunan, China\\
\affilnum{4}Independent Researcher\\
\affilnum{5}University of Toronto Institute for Aerospace Studies and the University of Toronto Robotics Institute, Toronto, Canada\\
\affilnum{6}Department of Computer Science and Technology, Tsinghua University, Beijing, China\\
\affilnum{7}School of Intelligent Science and Technology, University of Science and Technology Beijing, Beijing, China}

\corrauth{Xieyuanli Chen, College of Intelligence Science and Technology at National University of Defense Technology, Changsha, Hunan, China} \email{xieyuanli.chen@nudt.edu.cn}

\begin{abstract}
Target-based extrinsic calibration estimates the rigid transform between sensors from a few observations of a known target, and its result is commonly assessed by the root-mean-square error of the fitted residuals. However, this scalar error is computed on the same observations and therefore cannot indicate whether to accept the result, re-extract, recollect or add observations.
This paper presents VeriCalib, which formulates extrinsic calibration as a verifiable estimation problem and reports a calibration state.
The key insight is a two-level extrinsic compensation mechanism, in which the extrinsic compensates the rigid part of each observation error at the residual level, and rotation and translation compensate each other at the parameter level. Consequently, the fitted residual can reveal neither an inconsistency between the observations and the model nor whether the estimated extrinsic is reliable in every direction.
Accordingly, we assess the calibration through a structured covariance that jointly propagates the uncertainty of the camera and LiDAR board features with a known-geometry prior, and estimate its scale as a variance component from the residuals. On this covariance, we decouple rotation and translation to resolve the precision along each direction and guide data collection toward the weakest one. Since this precision presumes consistent observations, we further develop within-view and cross-view checks that examine the internal residual and the rigid discrepancy between observations, respectively, with thresholds that hold their significance level under an unknown noise level. We also extend the minimal detectable bias to the extrinsic to bound the smallest detectable fault along each direction. These criteria form a calibration state that indicates whether to accept the result or how to collect the data.
Experiments in simulation and on real-world robot platforms show that the structured covariance predicts the observed estimation scatter, and that weakest-direction acquisition meets a strict rotation requirement with up to $30\%$ fewer observations than D-, A-, and E-optimal designs. Under an unknown noise magnitude, the consistency checks detect faulty observations while removing few fault-free ones, unlike robust estimators with fixed thresholds. Cross-session validation further confirms a verified calibration state.
The code and datasets will be released to support the future research of verifiable state estimation.\footnote{\url{https://github.com/JokerJohn/VeriCalib}}
\end{abstract}


\maketitle

\section{Introduction}
\label{sec:intro}
\subsection{Motivation and Challenges}

Extrinsic calibration, which estimates the rigid transform
between two sensor frames, is a core problem in multi-sensor robotics perception.
Every fused measurement passes through this transform,
and its error propagates into dependent tasks as a systematic bias.
Such a bias degrades perception and planning
in autonomous driving~\citep{bai2022transfusion,dong2023benchmarking},
grasping on embodied platforms~\citep{fang2023anygrasp},
and the trajectory and map references used for evaluation~\citep{hu2024paloc,hu2025mapeval}.
Among sensor pairs, the LiDAR-camera pair is widely adopted,
as it combines 3D geometry with 2D visual
semantics~\citep{lin2022r,koide2024tightly,zhao2024superloc}.
Its extrinsic is commonly estimated from a calibration target with a known pattern,
such as a checkerboard or a circular-hole board,
which provides well-defined cross-modal
correspondences~\citep{mirzaei2012lidar,zhou2018automatic,beltran2022velo2cam,zheng2025fastcalib,hu2026p2calib}.
\begin{figure}[!t]\centering\includegraphics[width=\columnwidth]{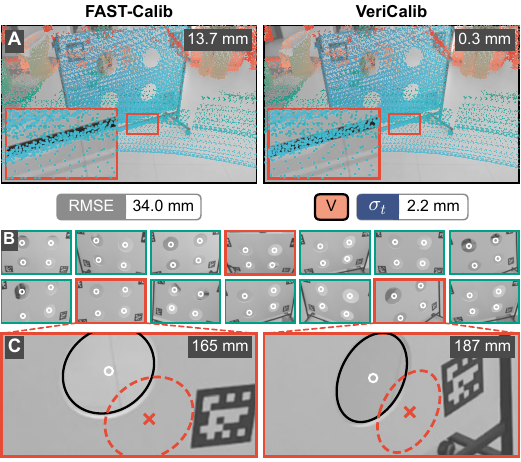}
\caption{Head-pair calibration from $14$ observations, two with hole centers misplaced by automatic rim extraction (\textbf{C}: true rim and center in black and white, extraction in red). \textbf{A}: LiDAR scans projected with the extrinsics of FAST-Calib and VeriCalib (tags: translation deviation from the reference). The RMSE of $34.0$\,mm neither reveals nor locates the errors, whereas VeriCalib removes both and one further observation (\textbf{B}: accepted in teal, removed in red) and reports the state, validation-limited (V), with its translation precision $\sigma_t$.}
\label{fig:teaser}\end{figure}
As there is no external ground truth, the calibration result is commonly assessed by the root-mean-square error (RMSE) of its fitted residuals~\citep{domhof2019extrinsic,zheng2025fastcalib}, its consistency~\citep{yuan2021pixel}, or covariance~\citep{brookshire2012extrinsic,pandey2015automatic,yamada2024probability}.
\cref{fig:teaser} illustrates the gap: on a set with two extraction errors, a conventional pipeline returns a residual, whereas our method identifies the inconsistent observations and reports a calibration state.
Learning-based assessments likewise return only a binary label~\citep{tahiraj2025calornocal} or a conformal interval~\citep{cocheteux2025uncertainty}, and neither specifies the corrective action that a failed calibration requires. Consequently, a large residual provides no criterion
for choosing among accepting the result, re-extracting the features,
recollecting the observations, and adding observations, and even in mass-production calibration the number and placement of targets are determined empirically.

We formulate this task as the \emph{verifiable calibration problem},
which assesses whether the estimated extrinsic satisfies given
precision and detection requirements
and determines the subsequent calibration action accordingly.
Solving this problem involves three challenges.

\begin{itemize}
\item \textbf{Independent and nominal noise modeling.}
Existing uncertainty-aware calibration methods~\citep{koo2020analytic,kummerle2020unified,brookshire2012extrinsic} model the target features as independent measurements with a nominal noise magnitude. However, the camera features on one board share the estimated board pose, the LiDAR features share the fitted board plane, and the actual noise magnitude varies across sensors and scenes. As a result, the reported precision and the test thresholds do not match the actual uncertainty, leading to overconfident acceptance or to the rejection of consistent observations.

\item \textbf{Coupled constraint analysis.}
Existing observability analyses of calibration~\citep{mirzaei2012lidar,lv2022observability,yang2023online} verify only that the extrinsic is observable and, like degeneracy analyses in LiDAR registration~\citep{zhang2016degeneracy,tuna2023x}, do not account for the compensation between rotation and translation. As a small rotation and a compensating translation yield nearly the same target observations, this compensation masks weakly constrained directions. Data collection is thus left without a criterion for which target pose to add and when the collected poses suffice.

\item \textbf{Insufficient model consistency assessment.}
Real target observations contain gross errors such as mis-extracted features and wrong correspondences. Existing calibration methods reject them by fixed robust thresholds~\citep{beltran2022velo2cam,zheng2025fastcalib,wiesmann2024joint}, without a stated significance level or a bound on the smallest fault that passes undetected, which reliability analyses provide only for individual measurements~\citep{baarda1968testing,hafez2025quantifying}. A faulty observation can therefore pass undetected and bias the accepted extrinsic, whereas a consistent one can be discarded.
\end{itemize}

The three analyses depend on one another, since the precision and thresholds rest on the covariance and all three presume consistent observations, whereas existing pipelines address them separately through a scalar residual. We therefore represent the observation as a unified \emph{structured measurement factor}, its residual together with the covariance and the local Jacobian, and resolve the dependence by first selecting the consistent observations under a robust working scale and then evaluating the scale, the directional precision, and the detection limits on the selected set (\cref{subsec:scaleest,subsec:diagnosis}). \cref{fig:overview} follows this procedure on the head pair, from the factor of each observation to the calibration state that feeds back into data collection.

\begin{figure*}[!t]\centering\includegraphics[width=\textwidth]{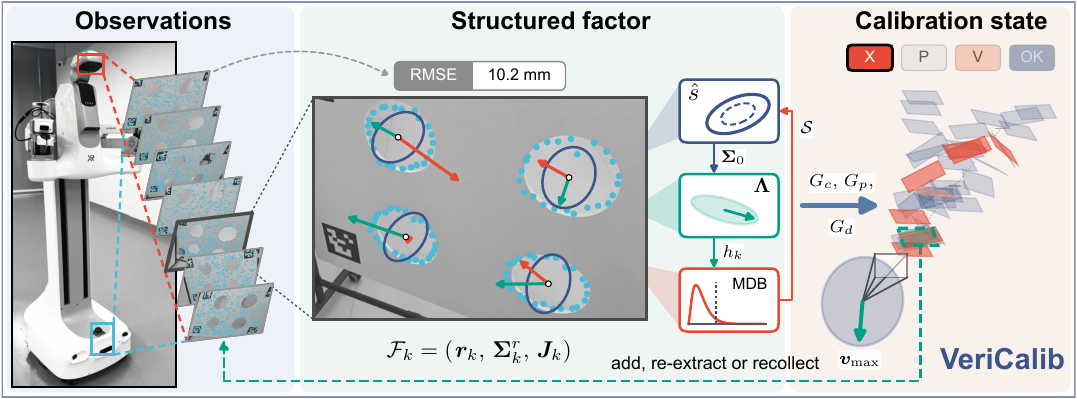}
\caption{VeriCalib overview on the head pair. \textbf{Left}: observations with projected LiDAR points (blue; red box: head camera; cyan boxes: chassis LiDAR and camera) and their fitted RMSE. \textbf{Center}: structured factor $\calF_k$ of one observation with the $1\sigma$ hole-center residual covariance (navy ellipses), the LiDAR boundary points (blue dots), and the absorbed (green) and internal (red) residual parts; the boxes derive the scale $\hat s$, the directional information $\bm\Lambda$, and the MDB, which depend on one another. \textbf{Right}: calibration state (X/P/V/OK: inconsistent, precision-limited, validation-limited, verified), accepted (navy), $p<0.05$ (salmon), and removed (red) board poses, detection-limit ellipse with its weakest direction $\bm v_{\max}$ (green), and recommended next pose (green dashed).}
\label{fig:overview}
\vspace{-0.3cm}
\end{figure*}

\subsection{Contributions}
This paper presents VeriCalib, which formulates target-based extrinsic calibration as a verifiable estimation problem on the structured measurement factor. We identify the root cause that prevents a scalar calibration residual from revealing the extrinsic state as a two-level extrinsic compensation (\cref{sec:factor}). At the residual level, the extrinsic absorbs the rigid part of each observation error, and at the parameter level, rotation and translation compensate each other, so that neither a small residual nor an aligned projection implies an accurate extrinsic. Whereas geodetic reliability analyzes the redundancy of individual measurements~\citep{forstner1987reliability}, this mechanism acts on the rigid-body extrinsic and its coupled rotation and translation, and it determines how the observations of a calibration are weighted, which target poses the calibration requires, and which faults it can detect.

\textbf{Our first contribution} (\cref{sec:method}) is a structured covariance model for target observations, on which every reported precision, threshold, and detection limit rests. Unlike independent feature covariances of nominal magnitude, its shape is derived from the measurement model of each sensor, through the shared board pose in the camera side and the shared board plane with a known-geometry prior in the LiDAR side (\cref{subsec:camcov,subsec:lidarcov}), and its scale is estimated from the absorbed residual components, which alone determine the estimate (\cref{subsec:scaleest}). We further quantify the precision gain of the weighted over the unweighted solver (\cref{thm:gain}).

As the precision of a calibration is limited by its weakest direction, \textbf{our second contribution} (\cref{sec:method:obs}) resolves the precision of the extrinsic by physical direction. Marginalizing the compensation between rotation and translation from the information of this covariance yields directional precisions and a weakest-direction score for selecting the next target pose (\cref{subsec:guidance}). Its closed form for the four-hole target further shows that a common increase of the board distance leaves the rotation precision invariant (\cref{subsec:closedform}).

Directional precision holds only for observations consistent with the model. To this end, \textbf{our third contribution} (\cref{sec:method:tests}) is a model consistency assessment with directional detection limits. A within-view check tests the internal residual of each observation, and a cross-view check tests its agreement with the remaining observations, both at a significance level that holds under an unknown noise magnitude (\cref{subsec:ftest,subsec:diagnosis}). We further extend the minimal detectable bias (MDB) of geodetic reliability~\citep{baarda1968testing,teunissen2018dia} from individual measurements to the directions of the extrinsic, which is, to the best of our knowledge, the first direction-resolved detection limit for target-based calibration (\cref{subsec:mdb}). The three analyses combine into a calibration state, each value of which indicates the next calibration action (\cref{subsec:status}).

The framework requires only a target-specific factor construction, which we provide for circular-hole and checkerboard targets (\cref{subsec:scope}).
\cref{sec:related} reviews related work and \cref{sec:formulation} formulates the problem.
\cref{sec:factor} analyzes the two-level compensation, and \cref{sec:method,sec:method:obs,sec:method:tests} develop the first three contributions in turn. \cref{sec:results} reports the experiments in simulation and on real-world data, \cref{sec:discussion} discusses the results and the scope of the verification,
and \cref{sec:conclusion} concludes the paper.

\section{Related Work}
\label{sec:related}
We first review target-based calibration (\cref{subsec:rw-target}) and then the noise modeling, constraint analysis, and consistency assessment (\crefrange{subsec:rw-uncertainty}{subsec:rw-reliability}), and \cref{tab:related} positions representative methods along these challenges.

\begin{table*}[!t]
\centering\caption{Comparison of extrinsic calibration methods.}
\label{tab:related}\footnotesize\setlength{\tabcolsep}{1.8pt}\renewcommand{\arraystretch}{1.12}\setlength{\aboverulesep}{0pt}\setlength{\belowrulesep}{0pt}\setlength{\extrarowheight}{1pt}\renewcommand{\cmark}{$\checkmark$}\renewcommand{\xmark}{\leavevmode\textcolor{black!50}{$\times$}}
\providecommand{\hdrule}[1]{\arrayrulecolor{tabhead}\specialrule{\dimexpr\cmidrulewidth+0.8pt\relax}{-0.4pt}{-\dimexpr\cmidrulewidth+0.4pt\relax}\arrayrulecolor{black}#1}
\begin{tabularx}{\textwidth}{>{\centering\arraybackslash}p{0.5cm}>{\raggedright\arraybackslash}Xl!{\color{black!30}\vrule width 0.4pt}>{\centering\arraybackslash}p{1.12cm}>{\centering\arraybackslash}p{1.12cm}>{\centering\arraybackslash}p{1.12cm}!{\color{black!30}\vrule width 0.4pt}>{\centering\arraybackslash}p{1.3cm}>{\centering\arraybackslash}p{1.45cm}!{\color{black!30}\vrule width 0.4pt}>{\centering\arraybackslash}p{1.36cm}>{\centering\arraybackslash}p{1.0cm}>{\centering\arraybackslash}p{1.0cm}>{\centering\arraybackslash}p{1.0cm}}
\toprule
\rowcolor{tabhead}[\dimexpr\tabcolsep+0.4pt\relax][\dimexpr\tabcolsep+0.4pt\relax] & & & \multicolumn{3}{c}{Uncertainty modeling} & \multicolumn{2}{c}{Constraint analysis} & \multicolumn{4}{c}{Result assessment}\\
\hdrule{\cmidrule(l{0.1em}r{0.1em}){4-6}\cmidrule(l{0.1em}r{0.1em}){7-8}\cmidrule(lr){9-12}}
\rowcolor{tabhead}[\dimexpr\tabcolsep+0.4pt\relax][\dimexpr\tabcolsep+0.4pt\relax] & & & & & & & & & \multicolumn{2}{c}{\rule{0pt}{2.5ex}Fault test} & \\
\hdrule{\cmidrule(lr){10-11}}
\rowcolor{tabhead}[\dimexpr\tabcolsep+0.4pt\relax][\dimexpr\tabcolsep+0.4pt\relax] & \multirow{-2}{*}{Method} & \multirow{-2}{*}{Venue} & \multirow{-2}{*}{LiDAR} & \multirow{-2}{*}{Camera} & \multirow{-2}{*}{Scale} & \multirow{-2}{*}{Analysis} & \multirow{-2}{*}{Guidance} & \multirow{-2}{*}{Metric} & \rule{0pt}{2.5ex}Within & Cross & \multirow{-2}{*}{\shortstack{Detect.\\limit}}\\
\midrule
& \citet{zhou2018automatic} & IROS & \xmark & \xmark & \xmark & Coupled & Empirical & Ref. & \cmark & \xmark & \xmark\\
& \citet{domhof2019extrinsic} & ICRA & Indep.$^\dagger$ & Indep.$^\dagger$ & Est.$^\dagger$ & \xmark & \xmark & RMSE & \cmark & \xmark & \xmark\\
& velo2cam~\citep{beltran2022velo2cam} & T-ITS & \xmark & \xmark & \xmark & \xmark & Empirical & Repeat. & \cmark & \xmark & \xmark\\
& LCE-Calib~\citep{jiao2023lcecalib} & T-MECH & \xmark & Indep. & \xmark & \xmark & \xmark & Ref. & \cmark & \xmark & \xmark\\
& \citet{persic2021spatiotemporal} & T-RO & Indep. & Indep. & \xmark & \xmark & Empirical & Repeat. & \cmark & \xmark & \xmark\\
& \citet{wiesmann2024joint} & RA-L & Indep. & Indep. & \xmark & Coupled & Empirical & Reproj. & \xmark & \cmark & \xmark\\
\multirow{-7}{*}{\rotatebox[origin=c]{90}{\small\textit{Target-based}}}& FAST-Calib (\hyperlink{cite.zheng2025fastcalib}{Zheng et al.}~\citeyear{zheng2025fastcalib}) & RA-P & \xmark & \xmark & \xmark & \xmark & \xmark & RMSE & \cmark & \xmark & \xmark\\
\midrule
& \citet{pandey2015automatic} & JFR & \xmark & \xmark & \xmark & Coupled & \xmark & Uncert. & \xmark & \xmark & \xmark\\
& \citet{maye2016self} & IJRR & Indep. & n/a & \xmark & Coupled & Info. & Uncert. & \xmark & \cmark & \xmark\\
& \hyperlink{cite.taylor2016motion}{Taylor et al.}~(\citeyear{taylor2016motion}) & T-RO & Indep. & Indep. & Est. & \xmark & \xmark & Ref. & \xmark & \cmark & \xmark\\
& \citet{yuan2021pixel} & RA-L & Indep. & Indep. & \xmark & Coupled & Empirical & Reproj. & \cmark & \xmark & \xmark\\
& OA-LICalib~\citep{lv2022observability} & T-RO & Indep. & n/a & \xmark & Coupled & Info. & Repeat. & \cmark & \cmark & \xmark\\
& TEScalib~\citep{hu2022tescalib} & IROS & Indep. & Indep. & \xmark & Coupled & \xmark & Ref. & \xmark & \cmark & \xmark\\
& DXQ-Net~\citep{jing2022dxq} & IROS & Indep. & \xmark & Learned & \xmark & \xmark & Ref. & \cmark & \xmark & \xmark\\
& \hyperlink{cite.moravec2024online}{Moravec et al.}~(\citeyear{moravec2024online}) & T-RO & \xmark & \xmark & Learned & \xmark & Learned & Valid. & \cmark & \cmark & \xmark\\
& CMRNext (\hyperlink{cite.cattaneo2025cmrnext}{Cattaneo et al.}~\citeyear{cattaneo2025cmrnext}) & T-RO & \xmark & \xmark & \xmark & \xmark & \xmark & Ref. & \cmark & \cmark & \xmark\\
& \citet{cocheteux2025uncertainty} & WACV & \xmark & \xmark & \xmark & \xmark & \xmark & Uncert. & \xmark & \xmark & \xmark\\
\multirow{-11}{*}{\rotatebox[origin=c]{90}{\small\textit{Targetless}}}& \citet{tahiraj2025calornocal} & IROS & \xmark & \xmark & \xmark & \xmark & \xmark & Valid. & \xmark & \xmark & \xmark\\
\midrule
\rowcolor{tabthird}[\dimexpr\tabcolsep+0.4pt\relax][\dimexpr\tabcolsep+0.4pt\relax] & \textbf{VeriCalib (ours)} & -- & \textbf{Joint} & \textbf{Joint} & \textbf{Est.} & \textbf{Decoupled} & \textbf{Info.} & \textbf{State} & \cmark & \cmark & \cmark\\
\bottomrule
\end{tabularx}
\par\vspace{1.5pt}
{\scriptsize\raggedright \textit{Note:} Indep./Joint: covariances independent or correlated within an observation; Est./Learned: noise magnitude estimated or learned; Coupled/Decoupled: rotation and translation analyzed jointly or with their compensation marginalized; Info./Empirical/Learned: pose selection by an information criterion, empirical advice, or learning; Reproj./Ref./Repeat./Uncert./Valid.: reprojection error, reference calibration, repeatability, covariance or interval, validity label; Within/Cross: fault test within or across observations; Detect. limit: smallest detectable fault; State: calibration state; n/a: no camera; $^\dagger$: tool variant.\par}
\end{table*}

\subsection{Target-Based Calibration}
\label{subsec:rw-target}
Target-based methods estimate the extrinsic from observations of a calibration object with known geometry. Planar checkerboards are the most common targets, whose plane constrains the extrinsic~\citep{mirzaei2012lidar} and whose edges and vertices yield line and point correspondences~\citep{zhou2018automatic}. Fiducial markers recover the board pose from a single image, and LiDARTag~\citep{huang2021lidartag} extends their detection to point clouds. Spherical targets~\citep{toth2020automatic,kummerle2020unified} and circular-hole boards instead provide features that both sensors localize directly. Four-hole boards are detected by velo2cam~\citep{beltran2022velo2cam}, fitted with a rigid hole-layout template by \citet{domhof2019extrinsic}, combined with a checkerboard in JointCalib~\citep{yan2022jointcalib}, and registered across mechanical and solid-state scan patterns by FAST-Calib~\citep{zheng2025fastcalib}. These known patterns provide a geometric prior for feature extraction and estimation~\citep{hu2026p2calib}, which targetless and increasingly learning-based methods do not exploit~\citep{pandey2015automatic,yuan2021pixel,koide2023general,petek2024automatic,zhang2025plk,an2024survey}.

\subsection{Noise and Uncertainty Modeling}
\label{subsec:rw-uncertainty}
The precision of an extrinsic depends on both the observation geometry and the measurement noise. Uncertainty-aware methods model the noise with analytic plane covariances~\citep{koo2020analytic} or probabilistic spherical-target fusion~\citep{kummerle2020unified}, but model the target features as independent. Since all camera features on one planar target are derived from a single Perspective-n-Point (PnP) pose, independent blocks count the pose information repeatedly and understate the uncertainty along correlated directions. LiDAR mapping propagates the point noise to the parameters of a fitted plane~\citep{yuan2022efficient}, whereas target-based calibration models each extracted LiDAR feature with an independent variance. Our previous work P$^2$Calib~\citep{hu2026p2calib} improves these features with pattern priors, which this paper propagates into their covariance (\cref{subsec:lidarcov}).

The absolute noise magnitude is rarely known. In LiDAR registration, closed-form covariances derived from white sensor noise~\citep{censi2007accurate} are smaller than the errors observed on real data~\citep{landry2019cello}, since pose-level errors and sensor biases dominate~\citep{brossard2020new,laconte2019lidar}. For the calibration result, Cram\'er-Rao bounds~\citep{brookshire2012extrinsic,pandey2015automatic} and Markov chain Monte Carlo posteriors~\citep{yamada2024probability} depend on assumed noise parameters, and conformal intervals~\citep{cocheteux2025uncertainty} are calibrated on held-out data but lack a physical covariance structure. Geodetic adjustment recovers the magnitude from the residuals by variance component estimation~\citep{koch1999parameter,teunissen2008least}, which \citet{ulrich2024uncertainty} applied to the observation groups of hand-eye calibration. Robot state estimation instead learns noise covariances from sensor data~\citep{liu2018deep}, which a single calibration session does not provide, and \citet{hagemann2022inferring} detect the systematic errors of camera calibration that a white-noise covariance does not represent. Most calibration methods that report a covariance adopt nominal magnitudes~\citep{brookshire2012extrinsic}, whose error propagates into the reported precision, the MDB, and every test referred to a nominal $\chi^2$ threshold. In this paper, we assemble the covariance shape of each observation into a single cross-modal residual covariance and estimate its scale from the residuals as one global variance component (\cref{sec:method}).

\subsection{Constraint Analysis}
\label{subsec:rw-observability}
Observability analysis determines which calibration parameters are constrained by the collected data, and it has been conducted for LiDAR-camera~\citep{mirzaei2012lidar,fu2019observability}, LiDAR-inertial measurement unit (IMU)~\citep{lv2022observability}, and visual-inertial calibration~\citep{yang2023online}. Several methods further select informative data segments or target poses by maximizing a scalar function of the information matrix~\citep{maye2016self,peng2019calibrationwizard,lv2022observability}. However, a calibration in which all parameters are observable may still constrain certain physical directions only weakly, and a scalar criterion does not indicate which directions these are.

Direction-resolved analysis has been studied mainly in LiDAR registration, where degeneracy detection examines the spectrum of the information matrix~\citep{zhang2016degeneracy,nobili2018predicting,hinduja2019degeneracy,hatles2025probab}. X-ICP~\citep{tuna2023x} analyzes the rotation and translation blocks separately, whereas our previous work DCReg~\citep{hu2025dcreg} eliminates their coupling through Schur complements, which has also been adopted for visual-inertial initialization~\citep{xu2026rotation}. These methods assume a nominal noise model and analyze registration rather than calibration. In this paper, we show that the rotation-translation compensation also arises in target-based calibration, as the parameter-level part of a two-level compensation (\cref{sec:factor}), and form the Schur complements from the structured covariance, so that their inverses are calibrated directional precisions (\cref{sec:method:obs}).

\subsection{Consistency Assessment and Reliability}
\label{subsec:rw-reliability}
Reliability theory originates from geodetic adjustment, where Baarda's data snooping~\citep{baarda1968testing} tests individual observations against the assumed model, and the associated MDB quantifies the smallest fault that the test detects at a specified significance level and power. Later studies formalized the distributions of these statistics~\citep{teunissen2018dia,imparato2018minimal,zaminpardaz2025conceptual}, and \citet{forstner1987reliability} applied the analysis to mensuration problems in computer vision, where the redundancy of each measurement determines the fraction of its error that the estimate absorbs. In robotics, integrity monitoring has quantified the localization safety of Kalman filtering~\citep{arana2020integrity}, fixed-lag smoothing, and simultaneous localization and mapping~\citep{hafez2020quantifying,hafez2025quantifying}. \citet{brunson2025analytical} applied the MDB to the pre-analysis of IMU error calibration in the integration of a global navigation satellite system with an IMU. In these applications, faults are defined on individual measurements or landmarks, and no target-based calibration method reports a detectability limit along the directions of the extrinsic.

Calibration results are also checked by repeatability, agreement with a reference calibration, or visual inspection, and none of these assessments provides a threshold calibrated to the noise of the observations. Fault tests in calibration act at two levels (\cref{tab:related}). Within-observation tests reject a board by geometric thresholds or RANSAC on its own features~\citep{beltran2022velo2cam,zheng2025fastcalib}, and cross-observation tests reject a board whose residual against the joint estimate exceeds a robust threshold~\citep{fischler1981ransac,yang2020graduated,wiesmann2024joint}, although a biased observation of high leverage pulls the fit toward itself and leaves a small residual. The learned detector of \citet{tahiraj2025calornocal} instead classifies the whole calibration without a controlled false-alarm rate. Certifiable algorithms~\citep{yang2020teaser,yang2023certifiable,wise2026certifiably} and globally optimal estimators of other inter-sensor misalignments~\citep{hu2025amalign} certify the optimality of the estimate for a given model, but do not assess whether the measurements are consistent with that model.

As \cref{tab:related} shows, a covariance modeled by a prior method is independent across features and mostly of fixed or learned scale, constraints are analyzed on the coupled six degrees of freedom, and no prior method reports a detection limit, so that the three challenges are addressed by separate literatures, each assuming the noise magnitude, the nominal noise model, or the calibrated covariance that another must supply. We therefore formulate all three analyses on a single structured measurement factor (\cref{sec:factor}).

\section{Problem Formulation}
\label{sec:formulation}
\label{sec:prelim}

In the target-based calibration problem, we are given $K$ observations of a known planar target by a LiDAR and a camera, and the goal is to estimate the rigid transform between the two sensors. 

\subsection{Target-Based Calibration}
\label{subsec:extrinsic}

We denote the LiDAR, camera, and board frames by $\calL$, $\calC$, and $\calB$, the transform from frame $\mathtt B$ to frame $\mathtt A$ by $\bm T_{\mathtt{AB}}\in\SE(3)$, and an estimate by a hat $\hat{(\cdot)}$. The unknown extrinsic $\bm T_{\mathtt{CL}}=(\bm R,\bm t)$ maps LiDAR coordinates to camera coordinates, $\delta\bm\xi=(\delta\bm\phi,\delta\bm t)\in\R^6$ denotes a left perturbation on $\SE(3)$, and the subscripts $k$ and $j$ index the observations and the holes within each observation.

The point coordinates and the board poses $\bm T_{\mathtt{CB},k}$ and $\bm T_{\mathtt{LB},k}$ satisfy
\begin{equation}
\bm x^{\mathtt{C}}=\bm R\,\bm x^{\mathtt{L}}+\bm t,\qquad
\bm T_{\mathtt{CB},k}=\bm T_{\mathtt{CL}}\,\bm T_{\mathtt{LB},k},
\label{eq:extrinsic}
\end{equation}
where the second equation states that both sensors observe the same board pose.

\label{subsec:fourhole}
The target consists of $m=4$ circular holes of known radius $r$ at board coordinates $\bm a_j\in\{(\pm a,\pm b,0)\}$, with hole spacings $2a$ and $2b$, whose centers both sensors localize and thus form cross-modal correspondences of known geometry. The formulation also applies to any target with $m\ge3$ non-collinear points of known board coordinates. We refer to each recorded board placement as a target observation and to its acquisition geometry as its target pose.

The classical solver extracts the hole centers in each sensor independently and then aligns them. In the camera, PnP on the detected fiducials yields the board pose $(\bm Q_k,\bm c_k)$ with a $6\times6$ covariance $\bm\Sigma_{\theta,k}$, and the hole centers follow as $\bm p^{\mathtt{C}}_{kj}=\bm c_k+\bm Q_k\bm a_j$. In the LiDAR, the board plane is first fitted by RANSAC, and then the hole boundaries in the plane region are fitted by circles with centers $\bm p^{\mathtt{L}}_{kj}$. When all correspondences $(\bm p^{\mathtt{L}}_{kj},\bm p^{\mathtt{C}}_{kj})$ are correct and the noise is zero-mean, isotropic, Gaussian, and independent across correspondences, the maximum-likelihood extrinsic solves
\begin{equation}
(\bm R,\bm t)=
\arg\min_{\bm R\in\SO(3),\,\bm t}
\sum_{k=1}^{K}\sum_{j=1}^{4}
\lVert\bm R\bm p^{\mathtt{L}}_{kj}+\bm t-\bm p^{\mathtt{C}}_{kj}\rVert^2.
\label{eq:baseline}
\end{equation}
\cref{eq:baseline} has a closed-form global minimizer by orthogonal Procrustes alignment. A single observation ($K{=}1$) already constrains all six degrees of freedom through its four correctly associated, non-collinear hole centers, and multi-observation implementations solve~\cref{eq:baseline} jointly~\citep{beltran2022velo2cam,zheng2025fastcalib}. The classical solver reports the RMSE (\cref{fig:overview}, left). On real target observations, the four camera-side hole centers share one PnP pose, the LiDAR circles are fitted on partial arcs within one fitted board plane, and individual extractions can be grossly wrong, so that these assumptions do not hold in general. \cref{sec:method} models the first two effects in the covariance, and \cref{sec:method:tests} detects the third.

\subsection{Observation Model}
\label{subsec:obsmodel}
Throughout the paper, we analyze the observation residual, obtained by stacking the four correspondence residuals of observation~$k$:
\begin{equation}
\bm r_k(\bm R,\bm t)=\bigl[\,\bm R\bm p^{\mathtt{L}}_{kj}+\bm t-\bm p^{\mathtt{C}}_{kj}\,\bigr]_{j=1}^{4}\in\R^{12}.
\label{eq:residual}
\end{equation}
In \cref{sec:method:cov}, we generalize~\cref{eq:baseline} to the covariance-weighted estimator on the accepted set $\calS$, the subset of observations retained by the consistency assessment (\cref{sec:method:tests}),
\begin{equation}
\hatT_{\mathtt{CL}}=\arg\min_{\bm T\in\SE(3)}\sum_{k\in\calS}\bm r_k(\bm T)^\top\bigl(\bm\Sigma^r_k\bigr)^{-1}\bm r_k(\bm T),
\label{eq:problem}
\end{equation}
by assigning a structured covariance $\bm\Sigma^r_k$ to each residual. Under the left perturbation, the correspondence Jacobian is
\begin{equation}
\bm T_{\mathtt{CL}}=\Exp(\delta\bm\xi^{\wedge})\hatT_{\mathtt{CL}},\qquad
\bm J_{kj}=\bigl[-\skewop{\bm x_{kj}}\;\;\bm I_3\bigr],
\label{eq:jacobian}
\end{equation}
where $\bm x_{kj}=\hat{\bm R}\bm p^{\mathtt{L}}_{kj}+\hat{\bm t}$, and $\bm J_k=[\bm J_{k1};\dots;\bm J_{k4}]\in\R^{12\times6}$ stacks the four correspondences. Rotation and translation are analyzed as distinct physical quantities and are never combined in a common unit.\footnote{The perturbation $\delta\bm t$ is the displacement of the camera-frame origin, in which all translation precisions and requirements are stated, and it differs from the difference of two translation vectors $\bm t$ by $\delta\bm\phi\times\bm t$.} The checkerboard, whose correspondences are its four outer corners, differs only in its target-specific covariance (\cref{sec:checkerboard}).

Linearized at the true extrinsic $\bm T^\circ_{\mathtt{CL}}$, the observation residual satisfies the model, with $\doteq$ denoting first-order equality:
\begin{equation}
\begin{gathered}
\bm r_k\bigl(\Exp(\delta\bm\xi^\wedge)\bm T^\circ_{\mathtt{CL}}\bigr)\doteq\bm J_k\delta\bm\xi+\bm b_k+\bm\varepsilon_k,\\
\bm\varepsilon_k\sim\mathcal N(\bm0,\bm\Sigma^r_k),
\end{gathered}
\label{eq:obsmodel}
\end{equation}
where $\bm\varepsilon_k$ models the measurement noise, whose covariance is constructed in \cref{sec:method}, and $\bm b_k$ is zero for a consistent observation and a fault vector otherwise. A local fault, such as a mis-extracted hole or a wrong correspondence, yields a $\bm b_k$ with a component that no rigid perturbation explains, in addition to an absorbed rigid part (\cref{subsec:factorrep}), whereas a differential fault displaces one group of observations relative to the others. A common-mode fault $\bm b_k=\bm J_k\delta\bm\xi_0$, shared by all observations, is indistinguishable from a different extrinsic. A batch denotes the observations of one calibration run.

\subsection{The Verifiable Calibration Problem}
\label{subsec:quality}
\label{subsec:scope}
\label{sec:formulation:scope}
Formally, we pose the following problem.
\begin{problem}[Verifiable calibration]
\label{prob:verifiable}
Given $K$ target observations, a precision requirement $(\epsilon_R,\epsilon_t)$, and a detectable-fault requirement $(\delta_R,\delta_t)$, estimate the extrinsic and its covariance from the observations consistent with~\cref{eq:obsmodel}, and determine whether 
(i) the accepted observations pass the consistency tests at level $\alpha$, (ii) the precision along every rotation and translation direction meets $(\epsilon_R,\epsilon_t)$, and (iii) every differential fault larger than $(\delta_R,\delta_t)$ is detected with power $1-\beta$ by a validation test on observations held out from the estimation.
\end{problem}
The result is reported as a calibration state, which indicates the first unmet requirement and is paired with a recommended procedure (\cref{subsec:status}).
Since every test in \cref{sec:method:tests} compares observations within the same batch, the verifiable scope is internal, as stated below.
\begin{remark}[Internally unobservable faults]
\label{prop:unobservable}
Let each observation contain the same common-mode fault $\bm b_k=\bm J_k\delta\bm\xi_0$ in \cref{eq:obsmodel}. Then, to first order, the substitution $\delta\bm\xi'=\delta\bm\xi+\delta\bm\xi_0$ reduces the weighted estimation on every observation subset to the problem without bias, and every statistic computed from fitted, cross-validated, or cross-subset residuals is invariant (\cref{app:proofs}). Such faults include the rigid-equivalent component of an intrinsic error or of a target deformation, and a time offset under constant platform motion. They displace the calibration and validation observations equally, even when the validation information has full rank.
\end{remark}
\noindent Beyond \cref{eq:obsmodel}, we make the following assumption.
\begin{assumption}[Local factor model]
\label{ass:model}
(i) The residuals are locally linearizable around $\hatT_{\mathtt{CL}}$ and conditionally Gaussian. (ii) Each derived covariance shape is correct up to one common scale, estimated as one global variance component (\cref{subsec:scaleest}); the estimate and its covariance require this only on the absorbed subspace of \cref{subsec:factorrep} (\cref{prop:scale}), and the within-view and batch tests also on the internal subspace. The accepted residuals have zero conditional mean. (iii) The residuals of different observations are independent, including those of the validation and calibration sets.
\end{assumption}
For simplicity, the covariance $\bm\Sigma^{r}_k$ is frozen at the current estimate, so that its $\bm R$-dependence and the log-determinant of the likelihood are not differentiated, and the likelihood becomes a quasi-likelihood.

\section{Two-Level Extrinsic Compensation}
\label{sec:factor}
This section illustrates the two levels extrinsic compensation mechanism on the factor $\calF_k=(\bm r_k,\bm\Sigma^r_k,\bm J_k)$ : within each observation, where the absorption hides an inconsistency with the model (\cref{subsec:factorrep}), and between rotation and translation, where it hides how reliably each direction is constrained (\cref{subsec:localinfo}, \cref{fig:compensation}).

\begin{figure*}[!t]\centering\includegraphics[width=\textwidth]{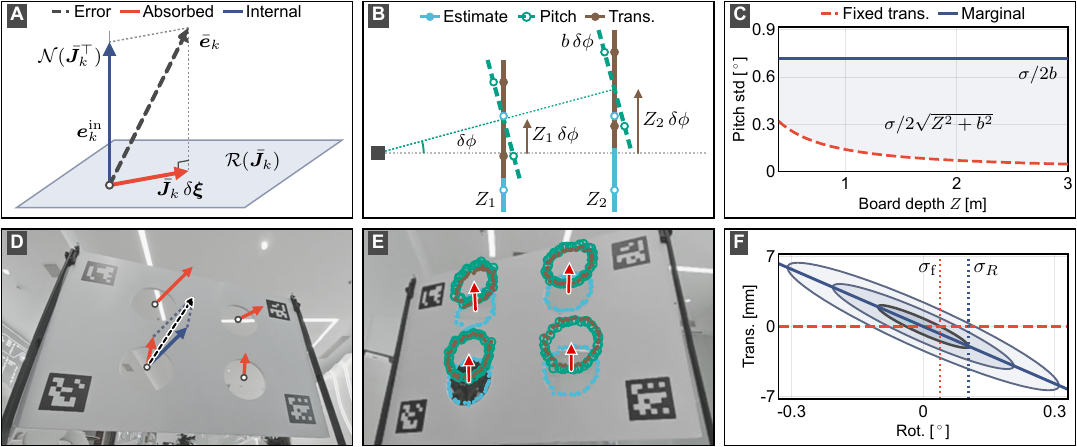}
\caption{Two-level extrinsic compensation (\textbf{A}--\textbf{C}: principle; \textbf{D}--\textbf{F}: chassis-pair data). \textbf{A}, \textbf{D}: a hole-center error split into absorbed and internal parts. \textbf{B}, \textbf{E}: a pitch and the compensating translation yield nearly the same target observation. \textbf{C}, \textbf{F}: rotation precision with the translation fixed ($\sigma_{\mathrm f}$) and marginalized ($\sigma_R$); C: $\sigma=5$\,mm, $b=0.2$\,m.}
\label{fig:compensation}\end{figure*}

\subsection{Residual-Level Compensation}
\label{subsec:factorrep}
Testing an observation without an external reference requires the part of its error that the estimate cannot absorb. At the residual level, the whitened residual space of each observation separates into two orthogonal subspaces (\cref{fig:compensation}A,D). An infinitesimal extrinsic perturbation acts on the first and leaves the second invariant, so that a discrepancy in the first is absorbed by the fitted transform, whereas a discrepancy in the second remains testable within the observation.
Target observation $k$ is represented by the factor
\begin{equation}
\calF_k=\bigl(\bm r_k(\bm T_{\mathtt{CL}}),\,\bm\Sigma^r_k,\,\bm J_k\bigr),
\label{eq:factor}
\end{equation}
whose negative log-likelihood is $\tfrac12\bm r_k^\top(\bm\Sigma^r_k)^{-1}\bm r_k$ up to a constant. We refer to $\bm\Sigma^r_k$, the covariance of the $12$-dimensional hole-center alignment residual under the stated measurement noise model at the true extrinsic, as the observation residual covariance. It is distinct from the estimator covariance $\bm H^{-1}$ in~\cref{eq:estcov} and from the deflated fitted-residual covariance of \cref{eq:deflation}.

Let $\bm L_k^\top\bm L_k=(\bm\Sigma^r_k)^{-1}$ be a whitening factor on the $d_k$-dimensional support of the residual covariance, $\bar{\bm r}_k=\bm L_k\bm r_k$, and $\bar{\bm J}_k=\bm L_k\bm J_k$. For a non-degenerate four-hole observation with $\rank(\bar{\bm J}_k)=6$ and $d_k=11$, since the saddle mode of the four coplanar centers receives no measurement (\cref{subsec:lidarcov,subsec:residcov}),
\begin{equation}
\R^{d_k}=\underbrace{\rangeop(\bar{\bm J}_k)}_{\dim=6}\;\oplus\;\underbrace{\nullop(\bar{\bm J}_k^\top)}_{\dim=d_k-6}.
\label{eq:structured_split}
\end{equation}
If $\bm N_k$ is an orthonormal basis of $\nullop(\bar{\bm J}_k^\top)$, the internal residual and statistic are
\begin{equation}
\begin{gathered}
\bm e^{\mathrm{in}}_k=\bm N_k^\top\bm L_k\bm r_k,\qquad
\bm N_k^\top\bm L_k\bm J_k=\bm0,\\
T^{\mathrm{in}}_k=\lVert\bm e^{\mathrm{in}}_k\rVert^2\sim\chi^2_{d_k-6}
\end{gathered}
\label{eq:inttest}
\end{equation}
under the local nominal model with the scale $s$ of \cref{subsec:scaleest} in $\bm\Sigma^r_k$, and the following result formalizes the separation.
\begin{proposition}[Residual-level compensation]
\label{prop:residual}
Let $\bm b_k$ in \cref{eq:obsmodel} be the fault of a non-degenerate observation, and let $\bar{\bm b}_k=\bm L_k\bm b_k$. The extrinsic perturbation that minimizes $\lVert\bar{\bm b}_k+\bar{\bm J}_k\delta\bm\xi\rVert$ cancels the component of $\bar{\bm b}_k$ in $\rangeop(\bar{\bm J}_k)$ exactly, whereas the internal residual $\bm e^{\mathrm{in}}_k$ is invariant to first order under every $\delta\bm\xi$. Consequently, $T^{\mathrm{in}}_k$ in \cref{eq:inttest} retains its null distribution for any absorbed component.
\end{proposition}
\noindent The proof is given in \cref{app:proofs}.
For a general target observation with residual dimension $d_k$, the internal dimension is $d_k-\rank(\bar{\bm J}_k)$, which is five for the four-hole target. These five dimensions are the non-rigid in-plane deformations of the four hole centers. Since the camera error and the error of the fitted LiDAR board plane move the four centers rigidly (\cref{subsec:camcov,subsec:lidarcov}), the internal residual of a four-hole observation isolates the in-plane LiDAR extraction error. In a batch, the absorbed part of a fault displaces the common estimate and is distributed over all residuals, so that the fitted residual neither reveals the error of the extrinsic nor locates the corrupted observations (\cref{fig:teaser}). 

\subsection{Parameter-Level Compensation}
\label{subsec:localinfo}
A rotation precision is meaningful only if it accounts for the translation that can compensate the rotation. The residual-level split has a counterpart in the parameter space, where a rotation and a translation compensate each other. Under the frozen-covariance Gaussian approximation, the Gauss-Newton information of the weighted cost of \cref{eq:problem} on the accepted set coincides with the local Fisher information, which we call the local information matrix,
\begin{equation}
\bm H_k=\bm J_k^\top(\bm\Sigma^{r}_k)^{-1}\bm J_k,\qquad
\bm H=\sum_{k\in\calS}\bm H_k=\begin{bmatrix}\bm A&\bm B\\\bm B^\top&\bm C\end{bmatrix},
\label{eq:localinfo}
\end{equation}
with rotation block $\bm A$, translation block $\bm C$, and coupling block $\bm B$. With the residuals linearized at the true extrinsic, the error of the weighted estimate on the accepted set and its covariance are
\begin{equation}
\bm\Sigma_{\delta\hat{\bm\xi}}\doteq\bm H^{-1},
\qquad
\delta\hat{\bm\xi}=-\bm H^{-1}\sum_{k\in\calS}\bm J_k^\top(\bm\Sigma^r_k)^{-1}\bm r_k.
\label{eq:estcov}
\end{equation}

A pitch $\delta\phi$ about a lateral axis displaces a board at depth $Z$ by approximately $Z\delta\phi$, which a translation reproduces up to an offset $b\,\delta\phi$ independent of $Z$, with $b$ the lateral distance of a hole from the board center (\cref{fig:compensation}B), and the rim points projected after a $5^\circ$ pitch and after the equivalent translation nearly coincide (\cref{fig:compensation}E). The off-diagonal block $\bm B$ in \cref{eq:localinfo} encodes this coupling (\cref{fig:compensation}F), and an aligned projection of the target therefore does not establish the rotation. Holding the translation fixed yields the conditional rotation information $\bm A$, whereas leaving it free yields, by the block inversion of $\bm H$, the rotation-marginal covariance $[\bm H^{-1}]_{\phi\phi}=(\bm A-\bm B\bm C^{-1}\bm B^\top)^{-1}$. The marginal information matrices are therefore the Schur complements adopted in DCReg~\citep{hu2025dcreg}, here formed from the structured covariance of \cref{sec:method}
\begin{equation}
\LamR=\bm A-\bm B\bm C^{-1}\bm B^\top,\qquad
\LamT=\bm C-\bm B^\top\bm A^{-1}\bm B .
\label{eq:lamR}
\end{equation}

\begin{proposition}[Parameter-level compensation]
\label{prop:parameter}
For a full-rank $\bm H$, $\LamR\preceq\bm A$ and $\LamT\preceq\bm C$, with equality if and only if $\bm B=\bm 0$. The difference $\bm B\bm C^{-1}\bm B^\top$ is the rotation information that a translation perturbation can compensate, and $\bm B^\top\bm A^{-1}\bm B$ the translation information that a rotation perturbation can compensate.
\end{proposition}
\noindent The proof is given in \cref{app:proofs}. A scalar summary of $\bm H$ combines the two subspaces and their different units $\mathrm{rad}^{-2}$ and $\mathrm{m}^{-2}$, and cannot attribute a weak direction to rotation or translation. \cref{prop:parameter} thus states that a rotation precision computed with the translation held fixed overstates the precision of a calibration in which the translation is estimated, as the smaller fixed-translation standard deviation $\sigma_{\mathrm f}$ of \cref{fig:compensation}F shows on a chassis-pair session.

\section{Structured Covariance Modeling}
\label{sec:method}
\label{sec:method:cov}
Since a fitted residual neither reveals the error of the extrinsic nor locates a corrupted observation (\cref{sec:factor}), the precision and the test thresholds are evaluated from the information of the factors, which requires a covariance $\bm\Sigma^r_k$ (\cref{fig:overview}, blue ellipses) of correct shape and scale. In this section, we first quantify how the covariance weighting determines the precision of the extrinsic (\cref{subsec:shapescale}). We then derive the shape of the camera and LiDAR blocks (\cref{subsec:camcov,subsec:lidarcov}), combine them into the joint residual covariance (\cref{subsec:residcov}), estimate its scale in two stages (\cref{subsec:scaleest}), and finally adapt the factor construction to a checkerboard target (\cref{sec:checkerboard}).

\subsection{Covariance Weighting Theorem}
\label{subsec:shapescale}
The precision of the extrinsic follows from the observation geometry through $\bm J_k$ and from the measurement noise through $\bm\Sigma^r_k$ (\cref{eq:localinfo}), and an error in the covariance shape or scale degrades the estimate or its reported precision. We write each covariance as $\bm\Sigma^r_k=s\,\bm\Sigma^r_{0,k}$, with a shape $\bm\Sigma^r_{0,k}$ and a common scale $s>0$, and analyze the cost of an error in each. Suppose the estimator of \cref{eq:estcov} is instead weighted by $\tilde{\bm\Sigma}^r_k$ while the residuals follow $\bm\Sigma^r_k$, and write $\tilde{\bm H}=\sum_k\bm J_k^\top(\tilde{\bm\Sigma}^r_k)^{-1}\bm J_k$ for the resulting information matrix. Propagating the true residual covariance through this suboptimal linear estimator yields the sandwich covariance
\begin{equation}
\bm V(\tilde{\bm\Sigma})
=\tilde{\bm H}^{-1}\Bigl(\sum_{k}\bm J_k^\top(\tilde{\bm\Sigma}^r_k)^{-1}\bm\Sigma^r_k(\tilde{\bm\Sigma}^r_k)^{-1}\bm J_k\Bigr)\tilde{\bm H}^{-1},
\label{eq:sandwich}
\end{equation}
so that $\bm V(\bm\Sigma)=\bm H^{-1}$, and $\bm V_{\mathrm u}=\bm V(\bm I)$ is the covariance of the unweighted solver~\cref{eq:baseline}. Let $\delta\hat{\bm\xi}$ and $\delta\hat{\bm\xi}_{\mathrm u}$ denote the errors of the correctly weighted and of the unweighted estimate. Weighting improves the estimate only where the observations differ in precision, and only if the weights follow these differences. The following theorem quantifies both conditions to first order.
\begin{theorem}[Weighting efficiency]
\label{thm:gain}
Assume (i) the residuals follow~\cref{eq:obsmodel} with $\bm b_k=\bm 0$, independently across observations, (ii) the weights are fixed and $\tilde{\bm\Sigma}^r_k\succ\bm 0$ on the support of $\bm\Sigma^r_k$, and (iii) $\rank(\bm H)=6$. Then:\\
(a) the covariance reduction of the correct weights is
\begin{equation}
\bm V_{\mathrm u}-\bm H^{-1}=\operatorname{Cov}\bigl(\delta\hat{\bm\xi}_{\mathrm u}-\delta\hat{\bm\xi}\bigr)\succeq\bm 0,
\label{eq:gain}
\end{equation}
with equality if and only if $\rangeop(\bm\Sigma\bm J)\subseteq\rangeop(\bm J)$;\\
(b) if $\lambda(\bm\Sigma^r_k,\tilde{\bm\Sigma}^r_k)\in[\mu,\omega\mu]$ for all $k$, then
\begin{equation}
\bm H^{-1}\preceq\bm V(\tilde{\bm\Sigma})\preceq\frac{(1+\omega)^2}{4\omega}\,\bm H^{-1},
\label{eq:kantorovich}
\end{equation}
and for a common factor $\tilde{\bm\Sigma}^r_k=c\,\bm\Sigma^r_k$ the estimate coincides with the correctly weighted one, and only the reported covariance $\tilde{\bm H}^{-1}=c\,\bm H^{-1}$ is scaled by $c$;\\
(c) for Gaussian noise and every unit direction $\bm v$,
\begin{equation}
\begin{gathered}
\Pr\bigl(\lvert\bm v^\top\delta\hat{\bm\xi}\rvert<\lvert\bm v^\top\delta\hat{\bm\xi}_{\mathrm u}\rvert\bigr)=\frac12+\frac1\pi\arctan\frac{\sqrt{\vartheta_{\bm v}-1}}{2},\\
\vartheta_{\bm v}=\frac{\bm v^\top\bm V_{\mathrm u}\bm v}{\bm v^\top\bm H^{-1}\bm v}.
\end{gathered}
\label{eq:winprob}
\end{equation}
Here $\bm\Sigma$ and $\bm J$ stack the covariances and Jacobians of all observations, and $\lambda(\cdot,\cdot)$ denotes the generalized eigenvalues on the support.
\end{theorem}
\noindent The proof is given in \cref{app:precision}. Parts (a) and (b) adapt the efficiency of generalized least squares~\citep{aitken1935least} and the matrix Kantorovich inequality~\citep{marshall1990matrix} to the structured factor. The covariance reduction of part (a) is the scatter of the unweighted estimate about the weighted one. It vanishes for homogeneous observations, e.g., $\bm\Sigma^r_k=\sigma^2\bm I$ with a common $\sigma^2$, and grows with the spread of the covariances across observations and directions. In part (b), a shape error, unlike a common factor, alters the relative weights and the directional analysis of \cref{sec:method:obs}, and loses at most the factor $(1+\omega)^2/(4\omega)$ of variance, which is $12.5\%$ for $\omega=2$. The lower bound is attained exactly when the weighted estimator is the best linear unbiased one~\citep{zyskind1967canonical,rao1967least}, which includes a covariance error confined to the internal subspace of \cref{eq:structured_split}. Omitting the cross-hole blocks of the camera covariance violates this condition (\cref{subsec:expcov}). For the unweighted solver, $\omega$ is the condition number of $\bm\Sigma$ on its support. Hence, a derived shape with ratio $\tilde\omega$ improves on the unweighted solver along $\bm v$ whenever $\vartheta_{\bm v}>(1+\tilde\omega)^2/(4\tilde\omega)$. Part (c) concerns a single calibration. At $\vartheta_{\bm v}=2$ the weighted estimate is the closer one in $65\%$ of the calibrations, and \cref{subsec:expcov} therefore evaluates the improvement over repeated trials and splits.

\cref{thm:gain} thus separates the two parts of the covariance. The shape determines the estimate but cannot be estimated from the single residual of each observation, and is therefore derived from the measurement process of each sensor (\cref{subsec:camcov,subsec:lidarcov}). The scale is one parameter of the batch, estimated from its residuals (\cref{subsec:scaleest}).

In addition, an added observation consistent with the model never decreases the precision of the correctly weighted estimate, since $(\bm H+\bm H_{K+1})^{-1}\preceq\bm H^{-1}$. No such property holds for the unweighted solver, since the unweighted mean of $K$ scalar measurements loses precision when a measurement whose variance exceeds $(2+1/K)$ times their mean variance is added (\cref{app:precision}). Additional observations therefore cannot degrade the weighted estimate, provided that they are consistent with the model, and the weakest-direction score of \cref{subsec:guidance} presupposes this property.
\begin{remark}[Faults on precise observations]
\label{rem:scope}
A fault $\bm b_k$ displaces the weighted estimate by $-\bm H^{-1}\bm J_k^\top(\bm\Sigma^r_k)^{-1}\bm b_k$ to first order, which grows with the weight of the observation, so that a fault on a precise observation can make the weighted estimate less accurate than the unweighted one. The same covariance determines the thresholds of \cref{sec:method:tests}, which compare each deviation with the noise of its own observation, and the weighting and the consistency assessment are therefore applied together.
\end{remark}

\subsection{Shared-Pose Camera Covariance}
\label{subsec:camcov}
\begin{figure}[!t]\centering\includegraphics[width=\columnwidth]{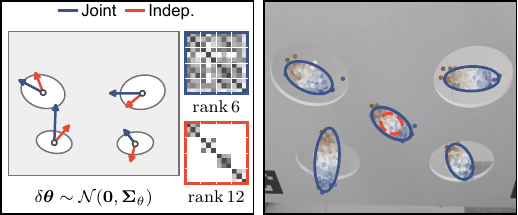}
\caption{Shared-pose camera covariance. \textbf{Left}: $3\sigma$ hole displacements and correlations of the shared-pose (Joint) and independent-hole (Indep.) models. \textbf{Right}: hole centers of PnP solutions from noise-perturbed marker corners on a head-pair observation ($\times60$), with $3\sigma$ ellipses of each hole and of their centroid (dashed: Indep.).}
\label{fig:covcamera}\end{figure}
Uncertainty-aware calibration models the camera features as independent, although all camera features on one board are computed from one board pose. We therefore propagate the covariance of this pose to the four hole centers, which are reconstructed from a single estimated board pose $\bm\theta_k=(\bm\phi_k,\bm c_k)$ returned by PnP, through the rigid map $\bm p^{\mathtt{C}}_{kj}=\bm c_k+\bm Q_k(\bm\phi_k)\bm a_j$. A first-order perturbation of the pose, $\delta\bm\theta=(\delta\bm\phi,\delta\bm c)$, displaces hole $j$ by
\begin{equation}
\delta\bm p^{\mathtt{C}}_{kj}=\delta\bm c+\skewop{\delta\bm\phi}\,(\bm Q_k\bm a_j)
=\underbrace{[-\skewop{\bm Q_k\bm a_j}\;\;\bm I_3]}_{\bm F_{kj}}\,\delta\bm\theta,
\label{eq:camjac}
\end{equation}
so these six pose errors displace all four holes through the lever arms $\bm Q_k\bm a_j$. Stacking $\bm F_k=[\bm F_{k1};\dots;\bm F_{k4}]\in\R^{12\times6}$ and propagating the $6\times6$ PnP covariance $\bm\Sigma_{\theta,k}$ yields the camera block
\begin{equation}
\bm\Sigma^{\mathtt{C}}_k=\bm F_k\,\bm\Sigma_{\theta,k}\,\bm F_k^\top,\qquad\rank(\bm\Sigma^{\mathtt{C}}_k)\le6 .
\label{eq:camcov}
\end{equation}
The pose covariance itself is obtained from the image measurement model. If $\bm u_{ki}$ are the detected fiducial or corner pixels, $\bm h_i(\bm\theta_k)$ their perspective projection, $\bm G_k=\partial\bm h/\partial\bm\theta$ the PnP Jacobian, and $\bm\Sigma^u_k$ the pixel covariance, then the Gauss-Newton covariance is
\begin{equation}
\bm\Sigma_{\theta,k}=\bigl(\bm G_k^\top(\bm\Sigma^u_k)^{-1}\bm G_k\bigr)^{-1}.
\label{eq:pnpcov}
\end{equation}
In all experiments we set the pixel covariance from a nominal marker-corner noise, and its absolute magnitude is calibrated by the global scale of \cref{subsec:scaleest}. \cref{eq:camcov} then yields the cross-hole covariance $[\bm\Sigma^{\mathtt{C}}_k]_{ij}=\bm F_{ki}\bm\Sigma_{\theta,k}\bm F_{kj}^\top$, which is generally nonzero for every pair $(i,j)$. Twelve coordinates therefore have only six independent error degrees of freedom, and a single pose perturbation displaces all four holes by the same rigid transformation, so that the perturbed PnP solutions of the four holes vary together and the scatter of their centroid exceeds its independent-hole ellipse (\cref{fig:covcamera}). Since a perturbation of the extrinsic also displaces the four centers rigidly, the camera error of an observation lies entirely in the six-dimensional absorbed subspace $\rangeop(\bar{\bm J}_k)$ of \cref{prop:residual}. It displaces the estimate without appearing in the internal residual, and only a comparison with the other observations reveals it. The common alternative, retaining only the four diagonal $3\times3$ blocks of~\cref{eq:camcov}, discards the off-diagonal lever-arm correlations and thereby counts the PnP information once for each hole.

\subsection{Shared-Plane LiDAR Covariance}
\label{subsec:lidarcov}
A LiDAR hole center is fitted from few boundary points on a partial arc, and the four centers of an observation lie on one fitted board plane. We therefore propagate the edge-point noise through a radius-regularized circle fit and through the shared plane. Edge-point noise on the board plane follows the two-parameter range-incidence model
\begin{equation}
\sigma^{2}_{L}(z,\eta)=\sigma^{2}_{L,0}+\sigma^{2}_{L,1}\,z^{2}/\cos^{2}\eta ,
\label{eq:lidarnoise}
\end{equation}
in which the ranging variance grows with distance $z$ and grazing incidence $\eta$. The noise floor $\sigma_{L,0}$ and the dimensionless range-growth coefficient $\sigma_{L,1}$ are fixed from the sensor datasheet and a one-time edge-residual calibration (\cref{sec:results:setup}). A beam covariance for each LiDAR return, as adopted in LiDAR odometry back ends~\citep{yuan2022efficient,zheng2024fast}, yields similar conclusions (\cref{app:lidarnoise}).
We fit a hole center to the $N$ extracted boundary points $\{\bm b_l\}$. On a short boundary arc the free three-parameter circle fit is ill-conditioned, since a radial displacement of the center is correlated with a variation of the radius~\citep{hu2026p2calib}. The effective radius is known from target design or a separate sensor-class calibration. It is therefore incorporated as a prior $\varrho\sim\mathcal N(r,\sigma_r^2)$ on the joint center-radius fit of each hole, which regularizes the fitted center and the conditional covariance of that hole under the nominal radius. With radial unit vectors $\hat{\bm u}_l=(\bm b_l-\bm c)/\lVert\bm b_l-\bm c\rVert$, the signed distance of $\bm b_l$ to the circle is $\hat{\bm u}_l^\top(\bm b_l-\bm c)-\varrho$, whose gradient in $(\bm c,\varrho)$ is $[-\hat{\bm u}_l^\top\;\;{-}1]$. As an accumulated cloud re-observes the same physical edge, the $N$ boundary points are correlated and provide the independent information of only
\begin{equation}
N_{\mathrm{eff}}=\frac{N}{1+(N-1)\,\bar\rho}
\label{eq:neff}
\end{equation}
effective points, with a fixed equicorrelation $\bar\rho$ (\cref{sec:results:setup}). The center and radius are estimated jointly by minimizing
\begin{equation}
\frac{N_{\mathrm{eff}}}{N\sigma_L^2}\sum_l\bigl(\lVert\bm b_l-\bm c\rVert-\varrho\bigr)^2
+\frac{(\varrho-r)^2}{\sigma_r^2},
\label{eq:circlemap}
\end{equation}
so that the prior is weighted against the same likelihood that defines the covariance, and the radial unit vectors are evaluated at the fitted center. The joint $(\bm c,\varrho)$ information of~\cref{eq:circlemap} is
\begin{equation}
\bm I_{(\bm c,\varrho)}=\frac{N_{\mathrm{eff}}}{N\sigma_L^{2}}\begin{bmatrix}\sum_l\hat{\bm u}_l\hat{\bm u}_l^\top & \sum_l\hat{\bm u}_l\\[2pt] \sum_l\hat{\bm u}_l^\top & N\end{bmatrix}+\begin{bmatrix}\bm0&\bm0\\\bm0&\sigma_r^{-2}\end{bmatrix},
\label{eq:circinfo}
\end{equation}
whose radius entry is $N_{\mathrm{eff}}/\sigma_L^2+\sigma_r^{-2}$. Marginalizing the nuisance radius $\varrho$ by the Schur complement leaves the in-plane center covariance
\begin{equation}
\left\{
\begin{aligned}
&\bm\Sigma^{\mathtt{L},\mathrm{pl}}_{kj}=\frac{N}{N_{\mathrm{eff}}}\,\sigma_L^{2}\Bigl(\textstyle\sum_l\hat{\bm u}_l\hat{\bm u}_l^\top\\
&\qquad\qquad-\alpha_{r,kj}\,\textstyle(\sum_l\hat{\bm u}_l)(\sum_l\hat{\bm u}_l)^\top\Bigr)^{-1},\\
&\alpha_{r,kj}=\frac{N_{\mathrm{eff}}}{N\bigl(N_{\mathrm{eff}}+\sigma_L^{2}\sigma_r^{-2}\bigr)},
\end{aligned}
\right.
\label{eq:radiuscov}
\end{equation}
which is rotated to $\calL$ and completed by the out-of-plane error. The four hole centers lie on the single board plane fitted to all $N_k$ boundary points of the observation. We therefore model their out-of-plane error as the error of that plane, namely one offset and two tilts shared by the four centers. Let $\bm U_k=[\bm t_{k1}\;\bm t_{k2}]$ be an orthonormal basis of the fitted board plane, $\bm n_k$ its normal, and $(u,v)$ the in-plane coordinates. With the design rows $[1\;u_l\;v_l]$ of the boundary points stacked into $\bm A_k$, those of the four centers stacked into $\bm B_k\in\R^{4\times3}$, and the point scatter $s^2_{\perp,k}$ about the plane, the plane-parameter covariance and the LiDAR covariance in~\cref{eq:residcov} are
\begin{equation}
\left\{
\begin{aligned}
&\bm\Sigma^{\pi}_k=\frac{N_k}{N_{\mathrm{eff},k}}\,s^2_{\perp,k}\,(\bm A_k^\top\bm A_k)^{-1},\\
&\bm\Sigma^{\mathtt{L}}_{k}=\operatorname{blkdiag}_j\bigl(\bm U_k\bm\Sigma^{\mathtt{L},\mathrm{pl}}_{kj}\bm U_k^\top\bigr)
+\bigl(\bm B_k\bm\Sigma^{\pi}_k\bm B_k^\top\bigr)\otimes\bm n_k\bm n_k^\top ,
\end{aligned}
\right.
\label{eq:lidar3d}
\end{equation}
where the propagation of the point noise to the plane parameters follows the plane uncertainty model of LiDAR mapping~\citep{yuan2022efficient}. The first term has rank eight and the second rank three, so that $\rank(\bm\Sigma^{\mathtt{L}}_k)=11$. Of the eight in-plane coordinates, three describe a rigid motion within the board plane and five deform the hole layout, and the three plane parameters move the four centers rigidly out of the plane. The extrinsic absorbs the six dimensions of a rigid motion, and only the five dimensions of a layout deformation reach the internal residual of \cref{prop:residual}. The twelfth direction is the saddle mode, in which the four centers leave their common plane with alternating signs. Since the LiDAR centers lie on the fitted plane and the camera centers on the estimated board, no measurement exists along this mode, and it is removed from the residual instead of being assigned a variance (\cref{subsec:residcov}). An independent normal variance for each hole would assign a measurement to the saddle mode. \cref{fig:covlidar} compares the free and the regularized fit. On the synthetic short arc (left), the free-radius covariance extends along the radial axis and the radius prior contracts this axis, whereas on the partial-arc hole of the head pair (right) the prior covariance is the larger one, since it counts only $N_{\mathrm{eff}}$ effective points, whereas the radius prior still contracts the radial axis. The two limits of the prior standard deviation and the conservative margin of \cref{eq:radiuscov} along the radial axis are given in \cref{app:radiusprior}.

\begin{figure}[!t]\centering\includegraphics[width=\columnwidth]{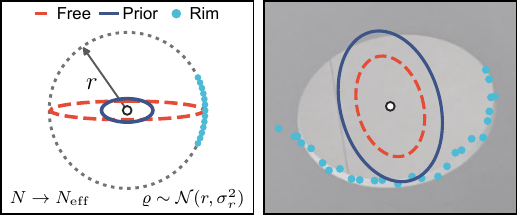}
\caption{Short-arc LiDAR center covariance ($3\sigma$) of the free-radius fit (Free) and of the radius prior with $N_{\mathrm{eff}}$ points (Prior); Rim: boundary points. \textbf{Left}: synthetic $\pm25^\circ$ arc. \textbf{Right}: partial-arc hole of a head-pair observation ($\times12$).}
\label{fig:covlidar}\end{figure}

The free-fit radius remains available to monitor the nominal radius, since the prior only shrinks a short-arc center toward the nominal-radius solution. As the effective radius of a sensor class is one quantity shared by all holes, the batch-level diagnostic averages the free-fit radii $\hat\varrho_{kj}$, with marginal variances $\sigma^2_{\hat\varrho,kj}$, over the $n_h$ holes of the accepted set,
\begin{equation}
z^{r}=\frac{\bar{\hat\varrho}-r}{\bigl(\sigma_r^2+\overline{\sigma^2_{\hat\varrho}}/n_h\bigr)^{1/2}},
\label{eq:radiuspool}
\end{equation}
referred to the standard normal distribution at level $\alpha$. The prior standard deviation $\sigma_r$ in the denominator is a tolerance, and a class-level deviation below it is a systematic error that the covariance does not represent. Its induced center displacement is reported as a bias budget outside $\bm H^{-1}$, and the diagnostic of each hole is given in \cref{app:radius}. The gain in point-estimate accuracy from this prior on partial arcs is given in \cref{app:shortarc} and by \citet{hu2026p2calib}.

\subsection{Joint Residual Covariance Estimation}
\label{subsec:residcov}

With independent modalities, the joint residual covariance of each observation combines the camera block of \cref{subsec:camcov} and the LiDAR block of \cref{subsec:lidarcov},
\begin{equation}
\bm\Sigma^{r}_{0,k}=\bm\Sigma^{\mathtt{C}}_k+(\bm I_4\!\otimes\!\bm R)\,\bm\Sigma^{\mathtt{L}}_k\,(\bm I_4\!\otimes\!\bm R)^\top,\qquad
\bm\Sigma^{r}_k=s\,\bm\Sigma^{r}_{0,k}.
\label{eq:residcov}
\end{equation}
The combined covariance is projected onto the complement of the saddle mode of \cref{subsec:lidarcov} and has rank $d_k=11$. For the four-hole target and for rigid checkerboard corners alike, all inverse and whitening operations below are restricted to the supported subspace of \cref{eq:supportwhitening}.

\subsection{Two-Stage Scale Estimation}
\label{subsec:scaleest}
The two blocks fix the covariance shape, whereas the noise magnitude $s$ is rarely known, since the pixel noise of the marker corners and the edge noise of the LiDAR depend on the detector, the lighting, the range, and the motion. By \cref{thm:gain}, $s$ leaves the estimate invariant, but it enters every reported quantity,
\begin{equation}
\begin{aligned}
\operatorname{Cov}(\delta\hat{\bm\xi})&=s\,\bm H_0^{-1},\qquad \MDB(\bm v)\propto\sqrt s,\\
\mathbb E\bigl[\bm r_k^\top(\bm\Sigma^{r}_{0,k})^{-1}\bm r_k\bigr]&=s\,d_k,
\end{aligned}
\label{eq:scaleeffect}
\end{equation}
with $\bm H_0=\sum_{k\in\calS}\bm J_k^\top(\bm\Sigma^r_{0,k})^{-1}\bm J_k$ and $\bm r_k$ the residual of a fault-free observation at the true extrinsic, whose fitted counterpart has the mean $s(d_k-h_k)$ (\cref{eq:deflation}). A nominal value $s=1$ therefore understates the reported covariance and the detection limits whenever the actual noise exceeds the nominal magnitude, and a threshold fixed under this magnitude falsely detects fault-free observations, as the fixed-threshold baselines show at twice the nominal noise (\cref{subsec:expcons}). Calibration methods that report a covariance usually fix this magnitude a priori~\citep{brookshire2012extrinsic}, and the a posteriori variance factor of geodetic adjustment uses all residual dimensions~\citep{koch1999parameter,teunissen2008least}. We instead estimate $s$ from the absorbed components alone, since by \cref{prop:residual} only these components enter the estimate. With $\bm L_k$ and the orthonormal basis $\bm V_k$ of $\rangeop(\bar{\bm J}_k)$ computed from the shape $\bm\Sigma^r_{0,k}$ instead of $\bm\Sigma^r_k$, and $\bm e^{\mathrm{ab}}_k=\bm V_k^\top\bm L_k\hat{\bm r}_k$ the absorbed component of the fitted residual, the estimate for a fixed accepted set is
\begin{equation}
\hat s=\frac{1}{d_{\mathrm{ab}}}\sum_{k\in\calS}\lVert\bm e^{\mathrm{ab}}_k\rVert^2,\qquad
d_{\mathrm{ab}}=\textstyle\sum_{k\in\calS}\rank(\bar{\bm J}_k)-6,
\label{eq:scale}
\end{equation}
where the six extrinsic parameters remove six absorbed dimensions from the total.
\begin{proposition}[Scale of the estimate]
\label{prop:scale}
Let the whitened residual $\bar{\bm r}_k$ of each accepted observation have the covariance $s\,\bm I$ on $\rangeop(\bar{\bm J}_k)$ and any positive-definite covariance on $\nullop(\bar{\bm J}_k^\top)$, with the two components uncorrelated and the observations independent. Then, to first order, the weighted estimate of~\cref{eq:problem} is the best linear unbiased estimate, its covariance is $s\,\bm H_0^{-1}$ with $\bm H_0$ of \cref{eq:scaleeffect} summed over $\calS$, and $\mathbb E[\hat s]=s$.
\end{proposition}
\noindent The proof is given in \cref{app:proofs}. Consequently, the scale of the internal components enters neither the estimate nor its covariance. A scale combined over all residual dimensions is correct only if both subspaces share one scale, which the batch test of \cref{subsec:constests} examines. The scale of~\cref{eq:scale} remains correct when they differ.
As the accepted set is itself an outcome of the consistency assessment, the scale is estimated in two stages. With the normalized innovation squared (NIS) $\NIS_k=\hat{\bm r}_k^\top(\bm\Sigma^{r}_{0,k})^{-1}\hat{\bm r}_k$ of the fitted residual under the shape, a first, median-type estimate on all observations provides the working scale of the consensus stage of \cref{sec:method:tests}; it uses the full NIS, since the median requires one statistic per observation, and it serves only to scale the consensus threshold. With the leverage $h_k$ of \cref{eq:deflation}, $\NIS_k/(d_k-h_k)$ has null mean $s$ and null median approximately $s\,c_k$, with $c_k=\operatorname{med}(\chi^2_{d_k-h_k})/(d_k-h_k)$, and
\begin{equation}
\hat s_1=\frac{\operatorname{med}\{\NIS_k/(d_k-h_k)\}}{\operatorname{med}\{c_k\}}
\label{eq:scale1}
\end{equation}
is approximately unbiased under the null and less inflated by gross faults than the mean-type estimate. Once $\calS$ is fixed, $\hat s$ is re-estimated on $\calS$ alone by \cref{eq:scale}, so that an observation rejected as inconsistent does not inflate the scale on which the reported precision, the detection limits, and the calibration state depend.
The scale $\hat s$ thus corrects a proportional variance mismatch (\cref{fig:covscale}). In \cref{fig:covscale} (left), a single scale aligns the joint covariance with the ground truth, whereas no common scale aligns the independent-hole covariance, whose shape differs from the true one. A range-dependent shape error is monitored by the trend test of \cref{subsec:constests}. The weighted estimate of \cref{eq:problem} on the accepted set and its covariance $\bm H^{-1}$ then follow from \cref{eq:localinfo,eq:estcov}, and the estimate reduces to the unweighted solver~\cref{eq:baseline} for $\bm\Sigma^{r}_k\propto\bm I$. If the factor model is correct, $\bm H^{-1}$ predicts the estimator scatter, and the tests of \cref{sec:method:tests} decide whether it is reported as a verified uncertainty (\cref{fig:covscale}, right).\label{subsec:wls}

\begin{figure}[!t]\centering\includegraphics[width=\columnwidth]{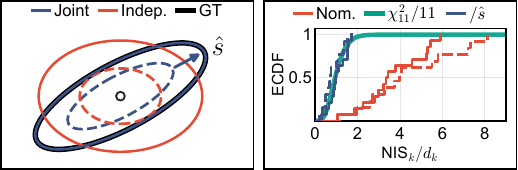}
\caption{Covariance shape and scale. \textbf{Left}: nominal (dashed) and $\hat s$-scaled (solid) covariances of the joint and independent-hole models against the ground truth (GT). \textbf{Right}: ECDF of $\mathrm{NIS}_k/d_k$ before (Nom.) and after ($/\hat s$) scale estimation (solid: head; dashed: chassis).}
\label{fig:covscale}\end{figure}

\subsection{Checkerboard-Target Instantiation}
\label{sec:checkerboard}
For a standard checkerboard, only the target-specific factor construction differs, and the estimation, the directional analysis, and the tests are reused. The LiDAR corners are those of one rigid rectangle of the known outer size fitted to the scanned boundary, so that, as for the soft radius prior, a known-geometry prior regularizes a weakly constrained LiDAR feature fit. The four LiDAR corners and the four camera corners are therefore both generated by six pose parameters, and the residual covariance of a checkerboard observation has rank $d_k=6$ instead of $11$ (\cref{app:cornercov}). A non-degenerate rigid board thus leaves no within-observation degrees of freedom, and the batch consistency criterion relies on the cross-observation test of \cref{eq:crosssubset}. The closed-form directional expressions~\crefrange{eq:single_rot}{eq:multi} are specific to the four-hole target, and the checkerboard directional information is evaluated from the numerical Schur complement of~\cref{eq:lamR}.

\section{Directional Information and Acquisition Guidance}
\label{sec:method:obs}
The parameter-level compensation of \cref{prop:parameter} makes the diagonal information blocks overstate the precision, and observability analysis does not reveal which physical direction limits a calibration. In this section, we define the precision of each physical direction from the marginal information of \cref{eq:lamR}, derive its closed form for the four-hole target (\cref{subsec:closedform}), and select the next target pose along the weakest direction (\cref{subsec:guidance}).

The predicted standard deviations along a unit rotation or translation direction $\bm v$ are
\begin{equation}
\sigma_R(\bm v)=\sqrt{\bm v^\top\LamR^{-1}\bm v},\qquad
\sigma_t(\bm v)=\sqrt{\bm v^\top\LamT^{-1}\bm v},
\label{eq:dirstd}
\end{equation}
with the worst case at the smallest eigenvalue of $\LamR$ and $\LamT$, respectively. For rank-deficient information the complements are defined variationally through pseudo-inverses (\cref{app:rankdef}), and directions outside the range of the marginal information have no finite variance and are reported as unconstrained.

\subsection{Closed-Form Directional Information}
\label{subsec:closedform}
Target poses are usually chosen heuristically, without a model of how each pose contributes to each direction. For the symmetric four-hole target, the decoupled information has a closed form, which shows how each pose contributes and yields design guidelines for the target-pose geometry.
\begin{proposition}[Single-observation information]
\label{prop:single}
For a frontal on-axis board at depth $Z$ with isotropic hole-center noise $\sigma^2\bm I_3$,
\begin{gather}
\LamR=\tfrac{1}{\sigma^2}\diag\!\bigl(4b^2,\,4a^2,\,4(a^2{+}b^2)\bigr),\label{eq:single_rot}\\
\LamT=\tfrac{1}{\sigma^2}\diag\!\Bigl(\frac{4a^2}{Z^2+a^2},\,\frac{4b^2}{Z^2+b^2},\,4\Bigr).\label{eq:single_trans}
\end{gather}
\end{proposition}
\begin{proof}
The holes are located at $\bm x_j=(\pm a,\pm b,Z)$ with hole Jacobian $\bm J_j=[-\skewop{\bm x_j}\;\;\bm I_3]$. Summing $\bm J_j^\top\bm J_j/\sigma^2$ and applying the symmetry sums $\sum_j\bm x_j=(0,0,4Z)$ and $\sum_j\bm x_j\bm x_j^\top=\diag(4a^2,4b^2,4Z^2)$ yields
\begin{gather*}
\bm A=\tfrac{1}{\sigma^2}\diag\!\bigl(4(Z^2{+}b^2),\,4(Z^2{+}a^2),\,4(a^2{+}b^2)\bigr),\\
\bm C=\tfrac{4}{\sigma^2}\bm I_3,\qquad
\bm B=\tfrac{1}{\sigma^2}\skewop{(0,0,4Z)}.
\end{gather*}
The coupling term $\bm B\bm C^{-1}\bm B^\top=\frac{4Z^2}{\sigma^2}\diag(1,1,0)$ then cancels the $Z^2$ entries of $\bm A$ exactly, leaving~\cref{eq:single_rot}. The same elimination applied to $\LamT$ yields~\cref{eq:single_trans}.
\end{proof}
The rotation precision of a frontal board is thus independent of its depth, whereas its lateral translation precision decreases with depth. Lateral translation information decreases as $1/(Z^2{+}a^2)$, since a distant pitch and a lateral shift become indistinguishable at large $Z$, whereas the depth component remains $Z$-invariant under constant noise. The conditional and marginal blocks can also rank the directions differently. For $Z>a>b$ the conditional block $\bm A$ of the proof is smallest about the $z$ axis, whereas $\LamR$ in~\cref{eq:single_rot} is smallest about the $x$ axis. An analysis of the diagonal block therefore reports the wrong weakest rotation. Rotation precision is determined by the hole baseline alone, and isotropy is required only for the diagonal constants in \cref{prop:single}.

The conditional block $\bm A$ grows with $Z^2$ about the two lateral axes, which explains the intuition that distant boards improve the rotation, as distant structure constrains the rotation in stereo odometry~\citep{kaess2009flow}, whereas the coupling term removes exactly this growth once the translation is estimated (\cref{fig:compensation}C). The multi-observation information decomposes according to the parallel-axis theorem,
\begin{equation}
\LamR=\tfrac{1}{\sigma^2}\Bigl[\textstyle\sum_k\bm Q_k\bm I_{\bm a}\bm Q_k^\top+m\sum_k\bigl(\lVert\bm d_k\rVert^2\bm I_3-\bm d_k\bm d_k^\top\bigr)\Bigr],
\label{eq:multi}
\end{equation}
with the hole-layout inertia $\bm I_{\bm a}$ and the centered board positions $\bm d_k$ (\cref{app:multi}).
\begin{corollary}[Standoff invariance]
\label{cor:standoff}
For residual covariances that do not depend on the board position, translating every board by a common offset leaves $\LamR$ invariant, and the rotation precision is thus determined by the orientation diversity and the relative spread of the board centers.
\end{corollary}
\begin{proof}
A common offset $\bm d$ replaces each Jacobian $\bm J_k$ by $\bm J_k\bm T$ with $\bm T=\bigl[\begin{smallmatrix}\bm I_3&\bm0\\-\skewop{\bm d}&\bm I_3\end{smallmatrix}\bigr]$, so that $\bm H$ becomes $\bm T^\top\bm H\bm T$. As the first block row of $\bm T^{-1}$ is $[\bm I_3\;\bm0]$, the rotation-marginal covariance $[\bm H^{-1}]_{\phi\phi}=\LamR^{-1}$ is invariant. For the isotropic model, \cref{eq:multi} shows this through the centered positions $\bm d_k$ and orientations $\bm Q_k$, which the offset preserves.
\end{proof}
\noindent The corollary contradicts the common-standoff version of the far-board-first practice and is consistent with the observation of \citet{zhou2018automatic} that parallel planar targets with parallel boundaries provide the same line and plane constraints, of which it is the counterpart for the marginal rotation information. Within a fixed field of view, distant boards can also be spread farther apart, which increases the second term of \cref{eq:multi}, and this spread, not the distance itself, improves the rotation precision. Rotation information is therefore gained by varying the board orientations and spreading the board centers, and the score of \cref{subsec:guidance} selects the next pose that strengthens the weakest direction. Under the range-incidence noise model, distant boards are penalized further, as they also inflate $\sigma_L$. The closed forms of \cref{prop:single} and \cref{eq:multi} assume the symmetric four-hole geometry and isotropic noise, and the directional information of the full structured covariance is evaluated numerically from~\cref{eq:lamR}.

\subsection{Weakest-Direction Data-Collection Guidance}
\label{subsec:guidance}
Data collection has to determine which target pose to add next. The D-, A-, and E-optimal criteria~\citep{fedorov1972theory}, applied in calibration as entropy, covariance-trace, and smallest-singular-value selection~\citep{maye2016self,peng2019calibrationwizard,lv2022observability}, maximize a scalar function of the coupled information. Its rotation and translation blocks have different units, so that these criteria do not indicate which requirement fails. We instead define the score on the marginal block that fails the requirement. A directional precision requirement $(\epsilon_R,\epsilon_t)$ bounds the worst-direction standard deviation of each block and is satisfied only if
\begin{equation}
\lambda_{\min}(\LamR)\ge 1/\epsilon_R^2,\qquad
\lambda_{\min}(\LamT)\ge 1/\epsilon_t^2 ,
\label{eq:precisiontarget}
\end{equation}
that is, $\sigma_R(\bm v)\le\epsilon_R$ and $\sigma_t(\bm v)\le\epsilon_t$ along every unit direction or a specified one. It bounds each direction separately, not a joint three-dimensional confidence region.
\begin{definition}[Weakest-direction score]
\label{def:score}
When one block fails, let $\bm v_\star$ be the weakest eigenvector of the failing Schur information. For a candidate target pose $q$ with information increment $\bm H_q$, the myopic weakest-direction score is defined on the Schur increments $\Delta\LamR(q)$ and $\Delta\LamT(q)$,
\begin{equation}
\begin{gathered}
S(q)=\bm v_\star^\top\Delta\bm\Lambda_b(q)\,\bm v_\star,\\
\Delta\bm\Lambda_b(q)=\bm\Lambda_b(\bm H+\bm H_q)-\bm\Lambda_b(\bm H),
\end{gathered}
\label{eq:weakdir_score}
\end{equation}
where $b\in\{R,t\}$ denotes the block that fails~\cref{eq:precisiontarget}, and $\bm\Lambda_R=\LamR$, $\bm\Lambda_t=\LamT$.
\end{definition}
\noindent For a simple smallest eigenvalue, $S(q)$ is the first-order increase of $\lambda_{\min}(\bm\Lambda_b)$, and by the Rayleigh quotient, $\lambda_{\min}(\bm\Lambda_b+\Delta\bm\Lambda_b)\le\lambda_{\min}(\bm\Lambda_b)+S(q)$ for any increment. When both blocks fail, the score is computed for each block separately. For a recorded candidate observation, $\bm H_q$ is the information of its factor at the current estimate, and for a planned pose it follows from the target geometry and the covariance model of \cref{sec:method}. The score ranks feasible target poses ($\arg\max_q S(q)$, subject to the radius and consistency diagnostics). It guides the operator toward a target orientation and range that strengthen the failing direction.

\section{Consistency Assessment}
\label{sec:method:tests}
The covariance and the directional precision above hold only for observations consistent with~\cref{eq:obsmodel}. By \cref{prop:residual}, the absorbed part of a fault is invisible within its observation, so that the internal residual is examined within each observation and the absorbed part by comparison with the other observations. This section follows the order of \cref{fig:pipeline}. We first derive test thresholds that hold their level under the unknown noise magnitude (\cref{subsec:ftest}). With them, the diagnosis selects a consensus set and removes inconsistent observations by a within-view check of each observation and a cross-view check (\cref{subsec:diagnosis}), and the batch test examines the accepted set (\cref{subsec:constests}). The smallest fault that remains undetected in each direction is given in \cref{subsec:mdb}, and the consistency, precision, and detectability criteria together form the calibration state of \cref{subsec:status}.

\begin{figure}[!t]\centering
\definecolor{npgnavy}{HTML}{3C5488}%
\begin{tikzpicture}[node distance=3.6mm, fl/.append style={draw=black},
  box/.style={nd, font=\sffamily\small, text=black, minimum width=60mm, minimum height=6.2mm, inner sep=2pt,
              fill=black!4, draw=black!65},
  fin/.style={box, fill=npgnavy!14, draw=npgnavy, line width=0.6pt},
  sec/.style={font=\sffamily\scriptsize, text=black, anchor=east, inner sep=3pt},
  lab/.style={font=\footnotesize, text=black, right, inner sep=3pt, align=left}]
\node[box] (a) {Structured measurement factors};
\node[box,below=of a] (b) {Observation diagnosis};
\node[box,below=of b] (c) {Estimation and batch test};
\node[box,below=of c] (d) {Directional information};
\node[box,below=of d] (e) {Directional detection limit};
\node[fin,below=of e] (f) {Calibration state};
\foreach \n/\l in {a/\ref*{sec:factor}--\ref*{sec:method},b/\ref*{subsec:diagnosis},c/\ref*{subsec:constests},d/\ref*{sec:method:obs},e/\ref*{subsec:mdb},f/\ref*{subsec:status}}
  \node[sec] at (\n.east) {\S\l};
\draw[fl] (a) -- node[lab]{$\calF_k$} (b);
\draw[fl] (b) -- node[lab]{$\calS$} (c);
\draw[fl] (c) -- node[lab]{$\hatT_{\mathtt{CL}},\,\bm H,\,\hat s$;\ \ $G_c$} (d);
\draw[fl] (d) -- node[lab]{$\LamR,\,\LamT$;\ \ $G_p$} (e);
\draw[fl] (e) -- node[lab]{$\MDB^{\max}$;\ \ $G_d$} (f);
\draw[fl,dashed] (f.west) -- ++(-5mm,0) |- node[pos=0.25,left,rotate=90,anchor=south,font=\sffamily\scriptsize,text=black]{add, re-extract, or recollect} (a.west);
\end{tikzpicture}
\caption{Verification pipeline (\cref{alg:pipeline}).}
\label{fig:pipeline}
\end{figure}

\subsection{Calibrated Test Thresholds}
\label{subsec:ftest}

The conventional in-sample residual check is not statistically calibrated. To first order, the fitted residual $\hat{\bm r}_k$ of observation $k$ and its in-sample $\NIS_k$ satisfy
\begin{equation}
\begin{gathered}
\bm\Sigma^{\hat r}_k{}\doteq\bm\Sigma^{r}_k-\bm J_k\bm H^{-1}\bm J_k^\top\preceq\bm\Sigma^{r}_k,\\
\mathbb E[\NIS_k]{}=s(d_k-h_k)<s\,d_k,
\end{gathered}
\label{eq:deflation}
\end{equation}
with $\bm\Sigma^{\hat r}_k$ the covariance of $\hat{\bm r}_k$, leverage $h_k=\tr(\bm H^{-1}\bm H_k)$, and $\sum_kh_k=6$ (\cref{app:insample}). The fitted residual is therefore deflated, most severely at small $K$, and a $\chi^2$ threshold applied to it is exceeded less often than its nominal significance level (\cref{fig:constests}, left).
The deflation and the unknown scale both prevent a fitted residual from being compared with a $\chi^2$ threshold, and a scale estimated from the same residuals does not remove this bias (\cref{app:tests}). Every residual quadratic-form test is therefore referred to an $F$ distribution with an independent scale estimate $\hat s_{\mathrm{ref}}$ on $\nu_s$ degrees of freedom, and a statistic $Q$ of nominal distribution $\chi^2_d$ is rejected when
\begin{equation}
\frac{Q}{d\,\hat s_{\mathrm{ref}}}>F^{-1}_{d,\nu_s}(1-\alpha),
\label{eq:ftest}
\end{equation}
which recovers the $\chi^2$ threshold as $\nu_s\to\infty$. The construction parallels Student's $t$ test, which replaces the unknown variance of a $z$ test by an independent estimate. The structured split of \cref{eq:structured_split} and cross-fitting provide these independent scales without additional data.

\subsection{Observation Diagnosis}
\label{subsec:diagnosis}
\label{sec:method:pipeline}

The diagnosis determines which observations disagree with the factor model and whether they can be removed, and the batch test of \cref{subsec:constests} then examines the accepted set.

Contamination that affects several observations together is addressed by a RANSAC-style consensus search~\citep{fischler1981ransac}, which selects the dominant set $\calS_0$ of~\cref{eq:badview} under the working scale $\hat s_1$ (\cref{app:tests}). If two comparable, mutually inconsistent consensus sets remain, the procedure reports inconsistent measurements and recommends recollecting the observations. This is the breakdown regime of correlated majority contamination. The final extrinsic is always the weighted least-squares solution~\cref{eq:problem} over the accepted set, which is efficient for fault-free data, with reported covariance $\bm H^{-1}$ of that set.

\paragraph{Within-view check.}
The within-view check tests the internal structure of each observation without reference to the others, and it is the null-space component of the split~\cref{eq:structured_split}. It is sensitive to the internal component of single-correspondence gross errors, within-observation differential hole-center errors, wrong correspondences, and inconsistent boundary extraction, that is, of every error that no infinitesimal rigid extrinsic perturbation explains completely. A rigid perturbation of a single observation lies in $\rangeop(\bm L_k\bm J_k)$, is removed by the null-space projection, and is left to the cross-observation tests.

\begin{figure}[!t]\centering\includegraphics[width=\columnwidth]{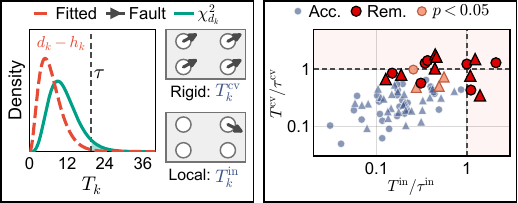}
\caption{Consistency tests. \textbf{Left}: null and fitted NIS densities with threshold $\tau$, and rigid and local faults. \textbf{Right}: within-view and cross-view statistics of the head (circles) and chassis (triangles) observations, normalized by their thresholds.}
\label{fig:constests}\end{figure}

\paragraph{Cross-view check.}
The cross-view check tests whether each observation agrees with the extrinsic estimated from the remaining observations. The cross-view statistic is based on the prediction residual $\bm\zeta_k=\bm r_k(\hatT_{-k})$,
\begin{equation}
T^{\mathrm{cv}}_k=\bm\zeta_k^\top\bigl(\bm\Sigma^{r}_k+\bm J_k\bm H_{-k}^{-1}\bm J_k^\top\bigr)^{-1}\bm\zeta_k\;\sim\;\chi^2_{d_k},
\label{eq:cvnis}
\end{equation}
which detects observation-level rigid shifts, PnP pose errors, synchronization errors, or variations of the extrinsic between sessions (\cref{fig:constests}, right). The studentized fitted residual $T^{\mathrm{sd}}_k$, the multivariate form of Baarda's data snooping~\citep{baarda1968testing}, is retained as an additional diagnostic when a leave-one-out re-estimation is unstable. Each observation-level statistic is referred to an $F$ distribution with an independent scale (\cref{app:tests}).

\paragraph{Removal criterion.}
The classical leave-one-out influence diagnostic of regression analysis~\citep{cook1977detection,belsley1980regression} ranks observations by the displacement of the solution upon their removal and is often applied for outlier detection, but it is not a calibrated bias detector. With the observation score $\bm g_k$, the displacement $\delta\bm\xi_{-k}$ and its information-normed magnitude $I_k$ are
\begin{equation}
\begin{gathered}
\bm g_k=\bm J_k^\top(\bm\Sigma^{r}_k)^{-1}\hat{\bm r}_k,\qquad
\delta\bm\xi_{-k}\approx\bm H_{-k}^{-1}\bm g_k,\\
I_k=\bm g_k^\top\bm H^{-1}\bm g_k,
\end{gathered}
\label{eq:influence}
\end{equation}
and at the true extrinsic, $\mathbb E[\bm g_k\bm g_k^\top]=\bm H_k$, so that a fault-free observation has the expected influence $\mathbb E[I_k]=\tr(\bm H^{-1}\bm H_k)=h_k$, its leverage (\cref{app:influence}).
Influence hence measures the effect of an observation on the estimate, but it does not separate a biased observation from an informative one (\cref{fig:consdiag}). In the influence-inconsistency plane of \cref{fig:consdiag} (left), the leave-one-out selection chooses observations along the influence axis, whereas the removal criterion of \cref{eq:badview} selects them along the inconsistency axis. Ranking by influence places first the tilted, widely spread, highly informative observations that a well-conditioned set requires (the classical distinction between influence and outliers~\citep{cook1977detection}).

\begin{figure}[!t]\centering\includegraphics[width=\columnwidth]{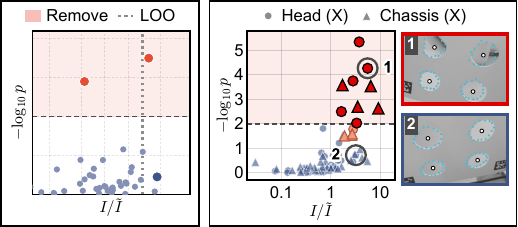}
\caption{Observation diagnosis in the influence-inconsistency plane. \textbf{Left}: removal criterion against leave-one-out (LOO) selection. \textbf{Right}: head (circles) and chassis (triangles); blue: accepted; salmon: $p<0.05$; red: removed; X: inconsistent.}
\label{fig:consdiag}\end{figure}
An observation detected as statistically inconsistent is removed if the remaining set meets the geometric condition:
\begin{equation}
\begin{gathered}
\calR=\Bigl\{k\in\calS_0:\ \max_i\frac{T^{i}_k}{\tau^{i}_k}>1,\ \ 
\lambda_{\min}(\bm\Lambda_{-k})\ge\rho_{\min}\Bigr\},
\end{gathered}
\label{eq:badview}
\end{equation}
where $i\in\{\mathrm{in},\mathrm{cv},\mathrm{sd}\}$ runs over the three observation-level statistics, each threshold $\tau^{i}_k$ is the level-$\alpha_{\mathrm{bias}}$ point of the $F$ distribution of the statistic (\cref{subsec:ftest}), $\bm\Lambda_{-k}$ is the directional information~\cref{eq:lamR} after removing $k$, which requires $\rank(\bm H_{-k})=6$, and $\rho_{\min}=\lambda_{\min}(\bm\Lambda)/2$ is half the value of the full set, for both blocks, evaluated for each observation. A detected observation is removed independently of its influence, since the detection rejects the zero-mean condition of \cref{ass:model}(ii) on which the scale and the reported precision rest, and the influence $I_k$ is reported with each decision. The rank condition preserves the directional information, and a detected observation that it retains makes the batch inconsistent (\cref{subsec:status}).

\subsection{Batch Consistency Test}
\label{subsec:constests}
The batch test aggregates the within-view energies $T^{\mathrm{in}}_k$ of~\cref{eq:inttest} over the batch, here formed with the shape $\bm\Sigma^r_{0,k}$ in place of $\bm\Sigma^r_k$. The batch internal energy then has nominal distribution $s\chi^2_{d_{\mathrm{in}}}$,
\begin{equation}
T^{\mathrm{in}}_{\calS}=\sum_{k\in\calS}T^{\mathrm{in}}_k,\qquad
d_{\mathrm{in}}=\sum_{k\in\calS}\bigl(d_k-\rank(\bar{\bm J}_k)\bigr).
\label{eq:batchstat}
\end{equation}
Its scale reference is the scale $\hat s$ of \cref{eq:scale}, which is formed from the absorbed components and is independent of every internal component under the null (\cref{app:tests}). The batch test is therefore the variance-ratio test of the within- and between-observation energies, exact for a fixed linear Gaussian design before selection,
\begin{equation}
\frac{\hat s_{\mathrm{in}}}{\hat s}\sim F_{d_{\mathrm{in}},\,d_{\mathrm{ab}}},\qquad
\hat s_{\mathrm{in}}=\frac{T^{\mathrm{in}}_{\calS}}{d_{\mathrm{in}}},
\label{eq:batchint}
\end{equation}
evaluated two-sided at level $\alpha$, since the factor model assigns one scale to both subspaces.

The upper tail detects internal energy larger than its expected value, that is, faults within observations, and fails the consistency criterion. The lower tail detects absorbed energy larger than its expected value, that is, rigid discrepancy between the observations beyond their covariance. The scale of \cref{eq:scale} is estimated from the absorbed components and therefore includes the mean of this excess in the reported covariance, exactly if the excess is isotropic on $\rangeop(\bar{\bm J}_k)$ (\cref{prop:scale}). A lower-tail rejection alone therefore leaves the consistency criterion satisfied (\cref{subsec:status}).

The ratio $\hat s_{\mathrm{in}}/\hat s$ is reported with the test and identifies the rejected tail. With $g$ batch-level tests at level $\alpha$, the family-wise false-alarm probability of the consistency criterion is bounded by $g\alpha$ under the stated model, and the effect of the selection steps is measured in \cref{subsec:expstate}. The observation-level statistics of \cref{subsec:diagnosis} identify observations for removal at the stricter level $\alpha_{\mathrm{bias}}$ and are not counted against the significance level of the acceptance decision.

\begin{algorithm}[!t]
\caption{VeriCalib: verifiable target-based calibration}
\label{alg:pipeline}
\begin{algorithmic}[1]
\Require observations $\calD$; requirements $(\epsilon_R,\epsilon_t)$ and $(\delta_R,\delta_t)$; levels $\alpha$, $\beta$, $\alpha_{\mathrm{bias}}$; optional validation set $\calV$
\Ensure extrinsic $\hatT_{\mathtt{CL}}$, covariance $\bm H^{-1}$, calibration state, detection limit $\MDB^{\max}$
\Statex // \textit{Structured measurement factors} (\cref{sec:factor,sec:method})
\State $\calF_k\gets(\bm r_k,\bm\Sigma^{r}_{0,k},\bm J_k)$, $k\in\calD$ \Comment{\cref{eq:factor}}
\Statex // \textit{Observation diagnosis} (\cref{subsec:diagnosis})
\State $\hat s_1\gets$ working scale on $\calD$ \Comment{\cref{eq:scale1}}
\State $\calS_0\gets$ consensus set under $\hat s_1\bm\Sigma^{r}_{0,k}$ \Comment{ambiguity: state X}
\State $\calS\gets\calS_0\setminus\calR$ \Comment{\cref{eq:badview} at $\alpha_{\mathrm{bias}}$}
\Statex // \textit{Estimation and batch test} (\cref{subsec:scaleest,subsec:constests})
\State $(\hatT_{\mathtt{CL}},\bm H)\gets$~\cref{eq:problem} on $\calS$;\ \ $\hat s\gets$~\cref{eq:scale}
\State $G_c\gets$ upper tail of~\labelcref{eq:batchint}, radius and trend tests~\labelcref{eq:radiuspool,eq:trend}, consensus ambiguity, rank-retained detections
\Statex // \textit{Directional information} (\cref{sec:method:obs})
\State $\LamR,\LamT\gets$~\cref{eq:lamR};\ \ $G_p\gets$~\cref{eq:precisiontarget}
\Statex // \textit{Directional detection limit} (\cref{subsec:mdb})
\If{$\calV\neq\emptyset$}
  \State $\MDB^{\max}\gets\max_{\bm v}\MDB_{\mathrm{eff}}(\bm v)$ on $\calV$ at $\hatT_{\mathtt{CL}}$; validation tests at level $\alpha$
\Else
  \For{$m=1,\dots,M$ random halves $\calS=\calS^{m}_{\mathrm f}\,\dot\cup\,\calS^{m}_{\mathrm c}$}
    \State $\mu_m\gets\max_{\bm v}\MDB_{\mathrm{eff}}(\bm v)$~\cref{eq:mdbeff} on $\calS^{m}_{\mathrm c}$ for the fit on $\calS^{m}_{\mathrm f}$
  \EndFor
  \State $\MDB^{\max}\gets\operatorname{med}_m\mu_m$ \Comment{needs $\calV$ after a lower tail}
\EndIf
\State $G_d\gets(\MDB^{\max}_R\le\delta_R)\wedge(\MDB^{\max}_t\le\delta_t)$
\State \Return $\hatT_{\mathtt{CL}}$, $\bm H^{-1}$, $\operatorname{state}(\calD)$ of~\cref{eq:state}, $\MDB^{\max}$
\end{algorithmic}
\end{algorithm}

The batch test of \cref{eq:batchint} assigns one scale to the internal and absorbed subspaces. Since the camera block of \cref{eq:camcov} lies in $\rangeop(\bm J_k)$, a camera covariance too small relative to the LiDAR block increases only the absorbed energy, as does an independent rigid error per observation.
\begin{proposition}[Variance-ratio ambiguity]
\label{prop:ratio}
Let
\begin{equation}
\bm G_k=\begin{bmatrix}\bm I_3&\bm0\\\skewop{\bm c_k}&\bm I_3\end{bmatrix}
\label{eq:gk}
\end{equation}
with the board center $\bm c_k$ of \cref{eq:camjac}, so that $\bm F_k=\bm J_k\bm G_k$ to first order, and let each batch observation contain, in addition to the factor model, an independent zero-mean Gaussian rigid error $\bm J_k\bm\beta_k$ with $\operatorname{Cov}(\bm\beta_k)=\gamma\,\bm G_k\bm\Sigma_{\theta,k}\bm G_k^\top$, $\gamma>0$. Then the batch residuals follow the factor model with the camera block $\bm\Sigma^{\mathtt{C}}_k$ of \cref{eq:residcov} replaced by $(1+\gamma)\,\bm\Sigma^{\mathtt{C}}_k$, and the weighted estimate remains unbiased.
\end{proposition}
\noindent The proof is given in \cref{app:aux}. Consequently, independent rigid errors can be indistinguishable from a camera covariance that is too small, and no statistic of the batch separates the two. Both leave the estimate unbiased, and the scale of \cref{eq:scale} includes their mean in the reported covariance. A rigid offset shared by part of the observations also increases mainly the absorbed energy but biases the estimate. A held-out validation set resolves this ambiguity (\cref{subsec:mdb,subsec:status}).

In addition to the batch test, a residual-trend test regresses the normalized residual $\NIS_k/d_k$ on the LiDAR range to monitor a range-dependent error of the shape~\cref{eq:lidarnoise} (\cref{app:tests}).

\subsection{Directional Minimal Detectable Bias}
\label{subsec:mdb}
\begin{figure}[!t]\centering\includegraphics[width=\columnwidth]{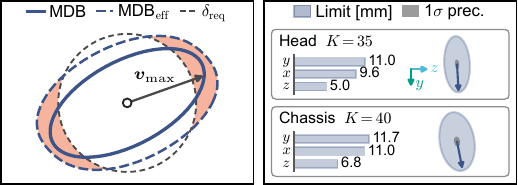}
\caption{Directional detection limits. \textbf{Left}: two-dimensional section (shaded: validation-limited). \textbf{Right}: translation detection limits of the head and chassis pairs along the principal axes, with the $1\sigma$ precision core and weakest direction.}
\label{fig:consmdb}\end{figure}
A passed test does not show that a fault of a given size would have been detected. Geodetic reliability quantifies this limit by the minimal detectable bias of an individual measurement~\citep{baarda1968testing,teunissen2018dia}. We resolve it along the rotation and translation directions of the extrinsic, on which the fault acts, with the compensating motion marginalized by the Schur complement. The validation information is built from held-out hole-center residuals, identical in form to~\cref{eq:problem} with the Jacobian of~\cref{eq:jacobian}. Their covariance is given by~\cref{eq:residcov}, and they are independent of the calibration observations by construction. The held-out information is decoupled as in~\cref{eq:lamR} into $\LamRv,\LamTv$.

The detection limit is stated for a differential fault. In \cref{eq:obsmodel}, $\bm b_k=\bm J_k\,\delta\bm v$ on the held-out validation observations and $\bm b_k=\bm0$ on the calibration observations, so that the validation data are displaced relative to the assessed extrinsic. A specified differential fault $\delta$ along unit direction $\bm v$ induces in the focused validation test the non-centrality $\delta^2\bm v^\top\bm\Lambda^{\mathrm v}\bm v$.
\begin{definition}[Directional minimal detectable bias]
\label{def:mdb}
The directional minimal detectable bias along a unit rotation or translation direction $\bm v$ is the smallest differential fault the focused validation test detects at level $\alpha$ with power $1-\beta$,
\begin{equation}
\MDB(\bm v;\alpha,\beta)=\sqrt{\frac{\lambda_{\alpha,\beta,1}}{\bm v^\top\bm\Lambda^{\mathrm v}\bm v}},
\label{eq:mdb}
\end{equation}
with $\lambda_{\alpha,\beta,1}=(z_{1-\alpha/2}+z_{1-\beta})^2$.
\end{definition}
\noindent \cref{def:mdb} converts a passed test into a quantitative statement, since, in the normal approximation, the focused validation test detects a differential fault larger than $\MDB(\bm v)$ along $\bm v$ with probability at least $1-\beta$. At $(\alpha,1{-}\beta)=(0.05,0.8)$, $\lambda_{\alpha,\beta,1}\approx7.85$ is the normal-approximation value. The focused test is referred to the $F_{1,d_{\mathrm{ab}}}$ distribution of \cref{eq:focused}, whose exact non-centrality for the same level and power increases the limit by at most $6\%$ for $d_{\mathrm{ab}}\ge18$, the smallest value of a four-observation set, and by about $1\%$ for $d_{\mathrm{ab}}\ge78$. If $\bm v^\top\bm\Lambda^{\mathrm v}\bm v=0$, the validation geometry is insensitive to the direction and $\MDB(\bm v)=\infty$. A common-mode fault lies outside the differential fault model and is discussed in \cref{subsec:scope}. The same expression on the full accepted-set information is not a detection limit, since the re-estimation absorbs a fault shared by all accepted observations.

The state report compares $\MDB$ directly with the application-specified detectable-fault requirement (\cref{fig:consmdb}). In \cref{fig:consmdb} (right), the detection-limit surface of each pair is elongated along its weakest direction and encloses the $1\sigma$ precision core, so that the smallest detectable fault is several times the reported precision. One finite-sample refinement is applied in all experiments. When the assessed estimate is itself fitted on a finite selection split, its estimation error propagates into the held-out residuals, and the null covariance of the focused score inflates from $\bm\Lambda^{\mathrm v}$ to $\bm\Lambda^{\mathrm v}+\bm\Lambda^{\mathrm v}\bm\Lambda_{\mathrm f}^{-1}\bm\Lambda^{\mathrm v}$ in the decoupled block. This inflation defines the corrected limit
\begin{equation}
\begin{gathered}
\MDB_{\mathrm{eff}}(\bm v){}=\kappa_{\bm v}\,\MDB(\bm v),\\
\kappa_{\bm v}^2{}=1+\frac{\bm v^\top\bm\Lambda^{\mathrm v}\bm\Lambda_{\mathrm f}^{-1}\bm\Lambda^{\mathrm v}\bm v}{\bm v^\top\bm\Lambda^{\mathrm v}\bm v}\ge1,
\end{gathered}
\label{eq:mdbeff}
\end{equation}
where $\bm\Lambda_{\mathrm f}$ is the decoupled information of the fit set in the same block, the inverse of the corresponding marginal block of $\bm H_{\mathrm f}^{-1}$. As the fit and certification subsets are disjoint halves of the accepted set, the limit refers to the assessed estimate of the fit half, so that a fault must exceed $\MDB(\bm v)$ by the factor $\kappa_{\bm v}$ to be detected in the validation data. As the limit depends on the random split, $\max_{\bm v}\MDB_{\mathrm{eff}}(\bm v)$ is evaluated on $M$ independent random splits and reported as the median over splits. For $|\calS|<8$ each observation certifies the fit on the others in turn and the largest limit is reported. For equal halves, the detectability requirement $(\delta_R,\delta_t)=3(\epsilon_R,\epsilon_t)$ calls for about $3.5$ times the observations that meet the precision requirement, so that the validation-limited state is expected between the observation counts that meet the precision and the detectability requirements (\cref{app:focused}).

The random-split limit is an internal reliability diagnostic of the accepted set and refers to a fit on half of it. For a validation set with the information of the certification half, the limit of the final estimate is smaller, since its fit information $\bm\Lambda_{\mathrm f}$ in \cref{eq:mdbeff} is larger, and the random-split limit is thus a conservative substitute for it. Independent confirmation instead holds out a validation set $\calV$ before observation selection and evaluates it at the final estimate with the same measurement model and scale. Its detection limit follows from \cref{eq:mdb,eq:mdbeff} with $\bm H_{\mathrm f}=\bm H$, and the focused tests along the worst directions and a six-degree-of-freedom validation test are run at level $\alpha$ (\cref{app:holdout}). A rejected validation test is reported with the state, which is then not confirmed by independent data (\cref{subsec:status}).

\subsection{Calibration State and Recommended Procedure}
\label{subsec:status}
\label{sec:method:report}
A residual alone does not indicate which action follows a calibration. The calibration state therefore indicates the first unmet requirement together with its procedure. The consistency criterion $G_c$ holds when the batch-level tests accept the accepted set, and the precision criterion $G_p$ when~\cref{eq:precisiontarget} holds. The detectability criterion $G_d$ holds when $\MDB_R^{\max}\le\delta_R$ and $\MDB_t^{\max}\le\delta_t$, and it fails in particular when the validation geometry is insensitive to a required direction and its detection limit is infinite (\cref{subsec:mdb}). The data-count criterion $G_K$, $\rank(\bm H)=6$ and enough observations for the required splits ($K\ge4$ in our implementation), is a precondition of the state. \begin{definition}[Calibration state]
\label{def:state}
For a batch that satisfies $G_K$, with the criteria evaluated in the order listed, the calibration state of $\calD$ is the first row of
\begin{equation}
\operatorname{state}(\calD)=
\begin{cases}
\mathrm{X}, & \neg G_c,\\
\mathrm{P}, & \neg G_p,\\
\mathrm{V}, & \neg G_d,\\
\mathrm{OK}, & G_c\wedge G_p\wedge G_d,
\end{cases}
\label{eq:state}
\end{equation}
whose condition holds, that is, inconsistent, precision-limited, validation-limited, and verified under the requirements.
\end{definition}
\noindent Evaluating the criteria in this order ensures that a consistency failure is reported even when the precision or the detection limit also fails, and the criterion vector $(G_c,G_p,G_d)$ is reported with the state. $G_c$ combines the upper tail of the batch test~\cref{eq:batchint}, the cross-observation consensus check, the radius diagnostic, and the trend test~\cref{eq:trend}, and also fails when the rank condition retains a detected observation.

For a rigid checkerboard the within-view check is replaced by the cross-subset test of \cref{eq:crosssubset} and the radius test is not applicable, and for the four-hole factor the cross-subset statistic~\cref{eq:subsettest} is reported as a diagnostic outside $G_c$. The detection limits of $G_d$ are evaluated on the held-out validation set when one exists and on the random splits otherwise. After a lower-tail rejection of the batch test, the random-split limits are not accepted (\cref{prop:ratio}), and $G_d$ then requires a held-out validation set, whose rejection fails the consistency criterion. Without one, the state is validation-limited when the precision criterion holds.

The failed criterion determines the recommended procedure. An inconsistent batch calls for re-extracting the features after an upper-tail rejection of the batch test or when the rank condition retains a detected observation. It calls for a revised LiDAR noise model after a rejected radius or trend test, and for recollection after a consensus ambiguity or a rejected validation test. A precision-limited state calls for an observation along the weakest Schur direction (\cref{def:score}). A validation-limited state calls for validation observations along the largest detection-limit direction, held out from the estimation after a lower-tail rejection of the batch test, or, when no target pose constrains that direction, for external checks or a modified design.

Within the internal scope of \cref{prop:unobservable}, a verified state establishes consistency with the stated model, nominal directional precision, and the detectability of the specified differential faults along the tested directions under \cref{ass:model}, and the level and power after observation selection are measured in \cref{subsec:expstate}.
\cref{alg:pipeline} summarizes the complete consistency-guided calibration.

\section{Experiments}
\label{sec:results}
\label{sec:results:sim}
\label{sec:results:real}
This section evaluates VeriCalib in simulation, where the extrinsic is known, and on real-world data recorded on three sensor platforms. \cref{subsec:expcov,subsec:expinfo,subsec:expcons} evaluate the structured covariance, the directional information, and the consistency assessment, and \cref{subsec:ablation} reports an ablation study and the runtime. \cref{subsec:expreal,sec:results:checkerboard,subsec:expstereo} then examine the calibration state across repeated sessions, on a checkerboard target, and against a stereo reference that involves no LiDAR data.

\begin{figure*}[!t]\centering\includegraphics[width=\textwidth]{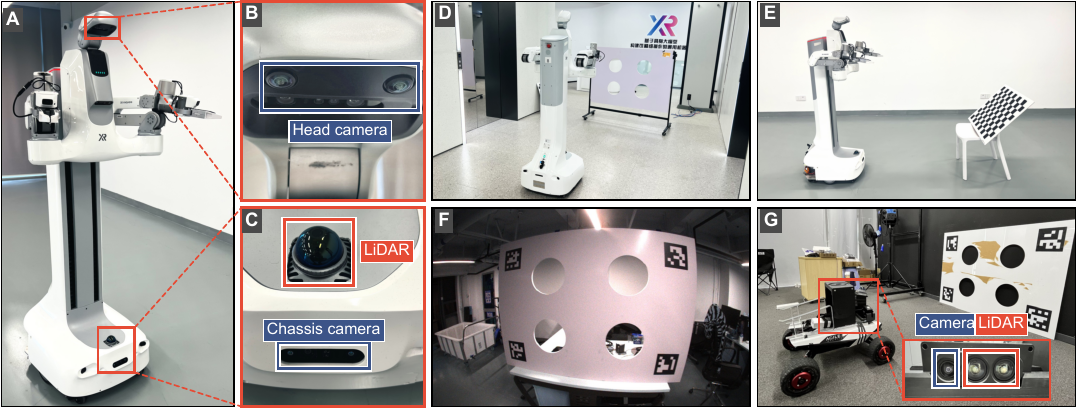}
\caption{Platforms and targets (navy: cameras; red: LiDARs). \textbf{A}--\textbf{C}: mobile manipulator with the head and chassis cameras. \textbf{D}, \textbf{E}: four-hole target and checkerboard. \textbf{F}: handheld platform with stereo cameras and an Ouster LiDAR. \textbf{G}: FS platform with a flash LiDAR.}
\label{fig:setup}\end{figure*}

\subsection{Experimental Setup}
\label{sec:results:setup}
\paragraph{Data and platforms.}
Following the synthetic benchmarks of robust registration~\citep{yang2020teaser} and velo2cam~\citep{beltran2022velo2cam}, a simulator generates hole-center correspondences under a known extrinsic with the noise model of \cref{sec:method} and injects faults. Each synthetic board ($a{=}0.25$\,m, $b{=}0.20$\,m) lies at a depth of $2$ to $5$\,m, with a lateral offset of up to $\pm1.2$\,m and a tilt of up to $40^\circ$ about a random axis. The four camera centers share one board-pose error with a standard deviation of $0.4^\circ$ and $4$\,mm on each axis, and the LiDAR centers have independent in-plane errors that grow with range and incidence, together with the error of a board plane fitted to $240$ rim points. For the circle-fit analysis, LiDAR boundary points are sampled on partial arcs and fitted as on real data. The injected faults are an offset of the four LiDAR centers of one observation, a swap of two hole correspondences, independent or shared offsets of several observations, and one extrinsic perturbation common to all observations.

\label{subsec:datasets}The real-world sets of \cref{tab:datasets} are recorded on four platforms (\cref{fig:setup}). On a mobile manipulator whose head and chassis cameras share one Livox Mid-360 LiDAR (\cref{fig:setup}A--C), S1 and S2 observe the four-hole target on Robot-A five weeks apart, with S2 recorded on consecutive days S2a and S2b, and S3 and S4 observe a checkerboard and the four-hole target on Robot-B, a second unit of the same design. S5 is recorded on a handheld platform with two global-shutter cameras and a spinning Ouster LiDAR, and FS-B and FS-C with a low-cost flash LiDAR~\citep{hu2026p2calib} (\cref{fig:setup}F,G). The four-hole targets share hole centers at $\pm0.25$\,m and $\pm0.20$\,m and marker centers at $\pm0.55$\,m and $\pm0.35$\,m, with a marker side of $0.16$\,m ($0.20$\,m on the FS targets) and a hole radius of $120$\,mm. We exclude acquisitions without a detection in either modality.

\begin{table}[!t]
\centering
\caption{Sensors and observation counts of the real-world sets.}
\label{tab:datasets}\footnotesize\setlength{\tabcolsep}{2.2pt}\renewcommand{\arraystretch}{1.15}
\begin{tabularx}{\columnwidth}{@{}llll>{\hsize=1.25\hsize}C>{\hsize=0.75\hsize}C>{\hsize=1.25\hsize}C>{\hsize=0.75\hsize}C}
\toprule
\rowcolor{tabhead} & & & & \multicolumn{2}{c}{Camera 1} & \multicolumn{2}{c}{Camera 2}\\
\cmidrule(lr){5-6}\cmidrule(lr){7-8}
\rowcolor{tabhead}Platform & Set & LiDAR & Target & FOV [$^\circ$] & $K$ & FOV [$^\circ$] & $K$\\
\midrule
Robot-A & S1 & Mid-360 & Four-hole & $83{\times}53$ & $14$ & $107{\times}91$ & $14$\\
 & S2 & & Four-hole & & $21$ & & $26$\\
\addlinespace[2pt]
Robot-B & S3 & Mid-360 & Checkerboard & $83{\times}53$ & $37$ & $107{\times}91$ & $20$\\
 & S4 & & Four-hole & & $94$ & & $101$\\
\addlinespace[2pt]
Handheld & S5 & Ouster & Four-hole & $80{\times}70$ & $147$ & $80{\times}70$ & $92$\\
\addlinespace[2pt]
FS & FS-B & Flash & Four-hole & $111{\times}78$ & $15$ & -- & --\\
 & FS-C & & Four-hole & & $16$ & -- & --\\
\bottomrule
\end{tabularx}
\par\vspace{1.5pt}
{\scriptsize\raggedright \textit{Note:} FOV: field of view; Camera 1/2: head/chassis (robots), stereo pair (handheld), single camera (FS); $K$: observations per camera-LiDAR pair; H1 to H3: held-out handheld scenes of \cref{app:projection}.\par}
\end{table}

\paragraph{Baselines.}
All methods are applied to the same observations and board configuration.
\begin{itemize}
\item \textbf{FAST-Calib}~\citep{zheng2025fastcalib}\footnote{\url{https://github.com/hku-mars/FAST-Calib}}, run end to end from its official release with default parameters on the board region of the cloud, its LiDAR centers matched to the camera centers by minimum distance;
\item \textbf{velo2cam}~\citep{beltran2022velo2cam}\footnote{\url{https://github.com/beltransen/velo2cam_calibration}} and \textbf{Dom2019}~\citep{domhof2019extrinsic}\footnote{\url{https://github.com/tudelft-iv/multi_sensor_calibration}}, given the camera hole centers of VeriCalib, with their unmodified LiDAR detectors applied to a virtual $128$-ring scan of the cropped board region~\citep{hu2026p2calib} (\cref{app:baselines});
\item \textbf{Ours-U}, \textbf{Indep.}, \textbf{Ours-W}, and \textbf{Ours-WC}, ablated variants on the same observations: the unweighted registration of free-radius circle centers by singular value decomposition (SVD) that the three calibrators above share, none of which reports a covariance or a calibration state; the weighted estimate with an independent camera block for each hole; the weighted estimate without the consistency assessment; and the weighted estimate after the consensus stage alone;
\item \textbf{LOO}, \textbf{RANSAC}, and \textbf{Huber} for outlier removal. LOO removes the observation with the largest leave-one-out influence~\citep{cook1977detection}. RANSAC~\citep{fischler1981ransac} is the correspondence-based implementation of Open3D~\citep{Zhou2018o3d}\footnote{\url{https://github.com/isl-org/Open3D}}, which classifies the hole-center correspondences with an inlier bound of three nominal standard deviations and removes an observation when one of its correspondences is an outlier. Huber reimplements the pose optimization of ORB-SLAM~\citep{mur2017orbslam2,campos2021orbslam3}\footnote{\url{https://github.com/UZ-SLAMLab/ORB_SLAM3}} under the nominal covariance. It applies a Huber kernel at the $95\%$ point of the $\chi^2$ distribution and then classifies each observation by a $\chi^2$ test;
\item \textbf{Random}, \textbf{Far-first}, which adds the most distant candidate first following the intuition that distant structure constrains the rotation~\citep{kaess2009flow}, and three optimal-design criteria~\citep{fedorov1972theory} for acquisition: the entropy of online self-calibration~\citep{maye2016self}\footnote{\url{https://github.com/ethz-asl/aslam_incremental_calibration}} (D-optimal), the covariance trace of next-best-pose camera calibration~\citep{peng2019calibrationwizard}\footnote{\url{https://github.com/pengsongyou/CalibrationWizard}} (A-optimal), and the smallest singular value of the information matrix in observability-aware LiDAR-IMU calibration~\citep{lv2022observability} (E-optimal). Each criterion ranks the candidates greedily on the same extrinsic information as our score (\cref{app:baselines}).
\end{itemize}
Ours denotes the complete procedure.

\paragraph{Metrics.}
\textit{Accuracy and precision.} The rotation and translation errors are $e_R=\lVert\delta\bm\phi\rVert$ and $e_t=\lVert\delta\bm t\rVert$, with $(\delta\bm\phi,\delta\bm t)=\Log(\hatT\,T^{\star-1})$~\citep{yang2020teaser}, where $T^\star$ is the known transform or, on real data, the weighted solution on the correctly extracted observations. An RMSE is computed over hole centers, against the true centers in simulation and between the aligned modalities on real data. The precision is reported as the worst-direction standard deviations $\sigma_R$ and $\sigma_t$ of \cref{eq:dirstd}. The Variance Ratio (VR) divides the empirical estimator variance along a direction by the variance that $\bm H^{-1}$ predicts, and it equals one for a consistent estimator~\citep{barshalom2001estimation,huang2010observability}; for split-half differences, it is the mean squared six-dimensional difference normalized by its predicted covariance, divided by six (\cref{tab:shape}). The error reduction of an estimator is the decrease of its RMS error relative to Ours-U.

\begin{table}[!t]
\centering\caption{Split-half repeatability of the extrinsic under three weightings on the real-world sets.}
\label{tab:shape}\footnotesize\setlength{\tabcolsep}{2.4pt}\renewcommand{\arraystretch}{1.15}
\begin{tabularx}{\columnwidth}{@{}l@{\hspace{3pt}}c*{8}{C}@{}}
\toprule
\rowcolor{tabhead}& & \multicolumn{3}{c}{Rot. [$^\circ$] $\downarrow$} & \multicolumn{3}{c}{Trans. [mm] $\downarrow$} & \multicolumn{2}{c}{VR}\\
\cmidrule(lr){3-5}\cmidrule(lr){6-8}\cmidrule(lr){9-10}
\rowcolor{tabhead}Set & $|\calS|$ & {\scriptsize\mbox{Ours-U}} & {\scriptsize Indep.} & {\scriptsize Joint} & {\scriptsize\mbox{Ours-U}} & {\scriptsize Indep.} & {\scriptsize Joint} & {\scriptsize Indep.} & {\scriptsize Joint}\\
\midrule
S1-H & $14$ & $0.237$ & $0.217$ & \cellcolor{tabbest}$0.211$ & $5.5$ & $4.6$ & \cellcolor{tabbest}$3.7$ & $2.03$ & $1.47$\\
S1-C & $13$ & $0.444$ & $0.342$ & \cellcolor{tabbest}$0.297$ & $12.0$ & $11.1$ & \cellcolor{tabbest}$8.5$ & $2.09$ & $1.37$\\
S2-H & $18$ & $0.141$ & $0.136$ & \cellcolor{tabbest}$0.133$ & $6.1$ & \cellcolor{tabbest}$5.9$ & \cellcolor{tabbest}$5.9$ & $1.51$ & $1.41$\\
S2-C & $23$ & $0.371$ & $0.138$ & \cellcolor{tabbest}$0.126$ & $16.7$ & $5.7$ & \cellcolor{tabbest}$5.1$ & $1.24$ & $0.96$\\
S4-H & $79$ & \cellcolor{tabbest}$0.080$ & $0.093$ & $0.085$ & \cellcolor{tabbest}$2.7$ & $3.6$ & $2.8$ & $2.03$ & $1.57$\\
S4-C & $89$ & $0.135$ & $0.101$ & \cellcolor{tabbest}$0.098$ & $6.0$ & $3.8$ & \cellcolor{tabbest}$3.7$ & $1.53$ & $1.32$\\
FS-B & $15$ & $0.518$ & \cellcolor{tabbest}$0.214$ & $0.298$ & $18.5$ & \cellcolor{tabbest}$7.2$ & $10.7$ & $0.99$ & $1.07$\\
S5-0 & $107$ & \cellcolor{tabbest}$0.204$ & $0.489$ & $0.238$ & \cellcolor{tabbest}$3.3$ & $9.3$ & $4.3$ & $1.42$ & $1.02$\\
S5-1 & $67$ & \cellcolor{tabbest}$0.156$ & $0.412$ & $0.192$ & \cellcolor{tabbest}$3.0$ & $7.9$ & $3.5$ & $1.32$ & $1.08$\\
\bottomrule
\end{tabularx}
\par\vspace{1.5pt}
{\scriptsize\raggedright \textit{Note:} H/C: head/chassis pair; 0/1: camera; median split-half difference over $200$ random splits ($60$ on S4 and S5); Indep./Joint: independent or shared-plane LiDAR covariance; VR: variance ratio (\cref{sec:results:setup}); \colorbox{tabbest}{\strut best} in each row and block.\par}
\end{table}

\textit{Repeatability without ground truth.} Following the held-out evaluation of calibration~\citep{peng2019calibrationwizard}, the Split-Half Difference (SHD) is the median difference between two extrinsics estimated from disjoint random halves of an accepted set. Equal halves minimize the variance of this difference (\cref{app:halfsplit}), which measures the estimator variance and not a bias common to both halves. The delete-one jackknife variance estimate~\citep{efron1979bootstrap} is a second repeatability measure, and perturbing the marker corners or resampling the observations yields the empirical scatter. The Held-out Alignment Error (HAE) of an observation is the mean distance between its camera and LiDAR hole centers under the extrinsic estimated without it, and its Single-Observation Deviation (SOD) is the translation of its own unweighted extrinsic relative to the joint unweighted solution.

\begin{figure*}[!t]\centering\includegraphics[width=\textwidth]{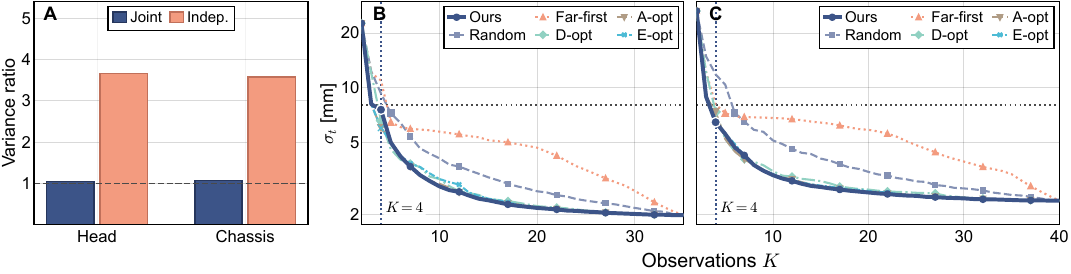}
\caption{Real-data covariance and acquisition. \textbf{A}: variance ratio of the joint and independent-hole (Indep.) camera covariances (dashed: unity). \textbf{B}, \textbf{C}: worst-direction $\sigma_t$ of six orderings from one initial pair on the head and chassis sets (dotted: $8$\,mm target; markers: first $K$ reaching it).}
\label{fig:expcovacq}\end{figure*}

\textit{Fault detection.} The Detected Fraction (DF) is the share of corrupted observations removed, and the False-Removal Rate (FRR) the share of fault-free observations removed. The Detection Rate (DR), or power, is the fraction of trials in which an injected fault is detected, and the False-Alarm Rate (FAR) is the fraction of fault-free batches or splits that a test rejects~\citep{tahiraj2025calornocal,teunissen2018dia}. Paired detection outcomes are compared by the McNemar test~\citep{mcnemar1947note}. The Detection Limit (DL) is the largest $\MDB_{\mathrm{eff}}$ of \cref{eq:mdbeff} over the rotation or translation directions~\citep{baarda1968testing,teunissen2018dia}. Its random-split value is the median over $M$ random equal-half splits of the accepted set, one half fitting the extrinsic and the other certifying it. Its hold-out value certifies the final estimate on a validation set held out before selection (\cref{app:holdout}).

\paragraph{Implementation details.}
VeriCalib is implemented in Python~3.10 with NumPy, SciPy, and OpenCV~4.10. In all experiments we set $\sigma_{L,0}=4$\,mm, $\sigma_{L,1}=10^{-3}$, $\bar\rho=0.05$, $\sigma_r=12$\,mm, a marker-corner pixel noise of $0.5$\,px, $\alpha=0.05$, power $0.8$, $\alpha_{\mathrm{bias}}=0.01$, and a consensus threshold of $28.4$, the $99.72\%$ point of $\chi^2_{11}$, over every observation pair or at most $500$ sampled pairs. Unless mentioned otherwise, the requirements are $(\epsilon_R,\epsilon_t)=(0.15^\circ,8\,\mathrm{mm})$ and $(\delta_R,\delta_t)=3\,(\epsilon_R,\epsilon_t)=(0.45^\circ,24\,\mathrm{mm})$, and the reported DL is the median over $M=200$ random splits ($500$ for the robot-platform batches), whereas the Monte Carlo experiments state their own split counts. The same constants are used for every dataset, the noise magnitude is estimated as $\hat s$, and the test parameters are probabilities, whose sensitivity is reported in \cref{app:alphabias}. All experiments are run on a single thread of an Intel Core i7-12700K.

\subsection{Covariance Prediction and Weighting Accuracy}
\label{subsec:expcov}
\label{sec:results:compact}
\paragraph{Covariance prediction.}
Every reported precision, threshold, and detection limit rests on the predicted covariance, and we therefore first compare the structured covariance and the independent-hole approximation with the empirical scatter on batches of $K{=}8$. With the camera noise generated by the shared board pose, the independent-hole model under-predicts the worst-direction variance by a factor of $1.89$, whereas the rank-six block of~\cref{eq:camcov} yields a ratio of $0.90$, since all four centers inherit the error of one board pose (\cref{subsec:camcov}).\label{subsec:realcov} On real PnP output, with the ArUco corners of the platform observations perturbed by corner noise and the PnP re-solved, the structured covariance matches the scatter of the four hole centers, with median VRs of $1.05$ and $1.07$ on the head and chassis pairs. The independent-hole model under-predicts this scatter by a factor of $3.6$ to $3.7$ (\cref{fig:expcovacq}A).

On simulated partial arcs, the soft radius prior coincides with the free fit at full coverage and reduces the center RMSE from $16.73$ to $7.31$\,mm at $22\%$ coverage (\cref{app:shortarc}). In a noise sweep with $K{=}8$ and $2000$ trials per magnitude, the structured prediction tracks the empirical worst-direction deviation for a camera or LiDAR noise from $0.25\times$ to $4\times$ its nominal value, with VRs of $0.99$ to $1.03$. The ratio of the independent-hole model instead increases with the share of the camera noise, from $1.5$ to $3.1$.

\begin{table}[!t]
\centering\caption{Weighting efficiency on synthetic batches.}
\label{tab:gain}\footnotesize\setlength{\tabcolsep}{3pt}\renewcommand{\arraystretch}{1.15}
\begin{tabularx}{\columnwidth}{@{}l*{4}{C}@{}}
\toprule
\rowcolor{tabhead} & \multicolumn{2}{c}{Ours-U} & \multicolumn{2}{c}{\textbf{Ours-W}}\\
\cmidrule(lr){2-3}\cmidrule(lr){4-5}
\rowcolor{tabhead}PnP noise & $e_R$ [$^\circ$] & $e_t$ [mm] & $e_R$ [$^\circ$] & $e_t$ [mm]\\
\midrule
Constant & $0.111$ & $5.60$ & $0.103$ & $5.09$\\
Depth-dependent & $0.165$ & $7.36$ & $0.115$ & $4.97$\\
Depth-dep.\ + offset (all) & $0.216$ & $10.32$ & $0.189$ & $9.02$\\
Depth-dep.\ + offset (near) & $0.187$ & $9.27$ & $0.177$ & $8.75$\\
\bottomrule
\end{tabularx}
\par\vspace{1.5pt}
{\scriptsize\raggedright \textit{Note:} $K{=}16$, $2000$ trials; offset: unmodeled $6$\,mm LiDAR offset on all or on the nearest third of the observations.\par}
\end{table}

\paragraph{Weighting efficiency.}\label{subsec:expgain}
We evaluate \cref{thm:gain} on synthetic batches with $K{=}16$ and $2000$ trials. Besides the constant PnP covariance of the simulator, a second generator increases the PnP standard deviation with the depth $Z$ as $0.4^\circ\,(Z/3\,\mathrm{m})$ and $4\,\mathrm{mm}\,(Z/3\,\mathrm{m})^2$. This generator is also combined with an unmodeled random offset of $6$\,mm on each axis, applied to every observation or to the nearest third. \cref{tab:gain} shows that the error reduction of the weighting alone over Ours-U increases with the heterogeneity of the observations, from $0.5$ to $2.4$\,mm, and that the RMS errors predicted by part (a) differ from the empirical ones by less than $0.05$\,mm. Part (c) predicts how often a single calibration benefits, and the weighted estimate is closer in $61$ and $77\%$ of the trials. An unmodeled offset, a covariance error in the sense of part (b), reduces the error reduction to $1.3$\,mm, and to $0.5$\,mm when it affects only the nearest observations, to which the estimator assigns the largest weights. On a fixed geometry, the weighted estimate is closer along each axis in $62$ to $69\%$ of the trials, which matches the $63$ to $68\%$ predicted by~\cref{eq:winprob}, and every mis-weighted estimator remains below the bound of~\cref{eq:kantorovich}. An added far observation reduces $\sigma_t$ from $3.9$ to $3.7$\,mm for the weighted estimate but increases it from $4.5$ to $6.3$\,mm for Ours-U.

On real data, \cref{tab:shape} reports the SHD of each accepted set. The shared-plane block of~\cref{eq:lidar3d} yields a smaller rotation difference than an independent normal variance for each hole on eight of the nine sets and a smaller translation difference on seven, with a tie on head S2. The delete-one jackknife shows the same tendency. On the two S5 sets, the independent variance more than doubles the differences of Ours-U, and the shared-plane block removes most of this excess. The handheld motion of S5 adds an observation-level error that neither covariance represents (\cref{subsec:expstereo}). The weighted estimate is also more repeatable than Ours-U on six of the seven robot and FS sets, by $0.01^\circ$ to $0.24^\circ$ and $0.2$ to $11.5$\,mm, and comparable on head S4 (within $0.005^\circ$ and $0.1$\,mm). The improvement is the largest on chassis S2, for which the model predicts the largest difference between the worst-direction translation standard deviations, $8.8$ against $2.1$\,mm. On S5, where the model predicts no measurable gain, the two estimates differ by at most $0.04^\circ$ and $1$\,mm. The same splits also measure how well each covariance predicts this scatter. The VR of the six-dimensional split-half difference is $0.96$ to $1.57$ under the shared-plane block, against $0.99$ to $2.09$ under an independent variance for each hole, and it is closer to one on eight of the nine sets (\cref{tab:shape}).

\paragraph{Model mismatch.}\label{subsec:expmismatch}\label{sec:results:mismatch}
\cref{tab:mismatch} perturbs the generating model with the estimator and tests fixed. A camera-only or LiDAR-only variance inflation alters the ratio of the two covariance blocks (\cref{prop:ratio}). The reported $\sigma_t$ still follows the error, with ratios of the median error to $\sigma_t$ of $1.3$ and $1.1$ against $1.2$ for the nominal model, as \cref{prop:scale} predicts. Under a range-dependent shape error, the reported $\sigma_t$ matches the error ($7.5$\,mm) after scale estimation, and the trend test rejects $14\%$ of trials, against $8\%$ under the nominal model. A common bias remains statistically nominal (\cref{prop:unobservable}).

\begin{table}[!t]
\centering\caption{Covariance-model perturbations on synthetic batches.}
\label{tab:mismatch}\footnotesize\setlength{\tabcolsep}{2.6pt}\renewcommand{\arraystretch}{1.15}
\begin{tabularx}{\columnwidth}{@{}l*{4}{C}@{}}
\toprule
\rowcolor{tabhead}Generating model & $\hat s$ & Rej. [\%] & $\sigma_t$ [mm] & $e_t$ [mm]\\
\midrule
Nominal & $1.00$ & $8$ & $4.4$ & $5.3$\\
Camera noise $\times2$ & $3.21$ & $7$ & $7.9$ & $10.1$\\
LiDAR noise $\times2$ & $1.80$ & $8$ & $5.9$ & $6.5$\\
Range-dependent shape & $2.92$ & $14$ & $7.5$ & $7.5$\\
Common-mode bias & $0.99$ & $6$ & $4.4$ & $9.3$\\
\bottomrule
\end{tabularx}
\par\vspace{1.5pt}
{\scriptsize\raggedright \textit{Note:} $K=12$, $200$ trials; Rej.: trend-test rejection rate; $\sigma_t$, $e_t$: after scale estimation.\par}
\end{table}

\subsection{Directional Precision and Acquisition Efficiency}
\label{subsec:expinfo}
\label{sec:results:directional}
We confirm the parameter-level compensation of \cref{subsec:localinfo} and then test whether the weakest-direction score lowers the number of observations needed for the required precision.

\paragraph{Parameter-level compensation.}
The parameter-level compensation of \cref{prop:parameter} is examined by comparing the conditional blocks $\bm A$ and $\bm C$ of \cref{eq:localinfo}, which hold the complementary block fixed as a diagonal-block analysis does, with the marginal blocks of \cref{eq:lamR}. On the accepted robot-platform and flash-LiDAR sets of \cref{tab:shape}, the conditional blocks understate the worst-direction standard deviation by a factor of $1.16$ to $1.74$ in rotation and $1.83$ to $2.56$ in translation. Their weakest axes differ from the marginal ones by $42^\circ$ to $86^\circ$ in rotation and $62^\circ$ to $89^\circ$ in translation, and along the weakest conditional rotation axis the marginal standard deviation is $58$ to $89\%$ of its worst-direction value. On $2000$ synthetic batches for each $K$, the rotation factor decreases from a median of $1.17$ at $K=4$ to $1.03$ at $K=16$, whereas the translation factor remains about three, and the weakest translation axes differ by more than $30^\circ$ in about $65\%$ of the batches. A diagonal-block analysis therefore overstates the precision and indicates a weakest direction that does not limit the calibration.

\begin{figure*}[!t]\centering\includegraphics[width=\textwidth]{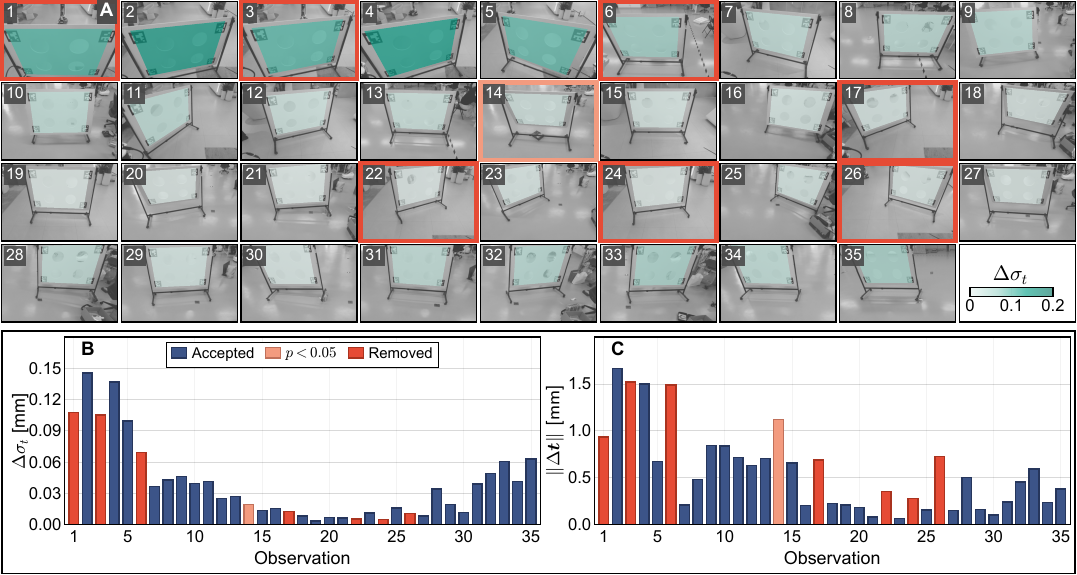}
\caption{Contribution of each observation to the head-pair extrinsic, with the observations sorted by depth. \textbf{A}: board images shaded by the resulting difference $\Delta\sigma_t$ of the worst-direction translation precision (salmon frames: accepted with $p<0.05$; red frames: removed). \textbf{B}: $\Delta\sigma_t$ when the observation is removed from the accepted-set estimate, or added if excluded. \textbf{C}: resulting translation shift $\Vert\Delta\bm t\Vert$.}
\label{fig:boardspace}\end{figure*}

\paragraph{Acquisition guidance.}
The Acquisition Cost (AC) is the number of observations required to reach the precision target~\citep{peng2019calibrationwizard,maye2016self}. Each ordering starts from two initial observations and adds one candidate at a time, on $20$ synthetic candidate sets of $24$ observations and, from $20$ initial pairs, on the $35$-observation head and $40$-observation chassis sets and on the $94$ head and $101$ chassis observations of S4 (\cref{tab:acqcost}).

\label{subsec:realacq}Under the default requirement, the weakest-direction score and the design criteria require comparable observation counts, since they aggregate the same information, and fewer than random and far-board-first ordering require (\cref{fig:expcovacq}B,C), whereas the weakest-direction score also reports the marginal block that fails the requirement. A far-board-first ordering that begins with the two farthest observations requires $14$ and $19$, as \cref{cor:standoff} predicts. On S4, every ordering reaches the $8$\,mm target after a median of at most four observations, and the weakest-direction score reaches $1.5$ times the precision of all recorded observations, the target of \cref{sec:results:checkerboard}, with fewer observations than random and far-board-first ordering.

A requirement separates the selection scores only when it binds. Under a rotation-only requirement, for which the weakest-direction score operates on $\LamR$, $\epsilon_R=0.05^\circ$ is reached on S4 after $21$ head and $33$ chassis observations, against $28$ to $30$ and $39$ to $44$ for the design criteria, and $\epsilon_R=0.10^\circ$ requires no more observations than any design criterion on every set. When the rotation is expressed in degrees and the translation in millimeters, the A- and E-optimal criteria reorder the candidates, which alters the acquisition cost of the E-optimal criterion in $17$ of the $20$ chassis trials, whereas the D-optimal criterion and the weakest-direction score retain their orders in all $20$ trials on each set. The weakest-direction score thus selects for the requirement that fails, independently of the units, and its largest advantage over the design criteria occurs on S4 under the strict rotation requirement, with up to $30\%$ fewer observations.

\begin{table}[!t]
\centering\caption{Median number of observations $\downarrow$ of the selection rules on synthetic and real-world sets.}
\label{tab:acqcost}\footnotesize\setlength{\tabcolsep}{2.6pt}\renewcommand{\arraystretch}{1.15}
\begin{tabularx}{\columnwidth}{@{}l*{4}{C}@{}}
\toprule
\rowcolor{tabhead}Set & Ours & D/A/E & Rand. & Far\\
\midrule
\multicolumn{5}{@{}l}{\textit{$(\epsilon_R,\epsilon_t)=(0.15^\circ,8\,\mathrm{mm})$}}\\
Synthetic & \cellcolor{tabbest}$4$ & \cellcolor{tabbest}$4$ & $6$ & $7$\\
S1+S2 H, C & \multicolumn{2}{c}{\cellcolor{tabbest}$3$--$4$} & $5$--$6$ & $5$--$12$\\
\midrule
\multicolumn{5}{@{}l}{\textit{$1.5\times$ full-set precision}}\\
S4 H & $24$ & \multirow{2}{*}{$23$--$29$} & $41.5$ & $72$\\
S4 C & \cellcolor{tabbest}$22.5$ & & $47.5$ & $79$\\
\midrule
\multicolumn{5}{@{}l}{\textit{$\epsilon_R=0.10^\circ$}}\\
Synthetic & \cellcolor{tabbest}$6$ & $6$--$8$ & $9$ & $10$\\
S1+S2 H & \cellcolor{tabbest}$6.5$ & $7$--$8$ & $13.5$ & $23$\\
S1+S2 C & \cellcolor{tabbest}$14$ & $15$--$18$ & $25$ & $35$\\
\midrule
\multicolumn{5}{@{}l}{\textit{$\epsilon_R=0.05^\circ$}}\\
S4 H & \cellcolor{tabbest}$21$ & $28$--$30$ & $37.5$ & $61$\\
S4 C & \cellcolor{tabbest}$33$ & $39$--$44$ & $52.5$ & $74$\\
\bottomrule
\end{tabularx}
\par\vspace{1.5pt}
{\scriptsize\raggedright \textit{Note:} H/C: head/chassis; D/A/E: range of the optimal-design criteria; Rand./Far: random and far-first; synthetic: $20$ candidate sets of $24$; real-world: $20$ initial pairs; \colorbox{tabbest}{\strut best} per row.\par}
\end{table}

Against the weighted estimate from the remaining observations of each session, the weakest-direction subsets of three to five observations deviate less than random subsets on most single-session sets, and from eight to ten observations random subsets are comparable or closer (\cref{app:subsetacc}).

The calibration state of \cref{subsec:status} also guides collection. On the recorded head S1, head S2, and chassis S2 sets, a loop that runs \cref{alg:pipeline} after each addition and adds the observation that most strengthens the direction indicated by the state meets both the precision and the detectability requirements with the fewest observations on all three sets, and on chassis S2 the unselected observations confirm the verified state in all $10$ replays (\cref{app:stateloop}).

\label{app:boardspace}The contribution of a single observation supports the same ordering. \cref{fig:boardspace} shows the variation of the worst-direction precision and of the translation when each of the $35$ head observations is removed from the accepted-set estimate in turn. The five observations nearest to the camera, at $1.1$ to $1.3$\,m, lower $\sigma_t$ most, whereas far observations contribute less, which is consistent with the larger acquisition cost of far-board-first ordering. No observation shifts the translation by more than $1.7$\,mm, and the shifts of the seven observations outside the accepted set overlap those of the accepted ones.

\subsection{Fault Detection and Detection-Limit Validation}
\label{subsec:expcons}
\label{sec:results:consistency}
The consistency assessment is designed to detect faulty observations at a stated level without knowledge of the noise scale, and to predict the smallest fault that it detects. We evaluate both properties and the resulting calibration state against the number of observations. Localized, independent, and shared faults are injected with $400$ trials for each single-fault and $250$ for each multi-fault condition.

\paragraph{Residual-level compensation and test levels.}
On $500$ fault-free batches of $K=8$, the within-view statistic is calibrated, with an empirical level of $0.056$ at $\alpha=0.05$ and a mean of $1.07$ for each degree of freedom, against $1.05$ under its $F$ reference. A rigid perturbation of one observation by $0.5^\circ$ and $20$\,mm does not alter this rate, as \cref{prop:residual} predicts. Local in-plane faults of $30$\,mm are instead detected with probability $0.854$ for a single correspondence and $0.998$ for a within-observation differential fault. The batch statistic~\cref{eq:batchint} also attains its level, with FARs of $0.044$ to $0.070$ for $K=4$ to $40$ after observation selection.

\begin{table*}[!t]
\centering\caption{Outlier removal on synthetic batches.}
\label{tab:robust}\footnotesize\setlength{\tabcolsep}{3pt}\renewcommand{\arraystretch}{1.12}
\begin{tabularx}{\textwidth}{@{}ll*{9}{C}@{}}
\toprule
\rowcolor{tabhead}& & \multicolumn{4}{c}{Detected [\%]} & \multicolumn{4}{c}{Falsely removed [\%]} & Rej. [\%]\\
\cmidrule(lr){3-6}\cmidrule(lr){7-10}\cmidrule(lr){11-11}
\rowcolor{tabhead}Faulty & Fault & LOO & RANSAC & Huber & \textbf{Ours} & LOO & RANSAC & Huber & \textbf{Ours} & \textbf{Ours}\\
\midrule
\multicolumn{11}{@{}l}{\textit{Noise at the nominal magnitude}}\\
$0/8$ & None & -- & -- & -- & -- & $12.5$ & $2.2$ & $2.8$ & $1.7$ & $10.0$\\
$1/8$ & Offset $30$\,mm & $62.3$ & $86.0$ & $96.0$ & $91.2$ & $5.4$ & $5.4$ & $4.6$ & $1.5$ & $12.8$\\
$1/8$ & Hole swap & $11.0$ & $100.0$ & $100.0$ & $100.0$ & $12.7$ & $2.2$ & $3.0$ & $1.9$ & $10.5$\\
$3/8$ & Indep.\ $30$\,mm & $27.7$ & $83.3$ & $94.3$ & $50.4$ & $3.4$ & $13.3$ & $6.2$ & $1.1$ & $66.8$\\
$3/8$ & Shared $30$\,mm & $17.2$ & $67.6$ & $77.9$ & $9.5$ & $9.7$ & $25.9$ & $13.3$ & $1.5$ & $84.4$\\
$5/8$ & Shared $30$\,mm & $7.8$ & $25.4$ & $14.5$ & $1.1$ & $20.3$ & $65.9$ & $72.9$ & $10.7$ & $87.2$\\
\midrule
\multicolumn{11}{@{}l}{\textit{Noise at twice the nominal standard deviation, unknown to every method}}\\
$0/8$ & None & -- & -- & -- & -- & $12.5$ & $78.3$ & $89.9$ & $1.4$ & $11.2$\\
$1/8$ & Offset $30$\,mm & $32.0$ & $96.8$ & $98.0$ & $18.5$ & $9.7$ & $76.6$ & $88.8$ & $1.5$ & $13.8$\\
$1/8$ & Offset $60$\,mm & $62.3$ & $99.5$ & $100.0$ & $91.2$ & $5.4$ & $76.2$ & $88.9$ & $1.5$ & $12.8$\\
\bottomrule
\end{tabularx}
\par\vspace{1.5pt}
{\scriptsize\raggedright \textit{Note:} $K{=}8$; $400$ trials per row ($250$ with several faults); Rej.: batch-level rejections; Indep./Shared: independent or shared offsets.\par}
\end{table*}

\paragraph{Fault detection.}
\cref{tab:robust} compares the complete procedure with the outlier-removal baselines on batches of eight observations, with the noise at its nominal magnitude and at twice its nominal standard deviation, which is not provided to any method. At the nominal magnitude and a single $30$\,mm fault, VeriCalib detects $91.2\%$ of the faults, against $62.3\%$ for LOO influence (paired McNemar test, $p{<}0.001$), $86.0\%$ for RANSAC, and $96.0\%$ for the Huber kernel. It removes the fewest fault-free observations, $1.5\%$ against $4.6$ to $5.4\%$, and attains a final error of $7.1$\,mm, against $7.3$\,mm for the Huber estimate and $11.6$\,mm for the unweighted registration. With three of eight observations displaced independently, the fixed thresholds of RANSAC and the Huber kernel, which here match the true noise magnitude, detect more faults but remove $13\%$ and $6\%$ of the fault-free observations, whereas VeriCalib holds its FRR at $1.1\%$ and reports $67\%$ of these batches as inconsistent or as requiring a validation set instead of returning a corrupted estimate. Under shared offsets, VeriCalib rejects $67$ to $87\%$ of the multi-fault batches at the batch level, against $10\%$ of the fault-free batches, and none of the baselines reports such a decision.

At twice the nominal noise, RANSAC and the Huber kernel remove $76$ to $90\%$ of the fault-free observations, whereas VeriCalib removes $1.4$ to $1.5\%$, since its thresholds follow from the scale estimated from the batch. Their detection rates then reflect the removal of most observations, whereas VeriCalib detects a $60$\,mm fault in $91.2\%$ of the trials, and its error of $13$ to $15$\,mm remains below the $21$ to $32$\,mm of the Huber estimate (\cref{tab:robust}). Every method except LOO rejects the hole swap in almost every trial, although the swap alters the error of the estimate by about $0.1$\,mm, since the within-view check detects a wrong correspondence independently of its influence. VeriCalib therefore detects moderate faults at a stated FRR without a known noise magnitude, and reports a batch state when observations cannot be removed.

\begin{figure}[!t]\centering\includegraphics[width=\linewidth]{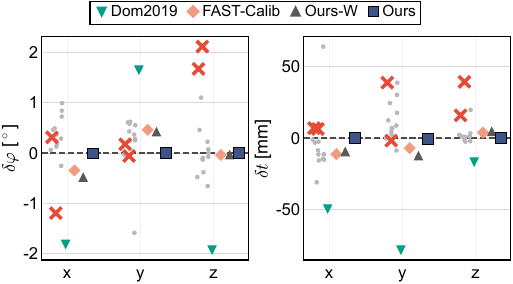}
\caption{Rotation (left) and translation (right) deviations of head S1 extrinsics from the reference (dashed): fault-free single-observation solutions (grey), the two corrupted observations (red crosses), and the joint solutions.}
\label{fig:paramdistfault}\end{figure}

\begin{figure*}[!t]\centering\includegraphics[width=\linewidth]{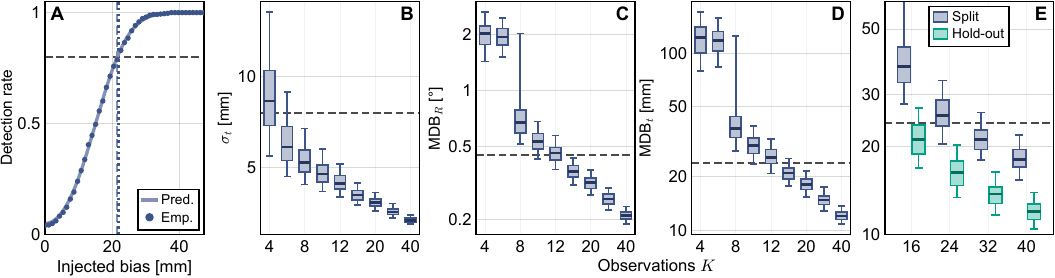}
\caption{Detection limit and calibration state on fault-free synthetic batches. \textbf{A}: detection rate of the focused test against a translation bias (dots: empirical; curve: predicted; dashed: $80\%$ power). \textbf{B}--\textbf{D}: $\sigma_t$, $\MDB_R$, and $\MDB_t$ against $K$ (dashed: requirements). \textbf{E}: hold-out against random-split $\MDB_t$. Boxes: quartiles; whiskers: $5$th to $95$th percentiles; logarithmic axes in C--E.}
\label{fig:storydirectional}\end{figure*}

\label{subsec:realfault}On head session S1, the rim-based hole centers of two observations are wrong by $165$ and $187$\,mm (\cref{fig:paramdistfault}). This is the case of \cref{fig:teaser}. In \cref{fig:paramdistfault}, the two corrupted single-observation solutions deviate by up to $2.1^\circ$ and $40$\,mm along one axis, yet on most axes they lie within the scatter of the fault-free solutions, so that they are not identifiable from the single-observation solutions alone. The official FAST-Calib detector returns four circles on none of the S1 observations. Its estimator, the unweighted registration Ours-U, deviates from $T^\star$ by $0.58^\circ$ and $13.7$\,mm, and the weighting alone by $16.2$\,mm, since the corrupted centers still form a rigid, non-degenerate pattern to which the weighting assigns a high precision. Dom2019 deviates by $3.11^\circ$ and $93.9$\,mm, as its circle detector displaces the hole centers of two detected clouds. VeriCalib removes the two observations and one further observation, and recovers $T^\star$ to within $0.01^\circ$ and $0.32$\,mm. With the nominal covariance, the Huber kernel classifies $12$ of the $14$ observations as outliers, and its estimate deviates by $2.1$\,mm. On FS-C, where five of the $16$ extractions are grossly wrong, the unweighted registration of the extracted centers has an RMSE of $229$\,mm, whereas the consensus stage excludes the five and the diagnosis two further observations. The remaining nine pass the batch test, but the range-trend test rejects them ($t=2.94$), and the state is inconsistent. On the real-world sets, with $\hat s$ from $2.8$ to $29.5$, the Huber kernel with the nominal covariance classifies $12$ of $21$ to $147$ of $147$ observations as outliers, whereas VeriCalib removes $0$ to $40$.

\paragraph{Detection limits.}\label{subsec:expstate}
\cref{fig:storydirectional}A injects a differential translation fault along the weakest validation direction into eight certification observations, with $600$ trials for each magnitude. The empirical $80\%$-power point, $21.8$\,mm, coincides with the predicted $\MDB_{\mathrm{eff}}$ of $21.5$\,mm, of which the finite-sample factor $\kappa_{\bm v}=1.39$ is a large part, and the FAR at zero fault is $0.042$. The same fault applied to all sixteen observations is detected in $5.0\%$ of the trials, as \cref{prop:unobservable} predicts for a common-mode fault.

We repeat this test on real residuals, for which the noise model is not known. Each accepted set is split into random halves, and a differential rigid fault of the predicted $\MDB_{\mathrm{eff}}$ is injected into the certification half along its worst rotation or translation direction, on $300$ splits ($100$ on S4 and S5). On the four sets of S1 and S2 and on FS-B, the fault is detected in $76$ to $85\%$ of the splits for rotation and in $73$ to $89\%$ for translation, against the nominal power of $80\%$. Without a fault, the test rejects $0$ to $9\%$ of the splits, against the nominal level of $5\%$. On S5 the DRs are $82$ to $100\%$ with no rejection without a fault, i.e., its reported covariance is conservative. On S4, the DRs are $79\%$ and $73\%$ on the head pair and $90\%$ and $73\%$ on the chassis pair, with FARs of $0$ to $8\%$, except for the head translation test, which rejects $18\%$ of the fault-free splits, since its reported translation covariance is smaller than the split-half differences (translation VR $1.7$).

\begin{table}[!t]
\centering\caption{Translation error [mm] $\downarrow$ of the cumulative ablation.}
\label{tab:ablation}\footnotesize\setlength{\tabcolsep}{3pt}\renewcommand{\arraystretch}{1.15}
\begin{tabularx}{\columnwidth}{@{}l*{4}{C}@{}}
\toprule
\rowcolor{tabhead}Method & Fault-free & Indep. & Shared & Head S1\\
\midrule
Ours-U & $4.84$ & $9.56$ & $14.09$ & $13.7$\\
Ours-W & \cellcolor{tabbest}$4.48$ & $8.85$ & \cellcolor{tabsec}$13.41$ & $16.2$\\
Ours-WC & \cellcolor{tabbest}$4.48$ & \cellcolor{tabsec}$6.93$ & $13.64$ & \cellcolor{tabsec}$0.37$\\
\textbf{Ours (full)} & \cellcolor{tabsec}$4.52$ & \cellcolor{tabbest}$5.85$ & \cellcolor{tabbest}$13.18$ & \cellcolor{tabbest}$0.32$\\
\midrule
\multicolumn{5}{@{}l}{\textit{Diagnosis}}\\
\quad Retained, cons. & $16$ & $13$ & $13$ & $12$\\
\quad Retained, full & $16$ & $12$ & $13$ & $11$\\
\quad Detected, cons. & -- & $72\%$ & $34\%$ & $2$ of $2$\\
\quad Detected, full & -- & $91\%$ & $40\%$ & $2$ of $2$\\
\bottomrule
\end{tabularx}
\par\vspace{1.5pt}
{\scriptsize\raggedright \textit{Note:} Medians of $300$ trials, $K{=}16$; Head S1: error against $T^\star$; cons.: consensus stage; Indep./Shared: independent or shared offsets; \colorbox{tabbest}{\strut best} and \colorbox{tabsec}{\strut second} per column.\par}
\end{table}
\paragraph{Calibration state.}
\cref{fig:storydirectional}B--D shows how the limiting condition shifts with $K$ over $500$ fault-free batches. Small batches are precision-limited, with $65\%$ in this state at $K=4$. From $K=6$ to $K=12$ the random-split DL exceeds its requirement in $75$ to $91\%$ of the batches ($50$ random splits for each batch). The verified fraction is $20\%$ at $K=12$, $79\%$ at $K=16$, and $89$ to $91\%$ from $K=20$ to $K=40$. The procedure reports $4.6$ to $7.8\%$ of the fault-free batches as inconsistent, below the bound $g\alpha=0.10$, which thus holds empirically also after the selection steps. With half of each batch held out before selection (\cref{fig:storydirectional}E), the hold-out limit of the final estimate is $0.57$ to $0.67$ times the random-split limit. The three validation tests reject $2.5$ to $9\%$ of the fault-free batches, and a fault of the predicted limit is detected in $76$ to $81.5\%$ of the trials against the nominal power of $80\%$. In this hold-out design, $95\%$ of the batches of $K=40$ are verified, and the held-out half confirms $88.5\%$.

On $200$ batches of $16$ estimation and $16$ validation observations, a camera noise of twice its nominal standard deviation and a $30$\,mm offset shared by $6$ estimation observations are both rejected by the batch test in its lower tail, in $100\%$ and $99.5\%$ of the batches, as \cref{prop:ratio} predicts. A LiDAR noise of twice its nominal standard deviation is instead rejected in the upper tail in $90.5\%$. The hold-out tests separate the two conditions, since they pass $84\%$ of the camera-condition batches and $86.5\%$ of the matched batches, but only $11\%$ of the shared-offset batches. With the absorbed-component scale of \cref{eq:scale}, the validation tests therefore remain at their nominal level under the camera-noise mismatch, and $89\%$ of the shared-offset batches are reported as inconsistent and $9\%$ as verified.

\begin{figure*}[!t]\centering\includegraphics[width=\textwidth]{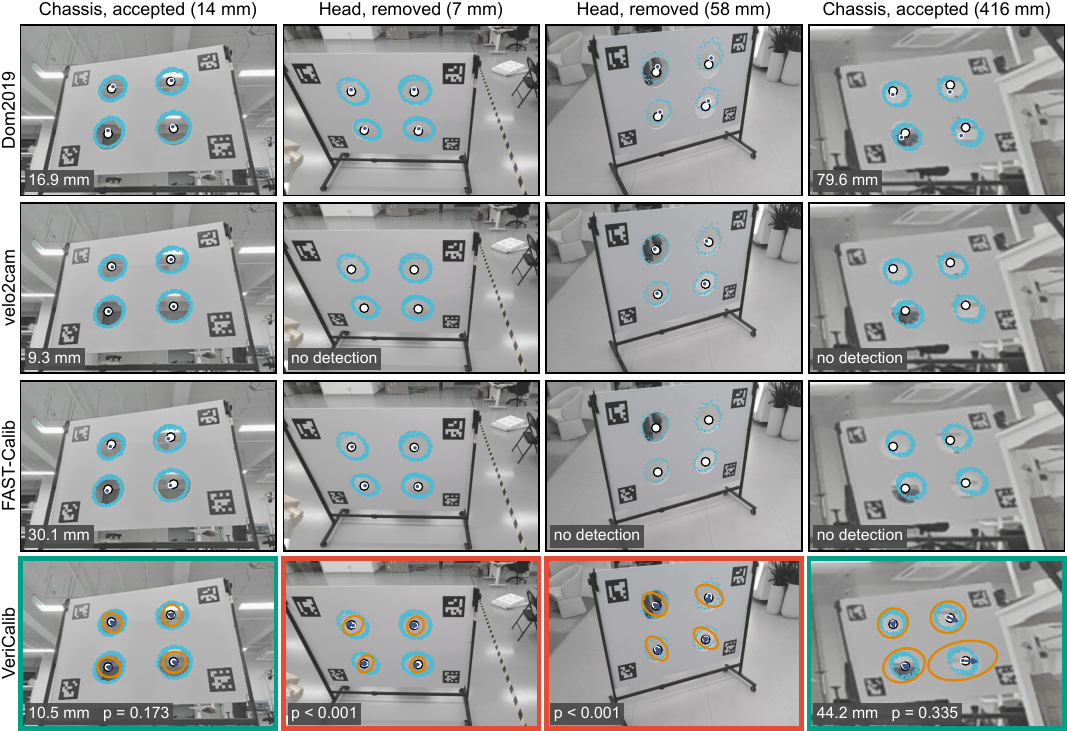}
\caption{Observation decisions on the combined robot-platform batches. Rows: joint extrinsic of each method; columns: an accepted observation, observations removed by the within-view and the cross-view check, and a far three-marker board explained by its covariance. Cyan: projected rim points; white and blue: camera and LiDAR hole centers; labels: held-out alignment error; bottom: $3\sigma$ residual ellipses ($\times5$) and smallest $p$-value; frames: teal, accepted; salmon, $p<0.05$; red, removed.}
\label{fig:status}\end{figure*}

\subsection{Ablation Study and Runtime Analysis}\label{subsec:ablation}

\cref{tab:ablation} adds the components cumulatively on $16$ synthetic observations ($300$ trials) and on head S1. On fault-free data the structured covariance reduces the median error of Ours-U from $4.84$ to $4.48$\,mm. The diagnosis removes $1.9\%$ of the fault-free observations at level $\alpha_{\mathrm{bias}}$, which alters the median error by $0.04$\,mm. Under independent offsets the covariance alone cannot recover the estimate, since a mis-extracted hole is biased and not noisy, and the diagnosis raises the DF from $72\%$ after the consensus stage to $91\%$. A shared offset, which no observation-level test attributes to individual observations, is caught at the batch level: the batch test rejects $79\%$ of the batches, all in the lower tail, at DFs of $34$ and $40\%$. With the other tests of $G_c$, $80\%$ of the batches are reported as inconsistent, $19\%$ as validation-limited because only the lower tail rejects them, and $0.7\%$ as verified.

\begin{table}[!t]
\centering\caption{Runtime [ms] on the head batch.}
\label{tab:timing}\footnotesize\setlength{\tabcolsep}{6pt}\renewcommand{\arraystretch}{1.15}
\begin{tabular}{@{}lccc@{}}
\toprule
\rowcolor{tabhead}Stage & $K{=}5$ & $K{=}14$ & $K{=}35$\\
\midrule
\multicolumn{4}{@{}l}{\textit{Estimation}}\\
\quad FAST-Calib & $22.0$ & $62.5$ & $167.9$\\
\quad \textbf{Ours} & $28.7$ & $81.6$ & $215.8$\\
\midrule
\multicolumn{4}{@{}l}{\textit{Verification (Ours)}}\\
\quad Information & $1.1$ & $2.6$ & $6.0$\\
\quad Diagnosis & $41.9$ & $279.8$ & $1663.5$\\
\quad Full & $311$ & $3259$ & $12584$\\
\bottomrule
\end{tabular}
\par\vspace{1.5pt}
{\scriptsize\raggedright \textit{Note:} Medians of $5$ single-thread runs, extraction excluded; Diagnosis: observation-level tests; Full: \cref{alg:pipeline} with $200$ detection-limit splits.\par}
\end{table}

\label{subsec:timing}\cref{tab:timing} shows that the radius-prior fits, the structured covariance, and the weighted solution require $29$ to $216$\,ms for $5$ to $35$ observations, within a factor of about $1.3$ of the FAST-Calib back end. The complete procedure, with the consensus search, the diagnosis, the state, and $200$ detection-limit splits, processes $35$ observations in less than $13$\,s on a single thread.

\subsection{Calibration States across Repeated Sessions}
\label{subsec:expreal}
\label{sec:results:realexp}
On the robot platform, the observations of one session may be sparse or internally inconsistent, and the extrinsic may vary between sessions. We examine the calibration state of each session and confirm it on an independent recording day.

\paragraph{Calibration states.}
The batch test~\cref{eq:batchint} accepts both pairs on S1 and rejects both pairs on S2 and both combined sets in the lower tail, with ratios $\hat s_{\mathrm{in}}/\hat s$ of $0.61$ and $0.39$ on S2 and $0.27$ and $0.24$ on the combined sets. Such ratios arise from rigid discrepancies between observations or from a camera covariance that is too small (\cref{prop:ratio}). On S1, the head pair is verified, with a precision of $0.066^\circ$ and $1.8$\,mm and random-split DLs of $0.39^\circ$ and $10.4$\,mm, and the chassis pair is validation-limited through its rotation limit of $0.61^\circ$. Both S2 pairs meet the precision requirement, and the lower-tail rejection directs them, as validation-limited, to a held-out validation set, which the cross-day recording below provides for the chassis pair. The combined sets also reject the range-trend test ($t=-3.53$ and $-2.57$) and are inconsistent.

\paragraph{Comparison with the baselines.}
\cref{fig:status} shows these decisions together with the baselines. The adapted velo2cam detector returns four centers on only $7$ of the $35$ head observations, and its extrinsic departs from VeriCalib by $0.73^\circ$ and $24.0$\,mm despite an in-sample RMSE of $10.6$\,mm, whereas Dom2019 departs by $2.54^\circ$ and $42.9$\,mm. The official FAST-Calib detector returns four centers on $19$ head and $24$ chassis observations, none from S1. On the head pair it differs from VeriCalib by $0.30^\circ$ and $13.9$\,mm. On the chassis pair it retains two camera hole-center estimates wrong by $296$ and $274$\,mm and departs by $1.86^\circ$ and $69.8$\,mm, with a held-out error of $49.9$\,mm against $14.5$\,mm. On held-out scenes, every extrinsic that deviates from VeriCalib by $0.7^\circ$ or more aligns significantly worse with the image or, for Dom2019, not significantly better (\cref{app:projection}). FAST-Calib reports only its RMSE, whereas VeriCalib reports both combined batches as inconsistent, which is consistent with the displacement of the LiDAR frame between S1 and S2 analyzed below.

\begin{figure}[!t]\centering\includegraphics[width=\columnwidth]{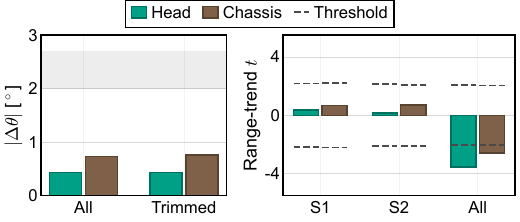}
\caption{Cross-session analysis. \textbf{Left}: rotation difference between S1 and S2 on all observations and without gross errors (Trimmed: SODs above $100$\,mm removed; grey: inter-unit range). \textbf{Right}: residual range trend with level-$\alpha$ thresholds.}
\label{fig:multisession}\end{figure}

Deviation magnitude separates neither corrupted from informative observations nor precise from imprecise ones (\cref{fig:status}). The largest SOD, $416$\,mm, belongs to a far and oblique chassis board with three detected markers. Its HAE of $44$\,mm lies within its PnP standard deviation of $83$\,mm along the depth axis ($p=0.34$), and it is accepted. Conversely, a head observation with a deviation of only $6.6$\,mm is removed, since its within-view test rejects it ($p<0.001$).

\paragraph{Cross-session and cross-day validation.}
\cref{fig:multisession} (left) shows that, after removing the observations with SODs above $100$\,mm, the equal-weight solutions of S1 and S2 differ by $0.43^\circ$ and $7.3$\,mm on the head pair and by $0.76^\circ$ and $18.1$\,mm on the chassis pair. This difference is below the $2.0$ to $2.7^\circ$ separation between units, but it exceeds the session precision by a factor of four to seven. In the LiDAR frame, the two displacements are rotations of $0.49^\circ$ and $0.84^\circ$ about axes within $2^\circ$ of each other. The composed head-to-chassis camera transform, which is invariant to a rigid displacement of the LiDAR frame (\cref{app:composition}), varies by only $0.36^\circ$ and $6.4$\,mm. A large part of the discrepancy is thus one rigid displacement of the LiDAR frame between S1 and S2, which no test within a single session detects (\cref{prop:unobservable}).

\begin{figure*}[!t]\centering\includegraphics[width=\textwidth]{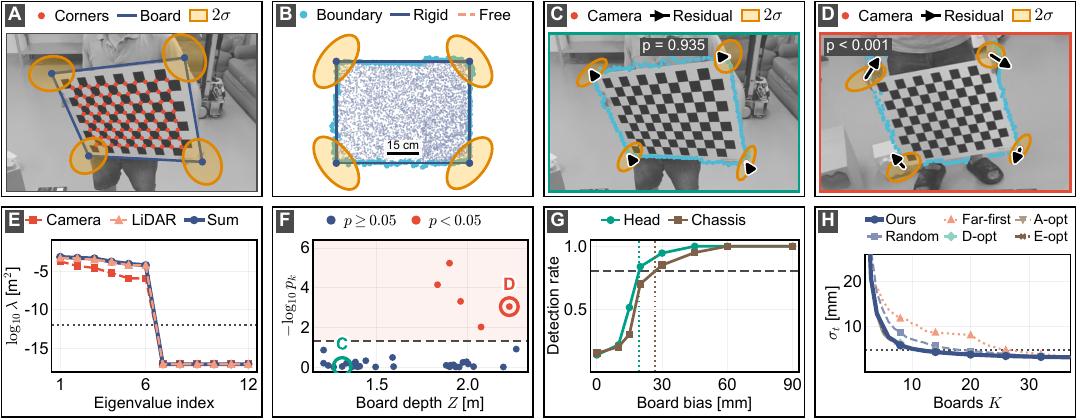}
\caption{Checkerboard target (\textbf{A}--\textbf{F}: head pair; \textbf{G}, \textbf{H}: both pairs). \textbf{A}, \textbf{B}: camera and LiDAR factors. \textbf{C}, \textbf{D}: an accepted (green) and a rejected (red) board with residual ellipses and $p$-values. \textbf{E}: eigenvalues of the residual covariance blocks. \textbf{F}: board consistency against depth. \textbf{G}: detection rate of an injected board bias. \textbf{H}: acquisition (dotted: precision target).}
\label{fig:checkerboard}\end{figure*}

\cref{tab:holdout} holds out one recording day of the chassis pair as the validation set, with the board placed and extracted again on each day. On both estimation days the batch test rejects in the lower tail ($p<0.001$ on both), which directs the procedure to the held-out day (\cref{subsec:status}). The held-out day deviates from the estimate of the other day by $0.099^\circ$ and $5.0$\,mm and by $0.070^\circ$ and $3.8$\,mm, below the hold-out DLs, and no validation test rejects ($p=0.71$ and $0.65$). Both estimates meet the precision requirement, and both DLs meet the detectable-fault requirement. Both days are therefore verified, and each is confirmed by an independent recording day. The absence of a displacement is consistent with a camera noise above its nominal value or with independent rigid errors, whose mean the scale of~\cref{eq:scale} includes in the reported covariance. A rigid offset larger than the DLs would have been detected with power $0.80$.

\begin{table}[!t]
\centering\caption{Cross-day hold-out validation of the chassis pair.}
\label{tab:holdout}\footnotesize\setlength{\tabcolsep}{2.0pt}\renewcommand{\arraystretch}{1.15}
\begin{tabularx}{\columnwidth}{@{}lc*{6}{C}cc@{}}
\toprule
\rowcolor{tabhead}& & \multicolumn{2}{c}{Precision} & \multicolumn{2}{c}{MDB} & \multicolumn{2}{c}{Deviation} & &\\
\cmidrule(lr){3-4}\cmidrule(lr){5-6}\cmidrule(lr){7-8}
\rowcolor{tabhead}Split & $|\calS|/|\calV|$ & [$^\circ$] & [mm] & [$^\circ$] & [mm] & [$^\circ$] & [mm] & $T$ & $p$\\
\midrule
S2a$\to$S2b & $8/17$ & $0.097$ & $3.1$ & $0.39$ & $12.9$ & $0.099$ & $5.0$ & $3.7$ & $0.71$\\
S2b$\to$S2a & $15/9$ & $0.084$ & $3.1$ & $0.32$ & $11.1$ & $0.070$ & $3.8$ & $4.2$ & $0.65$\\
\bottomrule
\end{tabularx}
\par\vspace{1.5pt}
{\scriptsize\raggedright \textit{Note:} S2a$\,\to\,$S2b: S2a estimates, S2b validates; $|\calS|/|\calV|$: accepted/validation observations; $T$: six times~\cref{eq:valtest} (thresholds $13.9$ and $13.3$), with $p$-value $p$.\par}
\end{table}

Session S4 records $94$ head and $101$ chassis observations of Robot-B. The accepted sets of $79$ and $89$ observations reach worst-direction precisions of $0.029^\circ$ and $0.93$\,mm and of $0.036^\circ$ and $1.11$\,mm, and random-split DLs of $0.17^\circ$ and $5.3$\,mm and of $0.21^\circ$ and $6.3$\,mm, well below the requirement. As the batch test rejects both pairs in its lower tail (ratio $0.07$), these DLs direct the procedure to a validation set (\cref{subsec:status}). On the head pair, the estimate from every second observation in recording order is verified and confirmed by the other half, which deviates by $0.056^\circ$ and $2.4$\,mm ($p=0.73$) against hold-out DLs of $0.15^\circ$ and $4.8$\,mm. With alternate observations in range order, the validation tests pass in both directions ($p=0.86$ and $0.51$), and the estimate is verified and confirmed in one of them.

\subsection{Generalization to a Checkerboard Target}
\label{sec:results:checkerboard}
We apply the framework to a standard $11\!\times\!8$ checkerboard ($45$\,mm squares) observed by the head and chassis cameras and the Mid-360 LiDAR of Robot-B (set S3), replacing only the target-specific factor (\cref{sec:checkerboard}). The PnP pose and its covariance are estimated from the $88$ inner corners and propagated to the four outer corners, and the LiDAR boundary is the outermost point of each $1$\,cm strip along every edge (\cref{fig:checkerboard}B). The free-rectangle fit yields an outer size of $0.614\!\times\!0.469$\,m, about $6$\,mm outside the printed frame, which is calibrated once and fixed for both sensors.

The corner correspondence search retains $37$ head and $20$ chassis boards, and on every board the projected residual covariance has rank six (\cref{fig:checkerboard}E). The complete procedure retains $30$ and $17$ boards, and it reports the head batch as inconsistent through the range-trend test ($t=-3.12$) and the chassis batch as precision-limited. On a fixed eight-board simulation, the cross-subset test and the cross-view board test reject $4.0\%$ and $4.8\%$ of the fault-free batches. On the recorded boards, a rigid-board translation injected into each board is detected with $80\%$ probability from $19.4$\,mm on the head and $26.7$\,mm on the chassis (\cref{fig:checkerboard}G). Weakest-direction ordering reaches $1.5$ times the full-set precision in $12$ head and $6$ chassis boards, comparable to the design criteria, against $19$ and $11$ for random and $27$ and $18$ for far-board-first ordering (\cref{fig:checkerboard}H). Therefore, the weighted estimator, the directional information, and the batch tests transfer to the checkerboard, which detects a faulty board through the cross-view test.

\subsection{Accuracy and Fault Detection against a Stereo Reference}
\label{subsec:expstereo}
All preceding real-world references are target-based. The handheld platform of S5 provides a reference independent of the LiDAR: its stereo extrinsic is calibrated from AprilGrid images (Kalibr, $117.6$\,mm baseline) without LiDAR data, and the composition $\hatT_{10}=\hatT_{1}\hatT_{0}^{-1}$ of the two camera-LiDAR estimates is compared with it. As a rigid LiDAR bias cancels in the composition (\cref{app:composition}), the comparison examines the camera-side errors of both estimates. The platform observes the stationary four-hole target during manual motion, with $260$ synchronized single scans over $26$\,s, from which extraction yields $147$ and $92$ observations for the two cameras, and the complete procedure calibrates each pair.

\begin{figure}[!t]
\centering
\includegraphics[width=\linewidth]{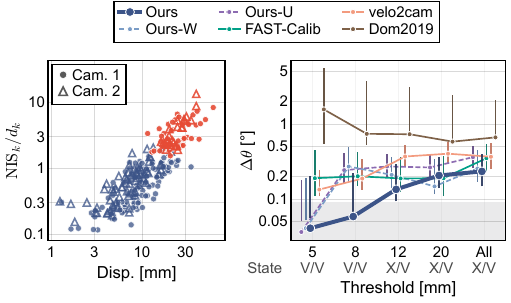}
\caption{Stereo check on S5. \textbf{Left}: $\mathrm{NIS}_k/d_k$ against the predicted board displacement (red: excluded). \textbf{Right}: rotation discrepancy to the stereo reference for each displacement threshold, with states, $10$th to $90$th bootstrap percentiles, and reference uncertainty (grey).}
\label{fig:ousterflir}
\end{figure}

On the accepted observations, the composed VeriCalib transform differs from the stereo reference by $0.23^\circ$ and $3.0$\,mm, about $1.2$ times its predicted worst-direction rotation standard deviation of $0.19^\circ$ (\cref{fig:ousterflir}, right). The official FAST-Calib, velo2cam, and Dom2019 differ by $0.35^\circ$ and $6.5$\,mm, $0.37^\circ$ and $7.0$\,mm, and $0.66^\circ$ and $12.1$\,mm, and the unweighted estimate on the accepted observations by $0.34^\circ$ and $4.4$\,mm. A bootstrap over the observations yields $10$th to $90$th percentile rotation discrepancies of $0.16^\circ$ to $0.40^\circ$ for VeriCalib and $0.23^\circ$ to $0.53^\circ$ for FAST-Calib.

The recording also contains a rigid observation error outside the static-scene model. The spinning LiDAR samples the board about $67$\,ms after the image within its $100$\,ms sweep, so that the board moves between the two samples, by a predicted $11$ and $9$\,mm in the median and up to $61$ and $39$\,mm, as obtained from the IMU rotation and the camera board pose. This displacement moves the four centers of an observation together, which is the error that the extrinsic absorbs at the residual level. The within-view statistic accordingly increases only weakly with the displacement (Spearman $0.14$ and $0.25$) and rejects no observation below $20$\,mm, as \cref{prop:residual} predicts. The cross-view statistic, evaluated at the estimate of the low-motion observations, increases with the displacement (Spearman $0.89$ and $0.91$) and, at the $1\%$ level, rejects $3\%$ of the observations below $5$\,mm and all observations above $20$\,mm. At the batch level, the predicted displacement explains $81$ and $84\%$ of the residual energy, and the normalized residual increases with it (Pearson $\rho=0.82$ and $0.88$, \cref{fig:ousterflir}, left). The procedure accepts $107$ of $147$ and $67$ of $92$ observations, and the batch test rejects both pairs in the lower tail with ratios $\hat s_{\mathrm{in}}/\hat s$ of $0.20$ and $0.34$ ($p<10^{-22}$), the signature of a rigid error in the absorbed subspace (\cref{prop:ratio}). The first pair is inconsistent, as its two competing consensus sets were recorded during opposite horizontal motion, and the second is validation-limited.

Restricting the estimation to the observations whose predicted displacement is below a threshold of the fixed grid of $2$, $3$, $5$, $8$, $12$, and $20$\,mm suppresses this error (\cref{fig:ousterflir}, right), where $2$ and $3$\,mm leave fewer than eight observations. At $5$\,mm ($20$ and $18$ observations), $\hat s$ decreases from $29.5$ and $17.9$ to $5.6$ and $4.6$, both pairs pass the batch test with ratios of $1.01$ and $1.33$, and the first pair becomes validation-limited, as the detection limits of $0.69^\circ$ to $0.81^\circ$ of these small sets exceed the requirement. The discrepancy decreases to $0.04^\circ$ and $1.9$\,mm, within the uncertainty of the reference itself, whose Kalibr covariance implies a root-mean-square rotation error of $0.09^\circ$ and which differs by $0.12^\circ$ from a second Kalibr calibration of the platform. On the same subset, FAST-Calib, velo2cam, and Dom2019 differ by $0.19^\circ$ and $4.9$\,mm, $0.13^\circ$ and $2.4$\,mm, and $1.57^\circ$ and $12.3$\,mm. At $8$\,mm, the consistency assessment excludes seven observations with predicted displacements of $6.6$ to $7.8$\,mm and reduces the discrepancy to $0.06^\circ$, against $0.24^\circ$ for Ours-U and $0.27^\circ$ for the weighting alone on all observations below $8$\,mm.

\section{Discussion}
\label{sec:discussion}
\paragraph{What does the validation-limited state indicate?}
It marks a calibration that meets its precision requirement before its detection limits meet theirs: a set that just meets the precision requirement needs about $3.5$ times its observations to meet the detectability requirement under $(\delta_R,\delta_t)=3(\epsilon_R,\epsilon_t)$ (\cref{subsec:mdb,app:focused}). The real-world sets are consistent with this prediction: the chassis pair of S1 meets the precision requirement but not the detectability requirement, whereas on S4, whose $79$ and $89$ accepted observations reach detection limits well below the requirement, the head-pair estimate from half of the observations is verified and confirmed by the other half. The state also indicates the next action, namely validation observations along the largest detection-limit direction (\cref{subsec:status}), and an independent recording day provided this set and verified the chassis pair of S2.

\paragraph{What does the analysis imply for target placement?}
The rotation precision follows from the orientation diversity and the lateral spread of the boards, and a common increase of their distance leaves it invariant for position-independent covariances (\cref{cor:standoff}). Distant boards lose lateral translation information and, under the range-incidence noise model, also have a larger LiDAR noise. Target placement should therefore vary the board orientation and cover the field of view before it increases the distance, and the weakest-direction score indicates the block that the next observation strengthens. When the detectable-fault requirement applies, the collection should also reserve a validation set, whose information enters the detection limit through \cref{eq:mdbeff}.

\paragraph{When does the weighting not improve the estimate?}
By \cref{thm:gain}, the improvement grows with the heterogeneity of the observation covariances and requires a shape close to the true one. The weighting therefore does not improve the estimate when the covariance reduction predicted by \cref{eq:gain} vanishes, as on S5, where they differ by at most $0.04^\circ$ and $1$\,mm, whereas on chassis S2 the largest predicted difference coincides with the largest gain (\cref{tab:shape}). A fault on a precise observation also displaces the weighted estimate more than the unweighted one (\cref{rem:scope}), as on head S1, which is why the weighting is always applied together with the consistency assessment.

\paragraph{Why is the noise scale estimated?}
On the real-world sets, the residual variance exceeds that of the nominal model by the estimated scale $\hat s$ of $2.8$ to $29.5$. In simulation, thresholds fixed under the nominal magnitude then remove most fault-free observations at twice the nominal noise (\cref{tab:robust}), whereas thresholds referred to an independent scale estimate retain their level, and the reported precision and detection limits scale with $\hat s$ by \cref{eq:scaleeffect}.

\paragraph{Generality.}
The framework requires an observation residual, its supported covariance, and the local extrinsic Jacobian. The checkerboard instantiation shows that all three are available for a target whose modality blocks are both shared-pose, and that transfer also requires checking which consistency statistics retain degrees of freedom.

\paragraph{Scope.}
The verification is internal by construction. A bias common to all observations, such as the extrinsic-equivalent part of an intrinsic error, is indistinguishable from a different extrinsic at any $K$ (\cref{prop:unobservable}), and the calibration state establishes consistency, directional precision, and detectability with respect to the observations of the batch under \cref{ass:model}. Where the batch cannot attribute a rigid discrepancy, the procedure requests a held-out validation set instead of accepting the batch (\cref{prop:ratio}).

\paragraph{Future work.}
The two-level compensation is expected to arise wherever a rigid pose is estimated from a few structured observations, as in hand-eye, multi-LiDAR, and camera-IMU calibration and in registration-based localization, since the pose absorbs the rigid part of each observation error, and its rotation and translation compensate each other. Future work will derive the structured covariance, directional precision, and consistency checks for these problems.

\section{Conclusion}
\label{sec:conclusion}
We formulated the acceptance of a target-based LiDAR-camera calibration as a verifiable estimation problem. The two-level extrinsic compensation explains why a fitted residual reveals neither an observation inconsistency nor the reliability of the extrinsic along each direction, and VeriCalib replaces the inspection of this residual with a calibration state that reports consistency, directional precision, and detection limits at a stated significance level and directs data collection toward the weakest direction. We expect verifiable estimation to become a standard requirement for robotic pose estimation, derived from the structure of the measurement model.

\begin{acks}
The authors thank Shenzhen Fortsense Co., Ltd.\ for providing part of the flash-LiDAR data of this work. The authors also acknowledge the use of Claude (Anthropic) and GPT-6 Astra (OpenAI) for manuscript polishing and for checking the consistency of the equations; the authors reviewed and revised all resulting text and take full responsibility for the content of this article.
\end{acks}

\begin{dci}
The authors declared no potential conflicts of interest with respect to the research, authorship, and/or publication of this article.
\end{dci}

\begin{funding}
The authors received no funding support for the research.
\end{funding}

\bibliographystyle{SageH}
\bibliography{refs}

\appendix
\section{Proofs and Derivations}
\label{app:proofs}\label{app:derivations}
This appendix proves the results of \crefrange{sec:formulation}{sec:method:obs}, the invariance of the stereo composition, and derives the main-text covariance and information expressions, including those of the checkerboard target.

\paragraph{Proof of \cref{prop:unobservable}.}
For an observation subset $\mathcal{A}$, the linearized weighted estimate minimizes $\sum_{k\in\mathcal{A}}(\bm r_k+\bm J_k\delta\bm\xi_0+\bm J_k\delta\bm\xi)^\top(\bm\Sigma^r_k)^{-1}(\bm r_k+\bm J_k\delta\bm\xi_0+\bm J_k\delta\bm\xi)$. The substitution $\delta\bm\xi'=\delta\bm\xi+\delta\bm\xi_0$ reduces it to the problem without bias, and hence $\delta\hat{\bm\xi}_{\mathcal{A}}=\delta\hat{\bm\xi}'_{\mathcal{A}}-\delta\bm\xi_0$ for every $\mathcal{A}$, and the residual of every observation $k$, inside or outside $\mathcal{A}$, is $\bm r_k+\bm J_k\delta\hat{\bm\xi}'_{\mathcal{A}}$, its value without bias. Fitted residuals ($\mathcal{A}=\calS$), cross-validated residuals ($\mathcal{A}=\calS\setminus\{k\}$), and cross-subset differences $\delta\hat{\bm\xi}_A-\delta\hat{\bm\xi}_B$ are therefore invariant, and so is every statistic of \cref{sec:method:tests} built from them, since the information matrices do not depend on the residuals. \cref{app:composition} states the corresponding result across two sensor pairs.

\paragraph{Invariance of the stereo composition to common-mode bias.}\label{app:composition}
As the stereo check composes two camera-LiDAR estimates that share one LiDAR, a bias of that LiDAR must be shown to cancel. Let the LiDAR data of both pairs be perturbed by one common rigid map $\bm T_\Delta\in\SE(3)$, $\bm p\mapsto\bm T_\Delta\bm p$ for every support point and hole center, which is the extrinsic-equivalent common-mode class of the LiDAR. Every residual depends on the LiDAR data only through $\bm T\bm p$, since the LiDAR-frame circle and plane fits are invariant under a rigid map and their outputs transform with $\bm T_\Delta$. The perturbed cost and its minimizer therefore satisfy
\begin{equation}
\left\{
\begin{aligned}
&\mathcal C_i'(\bm T)=\mathcal C_i(\bm T\bm T_\Delta),\qquad \hatT_i'=\hatT_i\bm T_\Delta^{-1},\\
&\hatT_1'(\hatT_0')^{-1}=\hatT_1\bm T_\Delta^{-1}\bm T_\Delta\hatT_0^{-1}=\hatT_1\hatT_0^{-1},
\end{aligned}
\right.
\label{eq:cmcancel}
\end{equation}
so the composition is invariant. In the left-perturbation parametrization the perturbed cost is the same function of the perturbation, $\mathcal C_i'(\Exp(\delta\bm\xi)\hatT_i')=\mathcal C_i(\Exp(\delta\bm\xi)\hatT_i)$, and the information matrices and the composed covariance $\bm S$ are invariant. The result is exact, requires no linearization, and holds for any weighting. The composed check of \cref{subsec:expstereo} therefore is insensitive to extrinsic-equivalent LiDAR bias and examines the differential camera errors.

\paragraph{Proof of \cref{prop:residual}.}
With $\bar{\bm J}_k$ of full column rank, the minimizer of $\lVert\bar{\bm b}_k+\bar{\bm J}_k\delta\bm\xi\rVert$ is $\delta\bm\xi^\star=-(\bar{\bm J}_k^\top\bar{\bm J}_k)^{-1}\bar{\bm J}_k^\top\bar{\bm b}_k$, and the minimized residual $(\bm I-\bm P_k)\bar{\bm b}_k$, with $\bm P_k$ the orthogonal projector onto $\rangeop(\bar{\bm J}_k)$, retains only the component in $\nullop(\bar{\bm J}_k^\top)$. Since $\bm N_k^\top\bm L_k\bm J_k=\bm 0$ by \cref{eq:inttest}, $\bm N_k^\top\bm L_k(\bm r_k+\bm J_k\delta\bm\xi)=\bm e^{\mathrm{in}}_k$ for every $\delta\bm\xi$. Under \cref{ass:model} the whitened residual at the true extrinsic has covariance $\bm L_k\bm\Sigma^r_k\bm L_k^\top=\bm I_{d_k}$, hence $\bm e^{\mathrm{in}}_k$ has covariance $\bm N_k^\top\bm N_k=\bm I_{d_k-6}$, and an error adds only $\bm N_k^\top\bar{\bm b}_k$, which is zero for an error in $\rangeop(\bar{\bm J}_k)$.

\paragraph{Proof of \cref{prop:parameter}.}
With $\bm H\succ\bm 0$, the block $\bm C$ is positive definite, which implies $\bm B\bm C^{-1}\bm B^\top\succeq\bm 0$ and $\LamR=\bm A-\bm B\bm C^{-1}\bm B^\top\preceq\bm A$. Equality requires $\bm x^\top\bm B\bm C^{-1}\bm B^\top\bm x=\lVert\bm C^{-1/2}\bm B^\top\bm x\rVert^2=0$ for every $\bm x$, that is $\bm B=\bm 0$. The translation block follows from the same argument applied to $\LamT$. By the block inversion of $\bm H$, $\LamR^{-1}=[\bm H^{-1}]_{\phi\phi}$ is the rotation covariance with the translation estimated jointly, whereas $\bm A^{-1}$ is the rotation covariance with the translation held fixed, and their difference arises from $\bm B\bm C^{-1}\bm B^\top$ alone.

\paragraph{Proof of \cref{prop:scale}.}
In the whitened coordinates of the shape $\bm\Sigma^r_{0,k}$, the weighted estimate is $\delta\hat{\bm\xi}=-\bm H_0^{-1}\sum_k\bar{\bm J}_k^\top\bar{\bm r}_k$, and $\bar{\bm J}_k^\top\bar{\bm r}_k=\bar{\bm J}_k^\top\bm P_k\bar{\bm r}_k$ with the projector $\bm P_k=\bm V_k\bm V_k^\top$ onto $\rangeop(\bar{\bm J}_k)$. Since $\operatorname{Cov}(\bm P_k\bar{\bm r}_k)=s\,\bm P_k$ and $\bm H_0=\sum_k\bar{\bm J}_k^\top\bar{\bm J}_k$, the covariance of the estimate is $s\,\bm H_0^{-1}$. With $\bm\Omega_k\succ\bm0$ the covariance of the internal component, the inverse of the assumed residual covariance is $\bm W_k=s^{-1}\bm P_k+\bm N_k\bm\Omega_k^{-1}\bm N_k^\top$, which satisfies $\bar{\bm J}_k^\top\bm W_k=s^{-1}\bar{\bm J}_k^\top$. The generalized least-squares estimate under $\bm W_k$ therefore has the same normal equations, and the weighted estimate is the best linear unbiased one. Finally, the stacked absorbed components $\bm V_k^\top\bar{\bm r}_k$ follow a linear model with the design $\bm V_k^\top\bar{\bm J}_k$ of rank six and the covariance $s\,\bm I$, whose residual sum of squares $\sum_k\lVert\bm e^{\mathrm{ab}}_k\rVert^2$ has the expectation $s\,d_{\mathrm{ab}}$.

\paragraph{Proof of \cref{thm:gain}.}\label{app:precision}\label{app:gain}
We prove the three parts of \cref{thm:gain} in turn. All quantities are stacked over the observations and restricted to the support of $\bm\Sigma$, which leaves every estimator invariant, the unweighted one included, since the Jacobian of a four-hole observation has no component along the removed saddle mode (\cref{subsec:residcov}).

We first establish the lower bound and the common-factor case of part (b), on which the equality condition of part (a) relies. Propagating $\operatorname{Cov}(\bm r_k)=\bm\Sigma^r_k$ through the linear estimator~\cref{eq:estcov} weighted by $(\tilde{\bm\Sigma}^r_k)^{-1}$, with residuals independent across observations, yields the sandwich form~\cref{eq:sandwich}. Since the estimator is unbiased for the linearized model~\cref{eq:obsmodel} with $\bm b_k=\bm 0$, the Gauss-Markov theorem applies in the Loewner order,
\begin{equation}
\operatorname{Cov}(\delta\hat{\bm\xi})\succeq\bm H^{-1},
\label{eq:gmbound}
\end{equation}
with equality exactly when the weighted estimator coincides with the correctly weighted one, that is, when a matrix $\bm M\in\R^{6\times6}$ exists such that the stacked matrices satisfy
\begin{equation}
\bm\Sigma\tilde{\bm\Sigma}^{-1}\bm J=\bm J\bm M
\label{eq:gmequality}
\end{equation}
\citep{zyskind1967canonical,rao1967least}. Two weight errors satisfy~\cref{eq:gmequality}. A common factor $\tilde{\bm\Sigma}^r_k=c\,\bm\Sigma^r_k$ cancels in the weighted normal equations,
\begin{equation}
\begin{gathered}
\tilde{\bm H}=\bm H/c,\\
\tilde{\bm H}^{-1}\bm J_k^\top(\tilde{\bm\Sigma}^r_k)^{-1}=\bm H^{-1}\bm J_k^\top(\bm\Sigma^r_k)^{-1},
\end{gathered}
\label{eq:commonfactor}
\end{equation}
so the estimate is invariant and~\cref{eq:sandwich} returns $\bm H^{-1}$. A perturbation confined to the internal subspace,
\begin{equation}
(\tilde{\bm\Sigma}^r_k)^{-1}=(\bm\Sigma^r_k)^{-1}+\bm L_k^\top\bm N_k\bm Z\bm N_k^\top\bm L_k
\label{eq:nullweight}
\end{equation}
with any symmetric $\bm Z$, leaves the normal-equation terms $\bm J_k^\top(\tilde{\bm\Sigma}^r_k)^{-1}\bm J_k$ and $\bm J_k^\top(\tilde{\bm\Sigma}^r_k)^{-1}\bm r_k$ invariant, since $\bm N_k^\top\bm L_k\bm J_k=\bm0$ by~\cref{eq:inttest}. A factor that differs between observations alters the relative weights and in general breaks the equality. Since the validation information also scales as $1/s$, the MDB of \cref{eq:mdb} scales as $\sqrt s$ (\cref{eq:scaleeffect}).

For part (a), with $\bm\varepsilon$ the stacked noise of~\cref{eq:obsmodel}, we write the two errors as $\delta\hat{\bm\xi}=-\bm E\bm\varepsilon$ and $\delta\hat{\bm\xi}_{\mathrm u}=-\bm E_{\mathrm u}\bm\varepsilon$, with $\bm E=\bm H^{-1}\bm J^\top\bm\Sigma^{-1}$ and $\bm E_{\mathrm u}=(\bm J^\top\bm J)^{-1}\bm J^\top$. Since $\bm E_{\mathrm u}\bm\Sigma\bm E^\top=\bm E\bm\Sigma\bm E^\top=\bm H^{-1}$, the difference of the two estimates is uncorrelated with the weighted one,
\begin{equation}
\operatorname{Cov}\bigl(\delta\hat{\bm\xi}_{\mathrm u}-\delta\hat{\bm\xi},\,\delta\hat{\bm\xi}\bigr)=(\bm E_{\mathrm u}-\bm E)\bm\Sigma\bm E^\top=\bm 0,
\label{eq:uncorrelated}
\end{equation}
which is the property of an efficient estimator on which the specification test of \citet{hausman1978specification} relies. Therefore, $\bm V_{\mathrm u}=\bm H^{-1}+\operatorname{Cov}(\delta\hat{\bm\xi}_{\mathrm u}-\delta\hat{\bm\xi})$, which proves~\cref{eq:gain} and recovers the optimality of generalized least squares~\citep{aitken1935least}. The difference vanishes exactly when $\bm E_{\mathrm u}=\bm E$, which is the condition~\cref{eq:gmequality} with $\tilde{\bm\Sigma}=\bm I$, i.e., $\rangeop(\bm\Sigma\bm J)\subseteq\rangeop(\bm J)$~\citep{zyskind1967canonical}.

For the upper bound of part (b), let $\bm X=\tilde{\bm\Sigma}^{-1/2}\bm J$ and $\bm S=\tilde{\bm\Sigma}^{-1/2}\bm\Sigma\tilde{\bm\Sigma}^{-1/2}$, so that $\bm V(\tilde{\bm\Sigma})=(\bm X^\top\bm X)^{-1}\bm X^\top\bm S\bm X(\bm X^\top\bm X)^{-1}$ and $\bm H=\bm X^\top\bm S^{-1}\bm X$. Since $\bm S$ is block diagonal, its eigenvalues are the generalized eigenvalues of the pairs $(\bm\Sigma^r_k,\tilde{\bm\Sigma}^r_k)$ and lie in $[\mu,\omega\mu]$. The upper bound of~\cref{eq:kantorovich} is then the matrix Kantorovich inequality of \citet{marshall1990matrix}.

For part (c), let $a=\bm v^\top\delta\hat{\bm\xi}$ and $g=\bm v^\top(\delta\hat{\bm\xi}_{\mathrm u}-\delta\hat{\bm\xi})$. By~\cref{eq:uncorrelated} and the Gaussian assumption, $a$ and $g$ are independent with $\operatorname{Var}(g)/\operatorname{Var}(a)=\vartheta_{\bm v}-1$. The event $\lvert a\rvert<\lvert a+g\rvert$ is equivalent to $g\,(g+2a)>0$. For $a$ and $g$ divided by their standard deviations, this event is the union of two opposite wedges, each with an opening angle of $\pi/2+\arctan(\tfrac12\sqrt{\vartheta_{\bm v}-1})$, and the rotational symmetry of the standard normal distribution yields~\cref{eq:winprob}.

For the added observation of \cref{subsec:shapescale}, $\bm H_{K+1}\succeq\bm 0$ implies $\bm H+\bm H_{K+1}\succeq\bm H$, and inversion reverses the Loewner order. For the unweighted mean of $K$ scalar measurements with variances $\sigma_i^2$, the variance is $\sum_i\sigma_i^2/K^2$. With an added measurement of variance $\sigma^2$, the variance becomes $(\sum_i\sigma_i^2+\sigma^2)/(K+1)^2$, which is larger whenever $\sigma^2>(2+1/K)\sum_i\sigma_i^2/K$.

\paragraph{Derivation of the radius-prior covariance~\cref{eq:radiuscov}.}\label{app:radiusprior}
The derivation shows where the center-radius trade-off of a short arc affects the covariance and how the prior removes it. Boundary points are modeled as $\bm y_l=\bm c+\varrho\,\bm u(\phi_l)+\bm\varepsilon_l$ with $\bm\varepsilon_l\sim\mathcal N(\bm0,\sigma_L^2\bm I_2)$, and the radial residual $e_l=\lVert\bm y_l-\bm c\rVert-\varrho$ has the gradients $-\hat{\bm u}_l^\top$ in $\bm c$ and $-1$ in $\varrho$. With the prior $\varrho\sim\mathcal N(r,\sigma_r^2)$, the joint information of the center and the radius is therefore~\cref{eq:circinfo} with $N_{\mathrm{eff}}=N$, that is, before the effective-count correction. Down-weighting the $N$ correlated boundary points to $N_{\mathrm{eff}}$ effective points scales the data block by $N_{\mathrm{eff}}/N$, and the Schur complement of~\cref{eq:circinfo} onto $\bm c$ is exactly~\cref{eq:radiuscov}. The two limits of the prior standard deviation, with $\bm\Gamma_{kj}=\sum_l\hat{\bm u}_l\hat{\bm u}_l^\top$, $\bm\gamma_{kj}=\sum_l\hat{\bm u}_l$, $\alpha_{r,kj}\to1/N$, and $\alpha_{r,kj}\to0$, yield the free-radius and the known-radius covariances
\begin{equation}
\bm\Sigma^{\mathtt{L},\mathrm{pl}}_{kj}\to\frac{N\sigma_L^2}{N_{\mathrm{eff}}}
\begin{cases}
\bigl(\bm\Gamma_{kj}-\tfrac1N\bm\gamma_{kj}\bm\gamma_{kj}^\top\bigr)^{-1}, & \sigma_r\to\infty,\\[2pt]
\bm\Gamma_{kj}^{-1}, & \sigma_r\to0,
\end{cases}
\label{eq:radiuslimits}
\end{equation}
respectively. On a partial arc, $\sum_l\hat{\bm u}_l$ is large, which increases the coupling term of~\cref{eq:radiuscov}, and the prior contracts the center covariance along the radial (arc-symmetry) axis, where the center-radius trade-off of the free fit is concentrated. The arc-tangential direction remains weak for short arcs, since no prior on the radius constrains it. \cref{eq:radiuscov} is the regularized (posterior) covariance of the center under the prior. For a fixed radius equal to the nominal value, it bounds the sampling covariance of the regularized center from above along the radial axis. The two coincide on a complete arc and diverge as the angular coverage decreases, and this direction-dependent margin is conservative and is not removed by the global scale of \cref{subsec:scaleest}.

\paragraph{Marginal information under rank deficiency.}\label{app:rankdef}
The decoupling of~\cref{eq:lamR} is stated through an inverse of the translation block, which a rank-deficient geometry does not provide. A variational form removes that restriction. For $\bm H\succeq0$,
\begin{equation}
\bm v^\top\LamR\bm v=\min_{\bm u}\,[\bm v;\bm u]^\top\bm H\,[\bm v;\bm u]
=\bm v^\top\bigl(\bm A-\bm B\bm C^{\dagger}\bm B^\top\bigr)\bm v
\label{eq:schurvar}
\end{equation}
with the Moore-Penrose inverse. The minimum is always attained, since $\bm H\succeq0$ implies $\rangeop(\bm B^\top)\subseteq\rangeop(\bm C)$, and~\cref{eq:schurvar} reduces to~\cref{eq:lamR} when $\bm C\succ0$. The minimization also determines the interpretation of every decoupled quantity, since the nuisance block is co-adjusted to compensate the perturbation optimally, and a decoupled directional precision is the worst case over that compensation, the conservative reading adopted throughout.

\paragraph{Derivation of the multi-observation information~\cref{eq:multi}.}\label{app:multi}
The decomposition expresses rotation information as an inertia, in which orientation diversity and board spread are the two geometric quantities that determine it. With equal isotropic weights over the $n=4K$ hole centers, the blocks of~\cref{eq:localinfo} are
\begin{equation}
\left\{
\begin{aligned}
&\bm A=\tfrac{1}{\sigma^2}\textstyle\sum_i\bigl(\lVert\bm x_i\rVert^2\bm I-\bm x_i\bm x_i^\top\bigr),\\
&\bm B=\tfrac{1}{\sigma^2}\textstyle\sum_i\skewop{\bm x_i},\qquad \bm C=\tfrac{n}{\sigma^2}\bm I_3 .
\end{aligned}
\right.
\label{eq:isoblocks}
\end{equation}
The coupling term is then the inertia of the centroid $\bar{\bm x}$,
\begin{equation}
\bm B\bm C^{-1}\bm B^\top=\tfrac{n}{\sigma^2}\bigl(\lVert\bar{\bm x}\rVert^2\bm I-\bar{\bm x}\bar{\bm x}^\top\bigr),
\label{eq:paraxis}
\end{equation}
so, by the parallel-axis theorem, $\LamR$ is the point-cloud inertia about the centroid, and the standoff of the whole set contains no rotation information once translation absorbs it. Substituting $\bm x_{kj}=\bm c_k+\bm Q_k\bm a_j$ and applying $\sum_j\bm a_j=\bm0$, which eliminates all cross terms and yields $\bar{\bm x}=\bar{\bm c}$, splits the centered inertia into the rotated hole-layout inertia and the board-center-cloud inertia, which is~\cref{eq:multi}.

\paragraph{Checkerboard corner covariance.}\label{app:cornercov}\label{subsec:cornercov}
As the LiDAR does not observe the printed pattern, the four cross-modal correspondences are the outer corners of the board, recovered from the scanned boundary. The edge-point noise follows the range-incidence model~\cref{eq:lidarnoise}. Whereas the circular-hole construction fits one circle to each hole, the checkerboard construction fits one rigid rectangle of known outer size $(w_b,h_b)$ to the boundary points of the whole board. Let $\bm U_k=[\bm t_{k1}\;\bm t_{k2}]$ be the fitted plane basis and $(\bm b_i)$ the in-plane boundary points. The rectangle has an in-plane pose $\bm g=(\bm\tau,\psi)$ with center $\bm\tau$ and yaw $\psi$. Its four edges have in-plane outward normals $\bm m_e(\psi)$ and known half-extents $\eta_e\in\{w_b/2,h_b/2\}$. The signed distance of a boundary point assigned to edge $e$ is $d_i(\bm g)=\bm m_e(\psi)^\top(\bm b_i-\bm\tau)-\eta_e$, with the known half-extent $\eta_e$ held fixed as the direct analog of the known radius in~\cref{eq:radiuscov}. Its gradient in $\bm g$ is $\bm j_i=[-\bm m_e^\top\;\;(\bm m_e^\perp)^\top(\bm b_i-\bm\tau)]$, where $\bm m_e^\perp=\partial_\psi\bm m_e$. With the same effective-count correction $N_{\mathrm{eff}}/N$ of~\cref{eq:neff} that down-weights correlated accumulated points, the rigid-pose information and covariance are
\begin{equation}
\bm\Lambda_g=\frac{N_{\mathrm{eff}}}{N\sigma_L^2}\sum_i\bm j_i\bm j_i^\top,\qquad
\bm\Sigma_g=\bm\Lambda_g^{-1}.
\label{eq:rectinfo}
\end{equation}
Fixing the two half-extents removes the size directions, which are weak for a free rectangle on a partially scanned board.

Since the four corners share the single pose $\bm g$, their in-plane positions $\bm q_j=\bm\tau+\bm R_\psi\bm a_j$ have the correlated covariance
\begin{equation}
\bm\Sigma^{\mathtt{L},\mathrm{pl}}_{k,jl}=\bm G_j\bm\Sigma_g\bm G_l^\top ,\qquad
\bm G_j=[\,\bm I_2\;\;\bm R_\psi\bm J_\perp\bm a_j\,],
\label{eq:cornercov}
\end{equation}
with $\bm R_\psi$ the in-plane rotation and $\bm J_\perp=\left[\begin{smallmatrix}0&-1\\1&0\end{smallmatrix}\right]$. Stacked over four corners, $(\bm\Sigma^{\mathtt{L},\mathrm{pl}}_{k,jl})_{jl}$ has rank three and is the LiDAR counterpart of the rank-six shared-pose camera block, with the rigid in-plane board pose replacing the camera board pose.

The fitted plane contributes three further shared parameters, which describe its motion normal to the board. For boundary coordinates $(u_i,v_i)$, let $\bm X_i=[1\;u_i\;v_i]$ and
\begin{equation}
\bm\Sigma_\pi=\sigma_\perp^2\frac{N}{N_{\mathrm{eff}}}(\bm X^\top\bm X)^{-1},\qquad
\bm P_j=\bm n_k[1\;u_j\;v_j],
\label{eq:cornerindep}
\end{equation}
where the in-plane and normal errors are uncorrelated to first order, and the covariance of the rigid corners is
\begin{equation}
\bm\Sigma^{\mathtt{L}}_{k,jl}=\bm U_k\bm G_j\bm\Sigma_g\bm G_l^\top\bm U_k^\top+
\bm P_j\bm\Sigma_\pi\bm P_l^\top.
\label{eq:ckbdlidar3d}
\end{equation}
Stacked over the four corners, this covariance has rank six, three in-plane pose parameters and three plane parameters, and no independent corner term is added to complete it, since the corners are generated by the rigid fit alone. Both modality blocks are then represented on the common rigid-board tangent space at the camera-board linearization, and the residual covariance is projected onto that space before whitening. With the positive eigendecomposition $\bm\Sigma^r_{0,k}=\bm U_d\bm D_d\bm U_d^\top$, the whitening becomes
\begin{equation}
\begin{gathered}
\bm L_k=\bm D_d^{-1/2}\bm U_d^\top,\qquad d_k=\rank(\bm\Sigma^r_{0,k}),\\
\bar{\bm r}_k=\bm L_k\bm r_k,\qquad \bar{\bm J}_k=\bm L_k\bm J_k,
\end{gathered}
\label{eq:supportwhitening}
\end{equation}
the information is $\bar{\bm J}_k^\top\bar{\bm J}_k$, and every residual dimension of a rigid board is absorbed, so that $d_k=\rank(\bar{\bm J}_k)=6$ and the scale denominator of \cref{eq:scale} is $d_{\mathrm{ab}}=6|\calS|-6$, as for the four-hole target. A non-degenerate rigid board thus has no within-observation degree of freedom, which renders the within-view check of \cref{subsec:diagnosis} not applicable, and the consistency assessment relies on the cross-observation test of \cref{eq:crosssubset}. A free rectangle fitted to the same boundary monitors the known dimension, as the free circle fit monitors the radius, and a systematic size gap is determined once for the board and fixed for both sensors (\cref{sec:results:checkerboard}).

\section{Consistency Statistics}
\label{app:aux}
This appendix derives the deflation and influence results of \cref{sec:method:tests}, proves \cref{prop:ratio}, and states the radius diagnostics, the observation-level, subset, and trend statistics, the scale references, the consensus search, and the hold-out confirmation with its focused statistic.

\paragraph{Derivation of~\cref{eq:deflation}.}\label{app:insample}
The fitted residual is smaller than the measurement error, since the fit has already absorbed part of it, and the derivation quantifies this reduction. To first order,
\begin{equation}
\left\{
\begin{aligned}
&\hat{\bm r}_k=\bm r_k+\bm J_k\delta\hat{\bm\xi},\qquad
\delta\hat{\bm\xi}=-\bm H^{-1}\textstyle\sum_l\bm J_l^\top(\bm\Sigma^{r}_l)^{-1}\bm r_l,\\
&\mathbb E[\delta\hat{\bm\xi}\,\bm r_k^\top]=-\bm H^{-1}\bm J_k^\top,\qquad
\bm\Sigma_{\delta\hat{\bm\xi}}=\bm H^{-1} .
\end{aligned}
\right.
\label{eq:insampleproof}
\end{equation}
Expanding $\operatorname{Cov}(\hat{\bm r}_k)$ with~\cref{eq:insampleproof} yields $\bm\Sigma^{\hat r}_k=\bm\Sigma^{r}_k-\bm J_k\bm H^{-1}\bm J_k^\top$, and the trace against $(\bm\Sigma^{r}_{0,k})^{-1}$ with $\bm\Sigma^r_k=s\,\bm\Sigma^r_{0,k}$ yields
\begin{equation}
\begin{gathered}
\mathbb E[\NIS_k]=s\bigl({d_k}-\tr(\bm H^{-1}\bm H_k)\bigr),\\
\textstyle\sum_k\tr(\bm H^{-1}\bm H_k)=\tr(\bm I_6)=6 ,
\end{gathered}
\label{eq:leveragesum}
\end{equation}
which is~\cref{eq:deflation}. The six fitted parameters thus remove six degrees of freedom, distributed according to the leverages.

\paragraph{Expected influence of a fault-free observation.}\label{app:influence}
The influence of an observation is set by its geometry and not by its correctness. Influence alone therefore cannot detect a fault. At the true extrinsic, $\operatorname{Cov}(\bm r_k)=\bm\Sigma^{r}_k$ yields
\begin{equation}
\left\{
\begin{aligned}
&\mathbb E[\bm g_k\bm g_k^\top]=\bm J_k^\top(\bm\Sigma^{r}_k)^{-1}\bm J_k=\bm H_k,\\
&\mathbb E[I_k]=\tr(\bm H^{-1}\bm H_k)=h_k .
\end{aligned}
\right.
\label{eq:influenceproof}
\end{equation}
A fault-free observation is therefore expected to displace the solution by its own leverage. At the estimate, the deflation of \cref{eq:deflation} shrinks the expectation to $h_k-\tr((\bm H^{-1}\bm H_k)^2)$, which also depends on the geometry alone.

\paragraph{Proof of \cref{prop:ratio}.}
With $\bm x_{kj}=\bm c_k+\bm Q_k\bm a_j$ at the fit, \cref{eq:camjac,eq:jacobian} give $\bm J_{kj}=\bm F_{kj}\bm G_k^{-1}$ with $\bm G_k^{-1}=\bigl[\begin{smallmatrix}\bm I_3&\bm0\\-\skewop{\bm c_k}&\bm I_3\end{smallmatrix}\bigr]$, and hence $\bm F_k=\bm J_k\bm G_k$. The error $\bm J_k\bm\beta_k$ therefore has covariance $\gamma\,\bm J_k\bm G_k\bm\Sigma_{\theta,k}\bm G_k^\top\bm J_k^\top=\gamma\,\bm F_k\bm\Sigma_{\theta,k}\bm F_k^\top=\gamma\,\bm\Sigma^{\mathtt{C}}_k$. As it is Gaussian and independent of the other errors, the residual of observation $k$ has the covariance of \cref{eq:residcov} with $\bm\Sigma^{\mathtt{C}}_k$ replaced by $(1+\gamma)\,\bm\Sigma^{\mathtt{C}}_k$, and the batch jointly follows the factor model with this camera block. The weighted estimate is linear in the residuals to first order, and a zero-mean error does not bias it.

\paragraph{Radius diagnostics.}\label{app:radius}
If the free fit of hole $j$ in observation $k$ yields the radius $\hat\varrho_{kj}$ with marginal variance $\sigma^2_{\hat\varrho,kj}$ from the inverse of~\cref{eq:circinfo} before adding the prior, the radius diagnostic of each hole is
\begin{equation}
T^{r}_{kj}=
\frac{(\hat\varrho_{kj}-r)^2}
{\sigma^2_{\hat\varrho,kj}+\sigma_r^2}
\label{eq:radiuscheck}
\end{equation}
referred to the $\chi^2_1$ distribution. This reference is exact for a radius distributed as the prior and conservative for a fixed radius equal to $r$. For the batch-level statistic~\cref{eq:radiuspool} and a fixed radius equal to $r$, the reference is likewise conservative, and it tolerates a class-level deviation within the prior standard deviation, which sets its detection limit. The induced center displacement $(\sum_l\hat{\bm u}_l\hat{\bm u}_l^\top)^{-1}\sum_l\hat{\bm u}_l\,(\bar{\hat\varrho}-r)$, the coupling term of \cref{eq:radiuscov}, is reported for the most partial arc as the bias budget of \cref{subsec:lidarcov}. For a fully sampled boundary, $\sum_l\hat{\bm u}_l\!\approx\!\bm0$ and the prior is inert.

\paragraph{Observation-level and subset statistics.}\label{app:tests}
The studentized fitted-residual statistic is
\begin{equation}
T^{\mathrm{sd}}_k=\hat{\bm r}_k^\top\bigl(\bm\Sigma^{\hat r}_k\bigr)^{+}\hat{\bm r}_k\sim\chi^2_{\rank(\bm\Sigma^{\hat r}_k)},
\label{eq:stdnis}
\end{equation}
with $\bm\Sigma^{\hat r}_k$ the fitted-residual covariance of~\cref{eq:deflation}. For $T^{\mathrm{in}}_k$ the scale $\hat s_{\mathrm{in},-k}$ is the internal energy of the other observations, which does not involve $\bm r_k$. For $T^{\mathrm{cv}}_k$ and $T^{\mathrm{sd}}_k$ it is the cross-fitted working scale $\hat s_{-k}$ of the complement, computed from the fitted residuals of the leave-one-out solution,
\begin{equation}
\begin{gathered}
\hat s_{\mathrm{in},-k}=\frac{\sum_{l\ne k}T^{\mathrm{in}}_l}{\sum_{l\ne k}\bigl(d_l-\rank(\bar{\bm J}_l)\bigr)},\\
\hat s_{-k}=\frac{\sum_{l\ne k}\NIS_l(\hatT_{-k})}{{\sum_{l\ne k}d_l}-6}.
\end{gathered}
\label{eq:loscale}
\end{equation}
Since the residuals of a Gaussian least-squares fit are independent of its estimate, $\hat s_{-k}$ is independent of $\hatT_{-k}$ and of $\bm r_k$, and, for a fixed full-rank four-hole Gaussian model before selection,
\begin{equation}
\begin{gathered}
\frac{T^{\mathrm{cv}}_k}{11\,\hat s_{-k}}\sim F_{11,\,11(|\calS|-1)-6},\\
\frac{T^{\mathrm{in}}_k}{5\,\hat s_{\mathrm{in},-k}}\sim F_{5,\,5(|\calS|-1)},
\end{gathered}
\label{eq:obsF}
\end{equation}
with the analogous distribution for $T^{\mathrm{sd}}_k$. The cross-fitted reference mixes both subspaces in the model proportion, and $T^{\mathrm{cv}}_k$ and $T^{\mathrm{sd}}_k$ are therefore mean-matched but not exactly calibrated under a scale mismatch between them.

A cross-subset Mahalanobis distance detects differential bias when two sufficiently large subsets converge to different solutions. For disjoint subsets $A,B$,
\begin{equation}
\left\{
\begin{aligned}
&\bm d_{AB}=\Log(\hat{\bm T}_{A}\hat{\bm T}_{B}^{-1}),\\
&\bm S_{AB}\doteq \bm H_A^{-1}+\bm H_B^{-1},\\
&M_{AB}=\bm d_{AB}^\top\bm S_{AB}^{-1}\bm d_{AB}\sim\chi^2_6 .
\end{aligned}
\right.
\label{eq:subsettest}
\end{equation}
Under the null, $\bm d_{AB}\doteq\bm\epsilon_A-\bm\epsilon_B$ in the common left-perturbation convention, and its covariance is the sum of the two local covariances. With an estimated scale, the shape-only covariance $\bm C_{AB}=\bm H_{0,A}^{-1}+\bm H_{0,B}^{-1}$ and the joint residual scale $\hat s_{AB}=(\mathrm{RSS}_A+\mathrm{RSS}_B)/(\sum_k d_k-12)$ yield
\begin{equation}
F_{AB}=\frac{\bm d_{AB}^\top\bm C_{AB}^{-1}\bm d_{AB}}{6\,\hat s_{AB}}
\sim F_{6,\,\sum_kd_k-12},
\label{eq:crosssubset}
\end{equation}
provided that each subset has at least two observations and full-rank information. The distribution is exact for a fixed linear Gaussian design. For a rigid checkerboard, whose internal statistic has no degrees of freedom, the batch criterion evaluates~\cref{eq:crosssubset} in its upper tail on a fixed odd/even split of the accepted set, which compares two rigid poses and is insensitive to the board shape. The consensus stage applies the same distance to overlapping candidate sets as an uncalibrated ambiguity check. Like every internal comparison, $M_{AB}$ is insensitive to a bias shared by both subsets (\cref{subsec:scope}).

\paragraph{Residual-trend test.}
For each accepted observation, the normalized residual $e_k=\NIS_k/d_k$ is regressed by ordinary least squares on the LiDAR range $\rho_k$ of the centroid of its hole centers,
\begin{equation}
e_k=\beta_0+\beta_1\rho_k+\varepsilon_k,\qquad
t=\hat\beta_1/\widehat{\mathrm{se}}(\hat\beta_1),
\label{eq:trend}
\end{equation}
and a range-dependent shape error is reported when $|t|$ exceeds the two-sided level-$\alpha$ point of the Student $t$ distribution with $|\calS|-2$ degrees of freedom. Under a covariance of correct shape, $e_k$ has expectation $s(1-h_k/d_k)$ by~\cref{eq:leveragesum}. It is therefore uncorrelated with range up to the leverage term, whose batch mean is $6/(d_0K)$ for equal $d_k=d_0$, and a misspecified range-growth coefficient induces a slope beyond this term. Neglecting the leverage term makes this threshold an approximation, whose false-alarm rate is checked in \cref{subsec:expcov}. The scale $\hat s$ is reported with the tests as a model-mismatch indicator. A large scale with a large discrepancy indicates inconsistent observations, whereas a large scale with a small discrepancy indicates a noise model that is too small but approximately proportional.

\paragraph{Scale references.}
A statistic divided by a scale estimated from the same residuals is not $\chi^2$ distributed, since the batch statistic~\cref{eq:batchint} is itself a large part of $\sum_k\NIS_k$. Its ratio to $\hat s$ therefore rejects fewer fault-free batches than the nominal level at the $\chi^2$ threshold. Three scales are therefore distinguished. The working scale $\hat s_1$ is applied only to the evaluation of minimal consensus sets at a fixed threshold, the independent references of~\cref{eq:ftest} calibrate the tests, and the accepted-set scale $\hat s$ of \cref{eq:scale} determines the reported precision and detection limits. The absorbed component $\bm e^{\mathrm{ab}}_k$ from which $\hat s$ is formed is, after the joint fit, the part of the absorbed residual that the common extrinsic cannot remove because the observations require different corrections. It is independent of every internal component $\bm e^{\mathrm{in}}_k$ under the null, since the estimate depends on the range-space components only, and the internal reference of $T^{\mathrm{in}}_k$ therefore remains calibrated under a scale mismatch between the two subspaces.

\paragraph{Consensus search.}
The consensus search seeds minimal sets with every pair of observations, or with a fixed number of pairs drawn with a fixed seed in an order defined by the observation content when more pairs exist. The outcome thus depends on the set of observations and is independent of their order. It operates under the block covariance scaled by $\hat s_1$, with a fixed tail probability mapped to the threshold of $\chi^2_{d_k}$ for each observation. The total NIS of each consensus set of maximal size is computed at its own weighted fit, and the set with the smallest total is returned as the dominant set $\calS_0$.

\paragraph{Hold-out confirmation.}\label{app:holdout}
The validation set $\calV$ of \cref{subsec:mdb} is evaluated at the fixed final estimate, whose fit information is $\bm H_{\mathrm f}=\bm H$. Since a re-estimation that includes $\calV$ would require a new confirmation, $\calV$ is not merged into the reported estimate. The six-degree-of-freedom validation test uses the validation-only correction $\bm c$,
\begin{equation}
\frac{\bm c^\top\bigl(\bm H_{0,\mathrm v}^{-1}+\bm H_{0}^{-1}\bigr)^{-1}\bm c}{6\,\hat s}\sim F_{6,\,d_{\mathrm{ab}}},\qquad
\bm c=-\bm H_{0,\mathrm v}^{-1}\bm g_{0,\mathrm v},
\label{eq:valtest}
\end{equation}
where the subscript $0$ denotes information matrices and scores built from the covariance shapes $\bm\Sigma^r_{0,k}$ alone, so that the accepted-set scale $\hat s$ of \cref{eq:scale} on $d_{\mathrm{ab}}$ degrees of freedom appears once. The two sets are independent, and the covariance of $\bm c$ is the sum of the validation and estimation covariances. The $F$ distribution holds for a fixed accepted set, and the scale estimated after observation selection is examined empirically in \cref{subsec:expstate}. The focused statistic along a worst direction (\cref{app:focused}) has the non-centrality $\delta^2\bm v^\top\bm\Lambda^{\mathrm v}\bm v/\kappa_{\bm v}^2$ under a differential fault $\delta\bm v$, which makes $\MDB_{\mathrm{eff}}(\bm v)$ its detection limit along $\bm v$ under the normal approximation of \cref{def:mdb}. The six-degree-of-freedom test has a different non-centrality for the same fault, and the two limits are not in a fixed ratio.

\paragraph{Focused validation statistic and equal-half splits.}\label{app:focused}\label{app:halfsplit}
Along a unit direction $\bm v$ of one block, with $\bm c_{\mathrm b}$ the block of $\bm c$ in \cref{eq:valtest} and $\bm\Lambda_0$, $\bm\Lambda^{\mathrm v}_0$ the shape-only counterparts of $\bm\Lambda_{\mathrm f}$ and $\bm\Lambda^{\mathrm v}$ ($\bm\Lambda_{\mathrm f}=\bm\Lambda_0/\hat s$), the focused score is $u_{\bm v}=\bm v^\top\bm\Lambda^{\mathrm v}_0\bm c_{\mathrm b}$ and
\begin{equation}
\frac{u_{\bm v}^2}{\hat s\,\bm v^\top\bigl(\bm\Lambda^{\mathrm v}_0+\bm\Lambda^{\mathrm v}_0\bm\Lambda_{0}^{-1}\bm\Lambda^{\mathrm v}_0\bigr)\bm v}\sim F_{1,d_{\mathrm{ab}}},
\label{eq:focused}
\end{equation}
whose denominator is the null variance of the score inflated by the fit error of the assessed estimate, the inflation that defines $\kappa_{\bm v}$ in \cref{eq:mdbeff}. The validation set is excluded from estimation, as in the camera-calibration evaluation of \citet{peng2019calibrationwizard}. For independent sets with comparable information for each observation, let $\hat\theta_E$ and $\hat\theta_V$ denote their local estimates along the same fixed rotation or translation direction in a common left-perturbation frame. To first order,
\begin{equation}
\begin{aligned}
\operatorname{Var}(\hat\theta_E-\hat\theta_V)
&=\operatorname{Var}(\hat\theta_E)+\operatorname{Var}(\hat\theta_V)\\
&\simeq c_{\bm v}\left(\frac{1}{n_E}+\frac{1}{n_V}\right),
\end{aligned}
\label{eq:halfsplit}
\end{equation}
where $n_E$ and $n_V$ are the numbers of observations in the two estimates and $c_{\bm v}$ is their common directional variance coefficient. For fixed $n_E+n_V$, this variance is minimized by equal halves, which the half-split experiments therefore adopt. Different pose information or observation rejection can shift this balance, and the detection limit is computed from the actual information matrices in \cref{eq:mdbeff}. For equal halves, the fit and validation information are comparable, which yields $\kappa_{\bm v}\approx\sqrt2$, and the random-split limit is about $2.8\sqrt2\,\kappa_{\bm v}\approx5.6$ times the worst-direction standard deviation of the full accepted set. Under the default requirement $\delta=3\epsilon$ it meets the detection-limit requirement only when the worst-direction standard deviation is below $0.54\,\epsilon$. With $\sigma\propto K^{-1/2}$, this requires about $3.5$ times the number of observations that meet the precision requirement. The factor holds for equal halves with a common dominant direction.

\section{Additional Experimental Results}
\label{app:extra}

\paragraph{Baseline configuration.}\label{app:baselines}
velo2cam and Dom2019 are given the camera hole centers of VeriCalib and run their unmodified LiDAR detectors. As both detectors expect a multi-ring scan, each Livox cloud is cropped to the board region and converted to a virtual $128$-ring scan over its measured elevation field of view~\citep{hu2026p2calib}. velo2cam uses a circle radius of $0.12$\,m on Robot-A, and its other parameters retain their defaults. Dom2019 retains its released configuration except the board and sensor entries, namely $128$ layers, the board region as pass-through filter, circle-radius limits scaled to the $0.12$\,m radius, and a $0.50\times0.40$\,m refinement template. Dom2019 is reported only if at least three circles are detected. The D- and A-optimal criteria follow the released code of \citet{maye2016self} and \citet{peng2019calibrationwizard}, and the E-optimal criterion the published formula of \citet{lv2022observability}, as its selection module is unreleased. All criteria rank candidates without unit normalization between rotation and translation.

\begin{figure*}[!t]\centering\includegraphics[width=\textwidth]{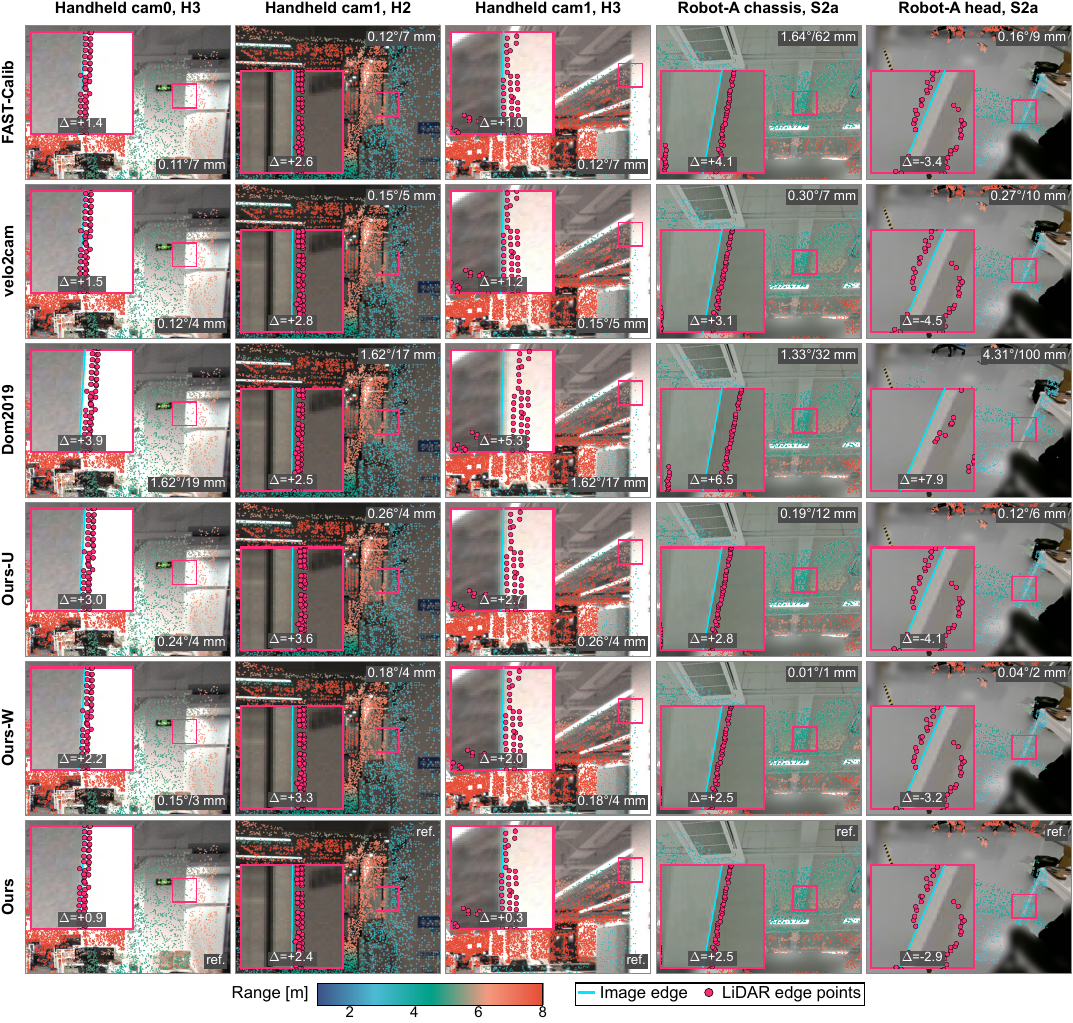}
\caption{LiDAR points of held-out scenes projected with each extrinsic (rows) on three handheld windows and two Robot-A captures, colored by range. Insets: a near-vertical image edge (cyan) with the depth-discontinuity points (magenta) and their mean horizontal offset $\Delta$ [px]; tags: deviation from Ours (ref.). In each image, the inset edge is the candidate with the most depth-discontinuity points, chosen before any offset is computed.}
\label{fig:projqual}\end{figure*}
\paragraph{Projection consistency on held-out scenes.}\label{app:projection}
As a check independent of the calibration target, LiDAR data of scenes not used for calibration are projected with every extrinsic and scored against the image. The scenes are three handheld sequences distinct from S5 (H1 to H3, both cameras) and the background of the S2a and S2b captures of Robot-A with the target masked. They yield $79$ handheld windows of five images with a local LiDAR map and $65$ Robot-A captures, all selected by data quality before any score is computed. Five scores are evaluated on fixed point sets: the edge score of \citet{levinson2013automatic} on LiDAR depth discontinuities, the median and the truncated mean of the distance of these points to the nearest image edge, the cross-view intensity consistency of map points, and the mutual information of LiDAR and image intensity~\citep{pandey2015automatic}. Each extrinsic is compared with VeriCalib by a paired Wilcoxon test over windows and captures with Holm correction at $p<0.01$. Of the $26$ pairs of a score and an extrinsic at least $0.7^\circ$ from VeriCalib, $19$ are significantly worse and none is better, and the seven non-significant cases all concern Dom2019, so that the held-out scenes detect the coarse errors of the compared extrinsics. \cref{fig:projqual} shows the alignment on image edges, where VeriCalib attains the smallest absolute offset $|\Delta|$ in every shown image.

\paragraph{Accuracy of selected subsets.}\label{app:subsetacc}
Beyond the predicted precision, we compare the estimate from the first $k$ selected observations with the weighted estimate from the remaining observations of the session, which the selection does not use, on the single-session sets with $20$ initial pairs. On head S2, the weakest-direction subsets of three to five observations deviate by medians of $0.09^\circ$ to $0.15^\circ$ and $4.2$ to $5.0$\,mm, against $0.16^\circ$ to $0.19^\circ$ and $6.0$ to $8.2$\,mm for random and $0.18^\circ$ to $0.27^\circ$ and $6.7$ to $12.8$\,mm for far-board-first subsets. On FS-B the rotation deviations are $0.21^\circ$ to $0.27^\circ$ against $0.41^\circ$ to $0.54^\circ$ for random subsets. On S4, the weakest-direction subsets deviate by $0.15^\circ$ to $0.26^\circ$ and $5.5$ to $9.4$\,mm on the chassis pair, against $0.31^\circ$ to $0.41^\circ$ and $9.2$ to $12.1$\,mm for random subsets, and by $0.15^\circ$ to $0.26^\circ$ against $0.19^\circ$ to $0.32^\circ$ in rotation on the head pair, with comparable translation. The advantage is largest for subsets of three to five observations, for which the choice of each observation matters most.

\begin{table}[!t]
\centering\caption{Circle-center RMSE [mm] $\downarrow$ of three circle fits on synthetic partial arcs.}
\label{tab:shortarc}\footnotesize\setlength{\tabcolsep}{7pt}\renewcommand{\arraystretch}{1.15}
\begin{tabular}{@{}lccc@{}}
\toprule
\rowcolor{tabhead}Arc fraction & Free & \textbf{Soft prior} & Fixed radius\\
\midrule
$1.00$ & $1.95$ & $1.95$ & $1.93$\\
$0.45$ & $4.17$ & $3.98$ & $2.01$\\
$0.33$ & $6.94$ & $5.72$ & $2.15$\\
$0.22$ & $16.73$ & $7.31$ & $2.76$\\
\bottomrule
\end{tabular}
\par\vspace{1.5pt}
{\scriptsize\raggedright \textit{Note:} $2000$ paired trials per row; $40$ points with $6$\,mm noise on a $120$\,mm circle; Fixed radius: exact radius assumed.\par}
\end{table}

\paragraph{State-guided collection.}\label{app:stateloop}
We replay the loop between assessment and collection on the recorded head S1, head S2, and chassis S2 sets ($14$, $21$, and $26$ observations). From two initial observations, the loop runs \cref{alg:pipeline} after each addition and adds the recorded observation that most increases the Schur-complement information along the direction that the state indicates, and the observations that it does not select validate the final estimate. Over $10$ initial pairs, every ordering meets the precision requirement after a median of $4$ to $6.5$ observations, and the DL is the binding requirement. The state-guided loop meets both requirements after medians of $10$, $10$, and $12.5$ observations on the three sets, against $11$, $11.5$, and $19$ for random ordering, $12$, $14.5$, and $20$ for far-board-first ordering, and $12$, $11$, and $16.5$ for the recording order. On chassis S2, every set chosen by the loop rejects the batch test in its lower tail, so that the state requests the validation set, and the unselected observations confirm the verified state in all $10$ replays, against $9$ for random, none for far-board-first ordering, and $6$ for the recording order.

\paragraph{Short-arc circle fit.}\label{app:shortarc}
\cref{tab:shortarc} compares the free three-parameter circle fit, the soft radius prior of \cref{eq:circlemap}, and a fit with the exact radius fixed on simulated boundary arcs of decreasing angular coverage, on which the free fit becomes ill-conditioned along the radial axis (\cref{subsec:lidarcov}).

\begin{table}[!t]
\centering\caption{Comparison of the beam and range-incidence LiDAR edge-noise models on real-world sets.}
\label{tab:lidarnoise}\footnotesize\setlength{\tabcolsep}{1.8pt}\renewcommand{\arraystretch}{1.15}
\begin{tabularx}{\linewidth}{@{}l*{7}{C}@{}}
\toprule
\rowcolor{tabhead} & & \multicolumn{2}{c}{Error [mm]} & \multicolumn{2}{c}{$\hat s$} & \multicolumn{2}{c}{$\hat\rho_z$} \\
\cmidrule(lr){3-4}\cmidrule(lr){5-6}\cmidrule(lr){7-8}
\rowcolor{tabhead}Dataset & $K$ & B & RI & B & RI & B & RI \\
\midrule
FS-B & $15$ & $15.14$ & \cellcolor{tabbest}$15.04$ & $50.67$ & \cellcolor{tabbest}$17.98$ & \cellcolor{tabbest}$-0.12$ & $-0.23$\\
Head S2 & $21$ & $6.85$ & \cellcolor{tabbest}$6.44$ & $14.69$ & \cellcolor{tabbest}$3.43$ & $-0.42$ & \cellcolor{tabbest}$-0.31$\\
Chassis S2 & $26$ & $17.82$ & \cellcolor{tabbest}$17.48$ & $6.85$ & \cellcolor{tabbest}$4.64$ & $-0.17$ & \cellcolor{tabbest}$+0.02$\\
\bottomrule
\end{tabularx}
\par\vspace{1.5pt}
{\scriptsize\raggedright \textit{Note:} B/RI: beam/range-incidence model; Error: leave-one-out alignment error; $\hat\rho_z$: residual-range correlation; \colorbox{tabbest}{\strut better} model per column pair.\par}
\end{table}

\paragraph{LiDAR edge-noise model.}\label{app:lidarnoise}
The range-incidence model of~\cref{eq:lidarnoise} is selected on held-out data against the beam covariance that LiDAR odometry back ends assign to each return~\citep{yuan2022efficient,zheng2024fast},
\begin{equation}
\bm\Sigma^{\mathrm{beam}}=\sigma_d^2\,\bm w\bm w^\top+d^2\sigma_\omega^2\bigl(\bm I-\bm w\bm w^\top\bigr),
\label{eq:beamcov}
\end{equation}
with ray direction $\bm w$, range $d$, ranging noise $\sigma_d$, and range-proportional bearing noise $\sigma_\omega$. When the beam covariance is resolved onto the board plane, grazing rays project ranging noise onto the in-plane direction, and incidence effects thus appear geometrically instead of through the explicit $1/\cos^2\eta$ term. The two models differ only in the point variance of the boundary Fisher information, with $\sigma_d=5$\,mm and $\sigma_\omega=0.3$\,mrad against $\sigma_{L,0}=4$\,mm and $\sigma_{L,1}=10^{-3}$, and they are compared inside an otherwise identical structured factor on the three raw-support sets with $K\ge14$ (\cref{tab:lidarnoise}). The range-incidence model attains the lower held-out discrepancy, $12.99$ against $13.27$\,mm on average, and the smaller residual scale on all three sets.

\begin{table}[!t]
\centering\caption{Sensitivity to the removal level on the combined Robot-A head and chassis batches.}
\label{tab:alphabias}\footnotesize\setlength{\tabcolsep}{1.8pt}\renewcommand{\arraystretch}{1.15}
\begin{tabularx}{\linewidth}{@{}l*{8}{C}@{}}
\toprule
\rowcolor{tabhead} & \multicolumn{4}{c}{Head ($K{=}35$)} & \multicolumn{4}{c}{Chassis ($K{=}40$)}\\
\cmidrule(lr){2-5}\cmidrule(lr){6-9}
\rowcolor{tabhead}$\alpha_{\mathrm{bias}}$ & $|\calS|$ & $t$ & $\sigma_t$ [mm] & $\MDB_t$ [mm] & $|\calS|$ & $t$ & $\sigma_t$ [mm] & $\MDB_t$ [mm]\\
\midrule
$0.002$ & $30$ & $-2.94$ & $1.84$ & $10.7$ & $34$ & $-2.62$ & $2.06$ & $11.9$\\
$0.005$ & $29$ & $-2.97$ & $1.95$ & $11.4$ & $32$ & $-2.57$ & $2.02$ & $11.7$\\
$0.01$ & $28$ & $-3.53$ & $1.88$ & $11.0$ & $32$ & $-2.57$ & $2.02$ & $11.7$\\
$0.02$ & $26$ & $-3.19$ & $1.92$ & $11.3$ & $32$ & $-2.57$ & $2.02$ & $11.7$\\
$0.05$ & $26$ & $-3.19$ & $1.92$ & $11.3$ & $28$ & $-2.14$ & $1.96$ & $11.4$\\
\bottomrule
\end{tabularx}
\par\vspace{1.5pt}
{\scriptsize\raggedright \textit{Note:} $|\calS|$: accepted observations; $t$: trend statistic~\cref{eq:trend}.\par}
\end{table}

\paragraph{Parameter sensitivity.}\label{subsec:sensitivity}\label{app:alphabias}
In a simulation with $K{=}16$, one random half split, and the median over $120$ trials, the observation-level test attains its level, with false-positive rates of $0.007$, $0.046$, and $0.095$ for $\alpha=0.01$, $0.05$, and $0.10$. Increasing the required power from $0.70$ through $0.80$ to $0.90$ increases the median $\MDB_t$ from $19.5$ through $21.9$ to $25.4$\,mm, and lowering the level from $0.10$ through $0.05$ to $0.01$ increases it from $19.5$ through $21.9$ to $26.8$\,mm, as expected for a stricter detection requirement. Across $\alpha_{\mathrm{bias}}=0.005$ to $0.05$, the diagnosis detects at least $0.99$ of the injected faults, and its false-removal rate, from $0.007$ to $0.051$, follows the removal level. On the combined Robot-A sets, varying $\alpha_{\mathrm{bias}}$ from $0.002$ to $0.05$ (\cref{tab:alphabias}) reduces the accepted count by four and six observations on the head and chassis pairs, and the reported precision and detection limit of both pairs vary by less than $7\%$. The head trend statistic lies between $-2.9$ and $-3.5$ and the chassis statistic between $-2.1$ and $-2.6$, and both combined sets remain inconsistent at every level. The removal level therefore determines which marginal observations are removed and leaves the state unaffected. Sweeping the boundary-point equicorrelation $\bar\rho$ of \cref{eq:neff} from $0$ to $0.3$ on session S1 lowers the scale from $12.7$ to $0.8$ on the head pair and from $12.8$ to $2.0$ on the chassis pair. It does not alter the accepted sets except at $\bar\rho=0$ on the chassis and alters the precision and detection limits by less than $15\%$ on the head and $13\%$ on the chassis of their values at the default $\bar\rho$. The batch decision of the chassis session differs for $\bar\rho\ge0.1$, and $\bar\rho$ is fixed at $0.05$ for all sets (\cref{sec:results:setup}).

\end{document}